\documentclass[11pt]{article}
\usepackage[margin=1in]{geometry}

\usepackage{graphicx}
\usepackage{booktabs}
\usepackage{amsfonts}
\usepackage{amsmath}
\usepackage{amsfonts}
\usepackage{amsthm}
\usepackage{eqparbox}

\usepackage{wrapfig}
\usepackage{rotating}
\usepackage{makecell}
\usepackage{thmtools,thm-restate}
\usepackage{hyperref}
\usepackage{xurl}
\usepackage{caption}
\newcommand{\Rset}{{\mathbb R}}
\newcommand{\indicator}[1]{ {\mathbb I}[ #1 ]  }
\newcommand{\expected}[1]{ {\mathbb E}[ #1 ]  }
\newcommand{\sibling}[1]{ \emph{sibling}(#1)}

\newcommand{\argmax}{\operatornamewithlimits{\mbox{\it arg\,max}}}

\renewcommand{\max}{\operatornamewithlimits{\mbox{\it max}}}
\renewcommand{\min}{\operatornamewithlimits{\mbox{\it min}}}

\newcommand{\ccwif}{ {\tt WIF}\  }
\newcommand{\ccthen}{ {\tt THEN}\  }
\newcommand{\ccelse}{ {\tt ELSE}\  }
\newcommand{\cccombine}{ {\tt COMBINE}\  }
\newcommand{\ccwith}{ {\tt WITH}\  }
\newcommand{\ccforall}{ {\tt FORALL}\  }

\newcommand{\ccsoftmax}{ {\tt SOFTMAX}\  }
\newcommand{\ccsum}{ {\tt SUM}\  }
\newcommand{\cclreg}{ {\tt LOG-REG }\  }

\newcommand{\boldv}{{\boldsymbol v}}
\newcommand{\boldi}{{\boldsymbol i}}
\newcommand{\boldr}{{\boldsymbol r}}

\newtheorem{proposition}{Proposition}

\newtheorem{definition}{Definition}
\newtheorem{Description}{Description}

\newcommand{\arity}[1]{\emph{arity}(#1)}

\newcommand{\leftms}{ \{\hspace{-2pt}|  }
\newcommand{\rightms}{ |\hspace{-2pt}\}  }

\usepackage[normalem]{ulem}
\usepackage{mwe}
\usepackage{mathtools}
\usepackage{amsmath}
\usepackage{multirow}
\usepackage{xcolor, colortbl}
\definecolor{gray}{rgb}{0.85,0.85,0.85}
\usepackage{algorithm}
\usepackage{algpseudocode}
\usepackage[beginLComment=/*~,endLComment=~*/]{algpseudocodex}

\usepackage{listings}
\usepackage[T1]{fontenc}
\newcommand{\new}[1]{{#1} }

\newcommand{\Datoms}{D}
\newcommand{\MAPatoms}{M}
\newcommand{\Oatoms}{O}

\begin{document}

\title{\textbf{A Neuro-Symbolic Approach for Probabilistic Reasoning on Graph Data}}

\author{
  Raffaele Pojer$^1$\quad
  Andrea Passerini$^2$\quad
  Kim G. Larsen$^1$\quad
  Manfred Jaeger$^1$ \\
  \small $^1$Aalborg University, Aalborg, Denmark \quad
  \small $^2$University of Trento, Trento, Italy \\
  \texttt{\small \{rafpoj, kgl, jaeger\}@cs.aau.dk}, \texttt{\small andrea.passerini@unitn.it}
}

\date{}
\maketitle
\begin{abstract}
{\bf Background:}
Graph neural networks (GNNs) excel at predictive tasks on graph-structured data. However, their inability to incorporate (symbolic) domain knowledge and their specialized discriminative functionality limit their applicability in more general learning and reasoning applications. Relational Bayesian Networks (RBNs), on the other hand, enable the construction of fully generative probabilistic  models over graph-like structures that can incorporate high-level symbolic knowledge and support a wide range of probabilistic inference tasks.

{\bf Objectives:}
Our aim is to develop a neuro-symbolic framework for learning and reasoning with graph data that combines the learning capabilities and predictive power of GNNs with the flexible modeling and reasoning capabilities of RBNs, opening new types of application areas and reasoning tasks in the field of graph learning. 

{\bf Methods:}
We develop an approach in which GNNs are integrated seamlessly as components into an RBN model, and present two specific implementations of the general approach: one in which GNNs are compiled into the native RBN language, and one in which the GNN is maintained as an external component. In both versions, the integration, on the one hand, faithfully maintains semantics and computational properties of the GNN, and, on the other hand, fully aligns with the existing RBN modeling paradigm. We present a method for maximum a-posteriori (MAP) probabilistic inference on our neuro-symbolic GNN-RBN models.   

{\bf Results:}
We demonstrate the range of possible applications of MAP inference in our neuro-symbolic framework by showing how it lets us solve two very different problems: in the first application we show how the integrated framework allows us to turn a GNN for node classification into a collective classification model that explicitly takes homo- or heterophilic label distribution patterns into account, thereby greatly increasing the accuracy of the base GNN model. In the second application, we introduce a new type of multi-objective network optimization task in an environmental planning scenario, and show how MAP inference in our framework provides decision support in such settings. For both applications, we introduce new publicly available benchmark data.

{\bf Conclusions:}
We introduce a new powerful and coherent neuro-symbolic framework for handling graph data, and demonstrate its applicability on two very different tasks.
\end{abstract}

\maketitle

\section{Introduction}

Learning with graph and network data is an area in machine learning that has seen explosive growth in recent years. This growth is fueled by an increasing amount of graph-structured data (generated, e.g., by social, sensor, or traffic networks) on the one hand, and the success of graph neural networks (GNNs)~\cite{ZHOU202057} for graph-related machine learning problems on the other hand.
In most cases, GNNs are used to solve specific node classification, link prediction, or graph classification tasks, and are trained in an end-to-end manner with a loss function customized for this task. 
Alternatively, GNNs can also be trained in an unsupervised manner to obtain representations that capture the graph structure. They can then be used for a variety of downstream applications \cite{kipf2016variational,HamYinLes17}, or as graph generators \cite{simonovsky2018graphvae,liu2019graph}.
In all cases, the construction of the GNN is almost entirely data-driven. Prior knowledge or known constraints on the solution can only be injected into the learning process by a careful design of the model architecture and the loss function. As in other deep learning approaches, the resulting model is then represented entirely by high-dimensional weight matrices, leading to the characteristic black box nature
of neural network models. Furthermore, predictive inference in the resulting models consists of numerical computations that do not naturally capture high-level symbolic reasoning.

\emph{Neuro-symbolic integration} tries to overcome these general limitations of neural network models by combining neural architectures with symbolic reasoning
components \cite{yu2023survey}. In most cases, this integration is seen as one of combining low-level ``perceptual'' tasks best handled by a neural component, with one for performing high-level symbolic reasoning operations. Moreover, in many cases, neuro-symbolic frameworks are designed to solve specific, standard machine learning tasks. Thus, \cite{yu2023survey}, for example, formalizes neuro-symbolic integration in terms of an underlying prediction problem.

In the field of \emph{statistical relational learning (SRL)}~\cite{getoor2007introduction}, a multitude of frameworks have been developed that combine
probabilistic and logic-based models and inference. Already bridging the gap between symbolic and numeric representations, but
lacking the computational efficiency of neural approaches, SRL
frameworks can provide useful components for neuro-symbolic systems.  A natural connection between logic-symbolic knowledge
and graph-structured data, as clearly embodied by knowledge graphs, suggests a promising role for GNNs as the neural component in
neuro-symbolic integration \cite{ijcai2020p679}. However, only a relatively small number of works seem to have considered
neuro-symbolic systems built from SRL and GNN components \cite{qu2019gmnn,zhangefficient}.
SRL frameworks, unlike neural network models, support generative probabilistic models and general probabilistic inference:
rather than being tailored to one specific prediction task at a time, a single model can be applied to a wide range of
reasoning tasks that can be framed in terms of conditional probability queries.

\emph{Relational Bayesian networks (RBNs)} are an SRL framework with the full expressive power of relational first-order logic, and the
general probabilistic inference capabilities of Bayesian networks \cite{rbn1997}. A tight connection between GNNs and RBNs was
first pointed out in \cite{jaegerAIB22}, where it was shown that a large class of GNNs can be directly represented as RBNs.
In this paper, we present a framework for a natural integration between RBNs and GNNs. It is designed to enable a wide range
of learning and reasoning tasks for graph and network domains, offering a
seamless integration of expert-defined symbolic knowledge and data-driven neural ``black box'' components, and support for general probabilistic inference.

In particular, in this paper, we
\begin{itemize}
\new{
\item develop a tight integration between GNNs and RBNs, based on two key insights:
  \begin{itemize}
  \item Semantic compatibility: functions computed by GNNs are semantically equivalent to conditional probability
    factors that are the building blocks of RBN models.
  \item Syntactic representability: GNN functions can be directly encoded in the native RBN representation language.  
  \end{itemize}
\item provide two concrete implementations of this integration within the \emph{Primula} RBN software:
  by compilation into native RBN code and by interfacing to an external GNN tool, 
  }
\item develop a method for \emph{maximum a-posteriori (MAP)} inference on our GNN-RBN models,
\item demonstrate the usefulness and versatility of the resulting framework by applications to two very different graph learning and
  inference tasks:
  \begin{itemize}
  \item the standard machine learning node classification task under homophilic and heterophilic label distributions,
  \item a novel type of multi-objective optimization tasks representing planning problems in network domains.
\end{itemize}

\end{itemize}

The basic principles of our GNN-RBN integration were first developed in \cite{jaegerAIB22,pojer2024generalized}. The integrated
software is described in \cite{primulatool}. This paper presents a consolidated, complete description of the whole framework, articulates our
first key
insight more formally as Proposition \ref{prop:reprequiv}, gives a full account of MAP inference in our framework, and introduces the novel
 multi-objective graph optimization task. 

Both data and code for the applications are publicly available at \url{https://github.com/raffaelepojer/NeSy-for-graph-data}.
Primula, the RBN framework used in this paper, is available at \url{https://github.com/manfred-jaeger-aalborg/primula3}.

\section{Related Works}
Neuro-symbolic (NeSy) integration~\cite{hitz2022,garcez2022neural} is a
growing field aimed at combining the predictive power of (deep) neural
networks with the reasoning capabilities of symbolic AI. A common
design principle consists in combining neural
predictors with logic-based reasoning components~\cite{ijcai2020p688}, leveraging
fuzzy~\cite{badreddine2022ltn,marra2021neural,pryor-nepsl} or
probabilistic~\cite{manhaeve2021neural,ahmed2022semantic,vanKrieken2023anesi}
logic to obtain end-to-end differentiable architectures.
Several approaches for combining deep learning with logic-based SRL frameworks have
been proposed~\cite{ManEtAl18,vsourek2021beyond,pryor-nepsl}. The \emph{Deep ProbLog} framework
of \cite{ManEtAl18} shares with our approach the principle of using neural network outputs
(not specifically GNNs) as inputs to a probabilistic SRL model.
The focus of \cite{sourek2018lifted,vsourek2021beyond}, on the other hand, is on using logic-based, symbolic representations as templates for the construction of neural networks. Like we do in our work, the authors of \cite{sourek2018lifted,vsourek2021beyond} establish an equivalence between a model represented in a probabilistic-logic framework and a neural network model. However, their equivalence is obtained by compiling a logic model into a neural representation, whereas our work takes the opposite direction. Moreover, their models, like standard neural networks, are designed for special-purpose predictive tasks.

Systems that combine neural networks with logic-symbolic knowledge representation include NeurASP~\cite{yang2021neurasp}, which integrates Answer Set Programming (ASP) with neural perception modules, allowing declarative logic rules to guide learning and inference.
All of these frameworks typically do not define a fully generative
probabilistic model, which limits their inference capabilities. As a
result, they have primarily been used for standard supervised learning
tasks. 
The recently introduced  SLASH~\cite{skryagin2023scalable} system introduces a unifying framework that combines neural, logical (also based on ASP), and tractable probabilistic modules. It can also support probabilistic queries beyond conditional class probabilities given input features. 
The key distinction between SLASH and earlier approaches like NeurASP and DeepProbLog is its ability to connect neural predicates not only to neural networks but also to probabilistic circuits.
\new{NeuPSL \cite{pryor-nepsl} shares with the already mentioned approaches the principle of integrating a distinguished set of \emph{neural predicates} into logic-symbolic rules or constraints. Unlike all other approaches, these rules and constraints are interpreted under many-valued \L ukasiewicz semantics.}
In terms of neural-symbolic integration, systems like \cite{skryagin2023scalable, yang2021neurasp, ManEtAl18,pryor-nepsl} follow a strict division of labor, where neural components are responsible for low-level perception tasks and the symbolic logic handles high-level reasoning. 
In contrast, our approach does not impose such a fixed separation and allows for more flexible and integrated interactions between neural and logical components.

Driven by the success of deep generative models for images and text, a
recent trend in the NeSy community targets the application of
neuro-symbolic ideas to generative tasks. Examples include
stroke-based drawing~\cite{Liang22} constrained image
generation~\cite{can2020,misino2022vael}, symbolic
regression~\cite{pmlr-v202-bendinelli23a} and dialogue structure
induction~\cite{pryor-etal-2023-using}. These powerful frameworks are
however typically tailored to instance generation and do not support
arbitrary inference tasks.

It is commonly assumed that due to the smoothing effect of the message-passing operations, GNNs for node classification perform better in the case of homophilic than heterophilic label distributions \cite{bodnar2022neural,luan2024heterophilic}. However, as pointed out in \cite{qu2019gmnn}, GNNs predict labels for different nodes independently, which generally limits their ability to capture label dependencies. To address this, \cite{qu2019gmnn} introduced GMNN, a method that combines standard GNNs with a customized additional GNN structure that captures label dependencies and is inspired by how SRL models based on Markov random fields handle such dependencies. In contrast to our approach, the work of \cite{qu2019gmnn} is highly specialized towards exploiting homophily for node classification and is not a general neuro-symbolic integration.  

All these existing approaches have in common that symbolic knowledge or logical constraints are specified prior to learning and result in a model that is trained in an end-to-end manner for a specific purpose, which is partly defined by the constraints. 
While our GNN-RBN integration also supports such task-specific end-to-end optimization solutions, our focus in this paper is somewhat different: we explore how embeddings of GNNs into a probabilistic-symbolic framework allow us to harness their predictive power to solve a wider range of reasoning tasks.

The connection between GNNs and neuro-symbolic integration has already
been highlighted in the recent past~\cite{ijcai2020p679}. 
More recently, \cite{zhangefficient} proposed ExpressGNN, a framework that integrates Markov Logic Networks with Graph Neural Networks.
This method is tailored to triplet completion over knowledge graph structures and does not support more general forms of structured domains or fully generative modeling.

\new{
\begin{table}[h]
\centering
\caption{\new{Comparison of neuro-symbolic frameworks. \checkmark:\ fully satisfied, $\circ$: partially satisfied, \texttimes:\ not satisfied.}}
\label{tab:nesy-comp}
\resizebox{\textwidth}{!}{%
\begin{tabular}{clccccc}
\toprule
 & \textbf{Property}
    & \parbox{20mm}{\textbf{DeepProbLog} \\
    \textbf{NeurASP}}
     & \textbf{NeuPSL}
    & \textbf{GMNN}
    & \textbf{SLASH}
    & \textbf{Ours} \\
\midrule
\multirow{3}{1.5cm}{
\parbox{1.1cm}{\emph{Model Semantics}}
} 
&
Fully generative probabilistic model
    & $\circ$  & \texttimes & \texttimes & $\circ$ & \checkmark \\
& First-order logic expressivity
    & $\circ$    & $\circ$ & \texttimes & $\circ$ & \checkmark \\
& Transitive closure expressivity
& \checkmark &\texttimes &\texttimes  & \checkmark & \texttimes\\
\midrule
\multirow{3}{1.5cm}{
\parbox{1.1cm}{\emph{NeSy \\ integration}}   
} 
& GNN support
    & \texttimes & \texttimes & \texttimes & \texttimes & \checkmark \\
& CNN support
    & \checkmark & \checkmark & \checkmark & \checkmark & $\circ$  \\
&
Arbitrary neural/symbolic interleaving
    & \texttimes & \texttimes & \texttimes & \texttimes & \checkmark \\
\midrule
\multirow{3}{1.5cm}{
\parbox{1.1cm}{\emph{Query support}}
} 
&  
Predictive inference
    & \checkmark & \checkmark & \checkmark & \checkmark & \checkmark \\
&
MAP inference
    & \texttimes & \texttimes & \texttimes & \texttimes & \checkmark \\
&
General conditional probabilities
    & \texttimes & \texttimes & \texttimes& $\circ$ & \checkmark \\
\midrule
\multirow{5}{1.5cm}{
\parbox{1.1cm}{\emph{Key computations}}   
}
& Gradient descent & \checkmark & \checkmark & \checkmark & \checkmark & \checkmark \\
& 
Logic solvers & \checkmark & \texttimes & \texttimes & \checkmark & \texttimes\\
& Weighted model counting 
& \checkmark & \texttimes & \texttimes & \checkmark & \texttimes\\
& 
MCMC sampling & \texttimes & \texttimes & \checkmark & \texttimes &  \checkmark\\
& 
Convex optimization & \texttimes & \checkmark & \texttimes &  \texttimes & \texttimes\\
\bottomrule
\end{tabular}}
\end{table}
}

\new{Table \ref{tab:nesy-comp} provides a comparison between our RBN-GNN framework and other NeSy frameworks that are most closely
  related to ours. We compare on the basis of four different groups of properties:  \emph{Model Semantics} properties 
  describe semantic and expressivity aspects of the models. \emph{Fully generative} refers to the capability to model full joint distributions, not only
  conditional label distributions, thus providing a semantic basis (but not necessarily algorithmic solutions) to sample graphs, and to
  compute arbitrary conditional probabilities.
  \emph{First-order logic} expressivity is the ability to model any properties expressible in first-order
  predicate logic. \emph{Transitive closure} expressivity is the ability to model the transitive closure of a relation. This is a
  characteristic strength of logic programming based approaches due to the fixpoint (or stable model) semantics of the symbolic components.
  Under \emph{NeSy integration} we list key elements of the interface between the neural and symbolic parts. Here we distinguish whether
  the integration is with graph neural or with standard (convolutional) neural networks. \emph{Arbitrary neural/symbolic interleaving} refers
  to the support for iterating input/output relationships between neural and symbolic components arbitrarily often, and in arbitrary order.
  Under \emph{Query support} we list what types of inference tasks are supported. \emph{Predictive inference} covers standard
  node and graph classification and link prediction tasks. \emph{MAP inference} is the computation of the most probable joint configuration
  of unobserved relations, as more fully defined in Section \ref{sec:likgraph} below. \emph{General conditional probabilities} means that
  arbitrary conditional probability queries are supported, without a predefined division into input (conditioning) and output variables.
  Finally, we also consider what kind of fundamental computational tasks are performed for learning and inference in the frameworks.
  \emph{Gradient descent} is a staple used by all approaches. 
  An important distinction arises in whether the symbolic component of a framework has a strictly logical semantics that is handled
  by dedicated \emph{logic solvers},  or whether the logical expressivity arises as a corner case of probabilistic or fuzzy semantics, and is
  also handled by numerical computations, which,  besides gradient descent, include \emph{sampling} or \emph{convex optimization}
  operations. 
  }

\section{GNN-RBN Integration}

\subsection{Preliminaries}
\label{sec:preliminaries}

We use the following concepts and notation taken from logic and here adapted to graphs:
a \emph{relation} $R$ is defined by a \emph{name}, an arity $\arity{R}\in 0,1,2,\ldots$ and a \emph{value range} that can
be Boolean, categorical, or numerical.  
A relation of arity 1 is also referred to as a (node) \emph{attribute}, a relation of arity 2 as an \emph{edge relation}, and a relation of arity 0 as a \emph{graph property}. A \emph{signature} ${\mathcal R}=\{R_1,\ldots,R_M\}$ is a set of relations.
For example, graphs describing chemical molecules can be defined over a signature containing an edge relation \emph{Bond} with
values $\{0,1\}$, an attribute $\emph{Element}\in \{ \mbox{H},\mbox{He},\ldots,\mbox{Og}\}$, and a graph property
$\emph{Toxic}\in \{\emph{true},\emph{false}\}$.
\new{We use $u,v,\ldots$ to denote variables ranging over nodes of a graph. Specific nodes are referenced by \emph{node identifiers}, which
  can either be meaningful names or generic integer indices denoted $i,j, \ldots$.
  Bold fonts $\boldv,\boldi,\ldots$ represent tuples of
variables or identifiers.}
An \emph{atom} is an expression composed of a relation name and node variables or node identifiers corresponding to the arity of the relation as arguments. An atom is \emph{ground} if all its arguments are node identifiers. For example $\emph{Bond}(v,u), \emph{Element}(v)$ are atoms,
$\emph{Bond}(\emph{atom\_12},\emph{atom\_7})$, $\emph{Element}(\emph{atom\_7})$ are ground atoms.

\new{A \emph{graph} over the signature  ${\mathcal R}$ is a pair $(V,\boldr)$ where $\boldr=(\boldr_1,\ldots,\boldr_M)$
  maps all ground atoms $R(\boldi)$
  constructible from $R\in {\mathcal R}$ and identifiers $\boldi$ for nodes in $V$ to a possible value of $R$. Thus, a graph in this
  sense can in fact be an attributed, multi-relational hyper-graph, if ${\mathcal R}$ contains node attributes, multiple edge relations,
  or relations of arity $>2$, respectively. We use ${\mathcal G}(V,{\mathcal R})$ to denote the
set of all graphs over ${\mathcal R}$ with vertex set $V$. }

We use uppercase letters (or short strings) $A, B, U, LH, \ldots$  to denote random variables, and corresponding
lowercase letters $a, b, u, lh, \ldots$ as generic symbols for specific values they can take. All our random variables will correspond to
ground atoms, and in the context of speaking about random variables, we often use ``atom'' as short for ``ground atom''.
An equation of the form $\boldsymbol{A}=\boldsymbol{a}$ states an element-wise equality
between the components of the tuples $\boldsymbol{A}$ and $\boldsymbol{a}$.

\subsection{Relational Bayesian Networks}

\begin{figure}
  \begin{minipage}{0.38\textwidth}
    \begin{center}
    \includegraphics[scale=0.2]{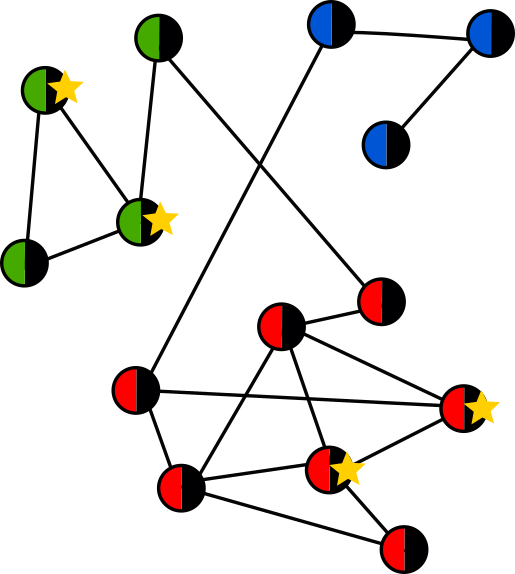}

    {\bf (A)}
    \vspace{3mm}
    
 \includegraphics[scale=0.4]{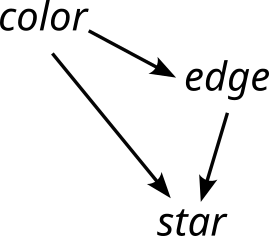}

    {\bf (B)}

    \end{center}
  \end{minipage}
  \hspace{\fill}
  \begin{minipage}{0.6\textwidth}
    
      $\emph{color}(v)$ := \ccsoftmax 5.1, 3.8, 3.4;
                    \vspace{1mm}
                    
      \begin{tabbing}
        $\emph{edge}(v,w)$ := \=  \ccwif $\emph{color}(v)=\emph{color}(w)$\\
	    	      \>	\ccthen 0.3\\
	    	      \>	\ccelse 0.04;
                    \end{tabbing}
                    \vspace{1mm}
                    
       \begin{tabbing}
                      $F_{\emph{triangle}}(v,w,u)$ := \= $ v\neq w \wedge v\neq u \wedge w\neq u $\\
                      \> $\emph{edge}(v,w) \wedge \emph{edge}(v,u) \wedge \emph{edge}(w,u)$;
                    \end{tabbing}
                    \vspace{1mm}

        \begin{tabbing}
          $\emph{star}(v)$ := \= \cccombine \= $ 0.3 \cdot \emph{color}(v) = \emph{red}$, \\
          \>\> $ -1.2 \cdot \emph{color}(v) = \emph{blue}$, \\
          \>\> $ 0.7 \cdot \emph{color}(v) = \emph{green}$, \\
          \>\> $ 0.5 \cdot$\=$\cccombine F_{\emph{triangle}}(v,w,u)$ \\
          \>\>\> \ccwith \ccsum \\
          \>\>\> \ccforall $w,u$ \\
          \> \ccwith \cclreg ;
        \end{tabbing}
                    \vspace{1mm}

        \begin{center}
          {\bf (C)}
        \end{center}
  \end{minipage}
  \caption{Elements of RBN representations. (A): a graph with categorical node attribute $\emph{color}\in\{\emph{red},\emph{green},\emph{blue}\}$ and Boolean node attribute $\emph{star}\in\{\emph{true},\emph{false}\}$. (B): directed acyclic attribute/edge dependency graph. (C): generative probabilistic RBN model. Language keywords in {\tt typewriter} font; user defined, domain specific names in \emph{italics}. 
  \label{fig:rbnillus}}
\end{figure}

\textit{Relational Bayesian Networks} (RBNs) \cite{rbn1997} are a generalization of classical Bayesian networks to relational domains. 
\new{A complete and fully formal introduction of the original framework is given in \cite{jaeger2001complex}. In the following we give a self-contained summary that provides all necessary details for what follows in this paper, introduces the support for categorical variables not provided in the original RBN formulation, but omits the details about the more powerful class of  "recursive" RBNs. 
}

\new{
An RBN is a generative probabilistic graph model that for a given signature ${\mathcal R}$ maps any finite set of vertices $V$ to a probability distribution $P$ over ${\mathcal G}(V,{\mathcal R})$.  $V$ being fixed, this amounts to a distribution over all possible
$\boldr=(\boldr_1,\ldots,\boldr_M)$, which we can also interpret as the values of $R_1,\ldots,R_M$, now seen as random variables. }
The model is defined by first arranging the relations $R_i$ in a directed acyclic graph defining the probabilistic dependencies, and then defining conditional probabilistic models for each relation given their parents in the dependency graph, \new{thus defining a joint distribution
according to the usual factorization
\begin{equation}
  \label{eq:factorization}
  P(R_1,\ldots,R_M)=\prod_{h=1}^M P(R_h|\emph{Pa}(R_h))
\end{equation}
with $\emph{Pa}(R_h)$ the parents of $R_h$ in the dependency graph.
\footnote{"Recursive" RBNs do not require an acyclic dependency graph at the relation level; it is sufficient that dependencies are acyclic at the ground atom level, which also enables auto-regressive relation models. We do not make use of such models in this paper, however.}
Generalizing (\ref{eq:factorization}), one can also allow for a designated special set ${\mathcal R}_{\emph{in}}$ of
\emph{input relations} that are assumed to be fixed (together with $V$), and that we can condition on for the probabilistic
models of ${\mathcal R}$, which then takes the form
\begin{equation}
  \label{eq:cond-factorization}
  P(R_1,\ldots,R_M|{\mathcal R}_{\emph{in}})=\prod_{h=1}^M P(R_h|\emph{Pa}(R_h),{\mathcal R}_{\emph{in}}).
\end{equation}
}
\new{
RBNs support model specifications of the form (\ref{eq:cond-factorization}) for the case where the $R_i\in{\mathcal R}$ are all
Boolean or categorical, whereas the relations in ${\mathcal R}_{\emph{in}}$ can also be numerical. Each conditional distribution is
further factorized as a product of conditional probabilities for ground atoms:
\begin{equation}
  \label{eq:atom-factorization}
  P(R_h|\emph{Pa}(R_h),{\mathcal R}_{\emph{in}})=\prod_{\boldi\in V^{\arity{R_h}}} P(R_h(\boldi)|\emph{Pa}(R_h),{\mathcal R}_{\emph{in}}).
\end{equation}
}

\new{RBNs define the conditional ground atom probabilities $P(R_i(\boldi)|\emph{Pa}(R_i),{\mathcal R}_{\emph{in}})$ in the formal language of \emph{probability formulas}.
  A probability formula is an expression $F(\boldv)$ in the formal language we describe in the following. It contains variables
  $\boldv=v_1,\ldots,v_k$ ranging over nodes.  For a specific tuple of nodes $\boldi=i_1,\ldots,i_k$ and given relational
  data $\boldr$ it evaluates to
  \begin{equation}
    \label{eq:probformeval}
    F[\boldr](\boldi)\in [0,1] \cup \bigcup_{l\geq 2}[0,1]^l \cup \Rset
  \end{equation}
  where the slightly redundant expression on the right highlights that the return value can be either a
  Bernoulli probability, a probability distribution for an $l$-state categorical variable, or an arbitrary
  real (without direct probabilistic semantics). 
A $k$-ary relation $R_i$ then is modeled by a probability formula $F$ with $k$ variables so that
  \begin{equation}
    \label{eq:probformforrel}
     P(R_i(\boldi)|(\emph{Pa}(R_i),{\mathcal R}_{\emph{in}})=\boldr) = F[\boldr](\boldi),
   \end{equation}
   and the return value of $F$ is a Bernoulli parameter if $R_i$ is Boolean, and a probability vector if $R_i$ is categorical.
 }

 \new{
   Figure~\ref{fig:rbnillus} shows an example that illustrates the construction and use of probability formulas.
   The underlying signature ${\mathcal R}$
here contains attributes \emph{color} and \emph{star}, and a single edge relation \emph{edge}. Figure~\ref{fig:rbnillus} (A) shows an
example graph with 14 nodes. The RBN model will define a probability value for this graph, within the sample spaces of all
14 node graphs for ${\mathcal R}$. The model is based on the factorization (\ref{eq:factorization}) defined by the graph in
Figure~\ref{fig:rbnillus} (B). Figure~\ref{fig:rbnillus} (C) shows probability formulas defining the conditional probabilities
$ P(R_h(\boldi)|\emph{Pa}(R_h))$. They are built using the following constructs:
}
\new{
\paragraph{Base formulas:} the language has three types of primitives: \emph{numerical constants},  \emph{atoms},
and \emph{equalities}. An atom
appearing in the formula defining $P(R_h|\emph{Pa}(R_h),{\mathcal R}_{\emph{in}})$ must belong to $\emph{Pa}(R_h)\cup{\mathcal R}_{\emph{in}}$,
so that its  value is given by the evaluation $\boldr$ of $\emph{Pa}(R_h)\cup{\mathcal R}_{\emph{in}}$. For Boolean atoms, the truth values
\emph{true},\emph{false} are evaluated as 1 and 0, respectively. 
A numeric atom directly evaluates to its real value.
An equality can be of the form $v=w$ for node variables $v,w$, or of the form $r(\boldv)=c$ or $r(\boldv)=r(\boldv')$,
where $r$ is a categorical relation and $c$ one of its possible values. Equalities also return 0 or 1, according to whether it evaluates
to false or true.
}
\new{
  \paragraph{Softmax function: } for any $l\geq 2$: if $f_1,\ldots,f_l$ are probability formulas, then {\tt SOFTMAX} $f_1,\ldots,f_l$ is
  a probability formula that defines a probability distribution for categorical relations with $l$ possible values. The formula defining
  the distribution for the categorical \emph{color} relation in Figure~\ref{fig:rbnillus} illustrates this construct with
  $f_1,f_2,f_3$ all being numerical constants. 
}
\new{
  \paragraph{Wif-then-else construct:} if $f_1,f_2,f_3$ are probability formulas, then ${\tt WIF}\ f_1\ {\tt THEN}\ f_2\ {\tt ELSE}\ f_3$ is
  a probability formula. It evaluates to $v_1v_2+(1-v_1)v_3$, where $v_i$ is the evaluation of $f_i$ ($i=1,2,3$). The Wif-then-else
  construct subsumes logical conjunction, disjunction and negation, which therefore we use freely as syntactic abbreviations in
  the specification of probability formulas. The example in Figure~\ref{fig:rbnillus} illustrates the use of Wif-then-else to
  define the probabilities of \emph{edge} relations between nodes $v,w$
  dependent on whether $v$ and $w$ have the same color, i.e., it defines the graph structure as a simple
  stochastic block model.
  }
  \new{
    \paragraph{Combination functions:} The most powerful RBN construct is the combination function. It subsumes full first-order quantification, and can be seen as a more general form of the message passing operation in GNNs. Syntactically, given arbitrary
    probability formulas $f_0,f_1,\ldots,f_m$, where $f_0$ is Boolean, i.e., always returning 0 or 1, a combination function
    formula is constructed as
    \begin{equation}
      \label{eq:combfunc}
      {\tt COMBINE}\ f_1,\ldots,f_m\ {\tt WITH}\ \langle\emph{operator}\ \rangle {\tt FORALL}\ \boldv\ {\tt WHERE}\ f_0(\boldv), 
    \end{equation}
  where $\langle\emph{operator}\ \rangle$ can be any combination (a.k.a. aggregation) function that takes a multiset of numbers,
  and returns a single number, or, specifically, a probability value. A combination function is evaluated, by evaluating the
  formulas $f_1,\ldots,f_m$ for all node tuples $\boldi$ for which $ f_0(\boldi)$ returns 1, and then applying
  $\langle\emph{operator}\ \rangle$ to the resulting multiset of values. In this paper we make use of two different
  combination operators: {\tt SUM} with the obvious meaning, and {\tt LOG-REG} (logistic regression) defined as
  \begin{displaymath}
    \mbox{\tt LOG-REG}(x_1,\ldots,x_n):=\sigma(\sum_{i=1}^n x_i)
  \end{displaymath}
  with $\sigma$ the sigmoid function (for simplicity we here write the multiset argument as a vector; by definition a combination function ignores the order
  of the argument, i.e., is permutation invariant). 
}
\new{
  The probability formula for the \emph{star} relation in Figure~\ref{fig:rbnillus} contains two nested combination functions. The inner
  function uses a Boolean sub-formula $F_{\emph{triangle}}$ that for nodes $i,j,k$  evaluates to 1 if these nodes form a triangle in the
  given $\emph{edge}$ relation. No condition $f_0$ here is imposed on the $w,u$ that are aggregated over, and the
  {\tt WHERE} clause is omitted. Evaluating this combination function for $v=i$ then returns the number of triangles that node
  $i$ is part of. The sub-formula $F_{\emph{triangle}}$ here is defined and displayed separately only to improve readability.
  This has no semantic implications and is equivalent to replacing the reference to $F_{\emph{triangle}}$ inside
  the $\emph{star}(v)$ formula
  by the right-hand side of the definition of $F_{\emph{triangle}}$. The outer combination function then applies the logistic regression
  function to four formulas without any further aggregation over other nodes (the {\tt FORALL} clause then can be omitted).
  Evaluating the whole formula e.g. for the node marked $i$ in Figure~\ref{fig:rbnillus}(A) returns $\sigma(0.3+0.5\cdot 3)=0.858$. 
}

The full joint model $P(R_1,\ldots,R_M)$ supports all kinds of conditional probability inferences, including predicting node attribute values given attributes of other nodes and a known \emph{edge} relation, or predicting the existence of an edge between two nodes, given observed node attributes and other observed edges.

\new{
As shown in  \cite{rbn1997}, RBNs can express any graph property expressible in first-order logic with equality. This makes RBNs
significantly more expressive than standard GNN architectures, which are limited to the 2-variable fragment of first-order logic
with counting \cite{barcelo2020logical}. The reasons for this added expressivity lie in the availability of  equality
conditions, and the ability to use probability formulas to define numerical features for arbitrary tuples of nodes. 
}

As indicated by the example in Figure~\ref{fig:rbnillus}, RBNs were conceived as being rich in (user-defined) structure, and sparsely parameterized by a relatively small number of interpretable, learnable probabilistic parameters. This is in sharp contrast to the neural representation paradigm, which is based on generic structures (i.e., neural architectures) parameterized with a large number of trainable weights, thus eliminating  the need for a separate and often difficult structure specification or structure learning phase. By integrating GNN model components into RBNs, we can exploit the latter's capabilities to
reduce structure learning (or feature discovery) to weight learning for those model components for which there is more  support from empirical data than from expert knowledge. 

\subsection{Graph Neural Networks}
\label{sec:gnn}
Graph Neural Networks (GNNs) are deep learning models that work on graph-structured data and have gained significant popularity in recent years due to their ability 
to leverage the expressive power of neural networks while learning in complex relational settings. 
At the heart of GNN architectures lies the \textit{message passing} paradigm, a general framework for updating
vector-valued node representations $\mathbf{h}(v)$ based on local neighborhood information.
These updates are performed over a fixed number of iterations, corresponding to GNN \emph{layers}.
\new{
  Node representations are initialized by \emph{input node feature vectors}
\begin{equation}
  \label{eq:h0init}
  \mathbf{h}^{0}(v) = (X_1(v),\ldots,X_l(v))
\end{equation}
where the $X_i$ are numeric or categorical node attributes
(following the conventions in the GNN literature, we here use $X$ to denote node features; this corresponds to a unary relation $R$ in the sense of Section \ref{sec:preliminaries}). }
The core message passing update of  a $d$-dimensional representation at layer $k$ to an
$m$-dimensional representation at layer $k+1$ is defined by
\begin{equation}
    \label{eq:gnn-update}
    \mathbf{h}^{k+1}(v) = \sigma\left( \sum_{u \in \mathcal{N}_v} W \mathbf{h}^{k}(u) \right),
\end{equation}
where $\mathbf{h}^{k}(u) \in \mathbb{R}^d$, W $\in \mathbb{R}^{m \times d}$ is a learnable weight matrix, 
$\mathcal{N}_v$ denotes the graph neighbors of $v$, $\sigma$ is a non-linear activation function such as the sigmoid, 
and $\mathbf{h}^{k+1}(v)\in \mathbb{R}^m$. To simplify our subsequent discussion in Section \ref{sec:GNN2RBN} we have
omitted from (\ref{eq:gnn-update}) an additional  functional dependence of $\mathbf{h}^{k+1}(v)$ on the previous
representation $\mathbf{h}^{k}(v)$ of the node $v$ itself,
which is present in almost all GNN architectures. This omission is for greater clarity only. The techniques we
introduce in the following sections fully accommodate also these slightly more complex update operations.

 In addition to the local neighborhood aggregations of  (\ref{eq:gnn-update}), also \emph{readout} operations can
be included that aggregate the representations of all nodes in the graph, and thereby compute a global representation of the graph from the node representations. 
A general class of GNNs that incorporates all these elements is the \emph{aggregate-combine-readout (ACR)} GNNs described in \cite{barcelo2020logical}.

\new{
The representation vectors $\mathbf{h}^{K}(v)$ obtained at the final layer $K$ are then used for node classification, graph classification, or link prediction tasks. In the case of node classification for a node \emph{Label} attribute, the final representation is turned
into a probability distribution over possible label values by a softmax function:
\begin{equation}
  \label{eq:gnncondprob}
  P(\emph{Label}(i)| X_1,\ldots,X_l,\emph{edge}):=\emph{softmax}(\mathbf{h}^{K}(X_1,\ldots,X_l,\emph{edge})(i))
\end{equation}
(here assuming that the dimension of $\mathbf{h}^{K}(v)$ equals the number of possible label values; alternatively, additional
standard neural network layers can transform $\mathbf{h}^{K}(v)$ into the required size). On the right-hand side of (\ref{eq:gnncondprob}) we make
explicit that  $\mathbf{h}^{K}(i)$ is a function of the input graph defined by $ X_1,\ldots,X_l,\emph{edge}$.
}
\new{
Equation (\ref{eq:gnncondprob}) elucidates that a GNN for node classification exactly computes  factors
$P(R_i(\boldi)|\emph{Pa}(R_i),{\mathcal R}_{\emph{in}})$ of (\ref{eq:atom-factorization}) in the case where
$R_i$ is an attribute, and where $\emph{Pa}(R_i)\cup {\mathcal R}_{\emph{in}} = \{ X_1,\ldots,X_l,\emph{edge}\}$.
Thus, a node classification GNN has the same semantics as a probability formula for a unary relation\footnote{The same
  correspondence holds between graph classification GNNs and probability formulas for 0-ary relations. The connection
  between GNNs for link prediction and binary probability formulas is a bit more complex, however. }. In the following
section we show that such a GNN can actually be directly encoded by a probability formula.
}

\subsection{GNN to RBN Encoding}
\label{sec:GNN2RBN}

The GNN update function (\ref{eq:gnn-update}) 
can be broken down to the computations at the level of individual vector components as follows: 
\begin{equation}
  \label{eq:scalar-mp}
    \mathbf{h}^{k+1}_i(v) = \sigma\left( \sum_{u \in \mathcal{N}_v} \sum_{j=1}^{d} W_{i,j} \cdot \mathbf{h}^{k}_j(u) \right),
\end{equation}
where subscripts $i\in \{1,\ldots, m\}$,   $j\in \{1,\ldots, d\}$ index individual vector and matrix components.

The key insight connecting GNNs with RBNs is that the scalar version (\ref{eq:scalar-mp}) of the message passing update 
can be written as a probability formula:

\begin{equation}
   \label{eq:probform-mp}
\begin{split}
    F_{k+1,i}(v) &:= \texttt{COMBINE } W_{i,1} \cdot F_{k,1}(u),\\
    &\quad\quad\quad\quad\quad \ldots \\
    &\quad\quad\quad\quad\quad W_{i,d} \cdot F_{k,d}(u)\\
    &\quad \texttt{WITH LOG-REG} \\ 
    &\quad \texttt{FORALL } u \\
    &\quad \texttt{WHERE } \emph{neighbor}(v,u).
\end{split}
\end{equation}
where  $F_{k,i}$ denotes the formula that computes $\mathbf{h}^{k}_i$. Note the dependence of $F_{k+1,\cdot}$ on $F_{k,\cdot}$, and that (\ref{eq:probform-mp})
thereby actually consists of a nested combination function construct.
Using  formulas $F_{k,i}$   as the main building blocks, a complete multi-layer
message passing operation can be
encoded as a probability formula. The same basic construct as shown in  (\ref{eq:probform-mp}) can also capture the additional
GNN elements mentioned above: a separate self-dependence of the update for $v$, and readout operations.
Finally, given the probability formulas computing the components of the final representation vectors
$\mathbf{h}^{K}(i)$, the application of the softmax in  (\ref{eq:gnncondprob}) is equivalent to the probability formula
\begin{equation}
  \label{eq:gnn-probform}
  \ccsoftmax  F_{K,1}(v),\ldots,F_{K,d}(v).
\end{equation}

This GNN-to-RBN compilation covers all GNN models in the ACR class of \cite{barcelo2020logical}.~\footnote{See Appendix \ref{sec:proof3.1} for a complete specification of the ACR class. Our current implementation only covers the case of sigmoid activation functions, but this is not a fundamental limitation.}

\new{
  \begin{restatable}{proposition}{reprequiv}
    \label{prop:reprequiv}
    Let $\mathcal{N}$ be any GNN in the class ACR whose input consists of node feature vectors $(X_1,\ldots,X_l)$ and an
    \emph{edge} relation, and whose output is the result of applying a \emph{softmax}
    function to the node representation computed by the final message passing layer. There exists a {\tt SOFTMAX}
    probability formula $F(v)$, such that for all graphs $(V,\boldr)$ over the signature
    $\mathcal{R}=\{X_1,\ldots,X_l,\emph{edge}\}$ and all nodes $i\in V$:
    \begin{equation}
      \label{eq:semequivalence}
      \emph{softmax}(\mathbf{h}^{K}(\boldr)(i)) = F[\boldr](i)
    \end{equation}
  \end{restatable}
  }

  The proof of this proposition is given in Appendix \ref{sec:proof3.1}.
  Proposition \ref{prop:reprequiv} establishes the semantic equivalence between a GNN and its probability formula
  encoding. 
The following proposition also shows that the learning objectives of the GNN and its RBN encoding are equivalent.

\begin{proposition}
  \label{theo:consistency}
  Let {\tt G} be a GNN for predicting a node attribute \emph{target} given node attributes $X_1,\ldots,X_l$ and edge relation
  \emph{edge}. Let $\boldsymbol{W}$ denote the weights of {\tt G}. Let {\tt R} be an RBN for a  signature
  $\mathcal{R}\supseteq \{X_1,\ldots,X_l,\emph{target},\emph{edge}\}$, in whose relational dependency graph
  $X_1,\ldots,X_l,\emph{edge}$ are the parents of
  $\emph{target}$, and in which the conditional distribution of \emph{target} is defined by (\ref{eq:gnn-probform}).
  Let $(\boldsymbol{U},\boldsymbol{W})$ denote the parameters of\ {\tt R}.

  Let $G_1,\ldots,G_N$ be a set of graphs over $\mathcal{R}$, and $\bar{G}_1,\ldots,\bar{G}_N$
  their reductions to $\{X_1,\ldots,X_l,\emph{target},\emph{edge}\}$.
  Then for a concrete setting $\boldsymbol{W}^*$ of $\boldsymbol{W}$ the following are equivalent: 
  \begin{description}
  \item[A.]  $\boldsymbol{W}^*$ minimizes cross-entropy loss of {\tt G} for training data  $\bar{G}_1,\ldots,\bar{G}_N$.
  \item[B.] There exists a setting $\boldsymbol{U}^*$ for $\boldsymbol{U}$ such that $(\boldsymbol{U}^*,\boldsymbol{W}^*)$
    maximizes the log-likelihood score of {\tt R} for training data  $G_1,\ldots,G_N$. 
  \end{description}
\end{proposition}

\begin{proof}
  The log-likelihood for {\tt R} given $G_1,\ldots,G_N$ decomposes according to the chain rule into a sum of
  conditional log-likelihoods for each relation in $\mathcal{R}$. Since the conditional distribution for $\emph{target}$ is defined
  by the encoding of {\tt G}, it only depends on the parameters $\boldsymbol{W}$, and the reductions $\bar{G}_h$ of the
  training graphs. Since the $\boldsymbol{W}$ does not occur in any other log-likelihood term than the one for
  $\emph{target}$, the local optimization of the conditional log-likelihood for $\emph{target}$  is part of the optimal solution
  for the full log-likelihood function. Finally, maximizing the log-likelihood for $\emph{target}$  is equivalent to
  minimizing the cross-entropy loss.
\end{proof}

We have formulated the encoding (\ref{eq:gnn-probform}) and Proposition \ref{theo:consistency} for the case of a node
classification GNN defining a unary relation in the RBN. The case for a graph classification GNN defining a relation of
arity 0 is completely analogous. In summary, we can integrate a node or graph classification GNN as the conditional
probability model for a relation of arity 1, respectively 0, such that the following consistency properties hold:

\begin{itemize}
\item Semantic equivalence: the output distribution computed by the GNN for the \emph{target} attribute based on input node features and given the graph \emph{edge} relation is equal to the
  conditional distribution of \emph{target} given the node features and  \emph{edge} relation in the generative model defined by the RBN.
\item Training equivalence: separate training of the GNN under the standard cross-entropy loss, and training the
  generative RBN under the maximum likelihood objective are equivalent (Proposition \ref{theo:consistency}).
\item Computational equivalence: computing output values for the \emph{target} attribute (``forward propagation''),
  or computing gradients for the loss/likelihood function (``backpropagation'') is performed by the same sequence of basic mathematical operations (addition, multiplication, exponentiation, \ldots) in the GNN and its RBN encoding.   
\end{itemize}

\subsection{RBN to GNN Interface}
\label{sec:interface}

GNN encodings in the native RBN language described in Section \ref{sec:GNN2RBN} lead to a very tight integration, but they come with two disadvantages: (1) the GNN-to-RBN compiler needs to be continuously extended to accommodate new GNN architectural elements  not yet covered (e.g., attention mechanisms, ReLu activation, \ldots). (2) While computationally equivalent in principle, the RBN encoding can be much slower in practice. The exact causes of this performance loss are difficult to identify, but two contributing factors are: the execution of vector/matrix operations at the scalar level, and the loss of GPU acceleration.

As an alternative to the compilation approach, we therefore also developed an interface version of the GNN-RBN integration.  Instead of model components of the form
\begin{equation}
  \label{eq:targetviacompile}
   \emph{target(v)}:= F(v)
 \end{equation}
 where $F(v)$ is a GNN encoding probability formula as provided by Proposition \ref{prop:reprequiv},
 the RBN now can  contain a declaration
\begin{equation}
  \label{eq:gnn-interface}
  \emph{target(v)}:= <\mbox{Interface to PyTorch GNN model}>;
\end{equation}
The $ <\mbox{Interface to PyTorch GNN model}>$ element is an extension of the RBN syntax. It contains a link to a PyTorch GNN model, and
the specification of a mapping between node attributes and relations as declared in the RBN model, and the node input features and edge
relations used for message passing at the GNN level. 
The training of the GNN and forward and backward propagations required at inference time are then executed in the original PyTorch implementation.
The interface supports PyTorch models for heterogeneous, multi-relational GNNs, which are often needed inside the rich
RBN modeling framework. 
\new{In order to obtain with (\ref{eq:gnn-interface}) the same functionality as with (\ref{eq:targetviacompile}), the interface has to implement the methods that are required for the abstract probability formula class in the \emph{Primula} RBN implementation. The most important of these methods just correspond 
to forward and backward propagation:}

\begin{quote}
\new{\textbf{\ttfamily evaluate($i$, \boldmath$\boldr$)}}
\vspace{2pt}

\new{Evaluates the ground atom \emph{target}($i$), under a given
instantiation of the input relations for the GNN.}
\new{
\begin{description}
\item[Parameters:] $i$, the node for which \ttfamily target \normalfont is queried ; $\boldr$,  values of the GNN input relations.
\item[Returns:] a value in $[0,1]$ if \ttfamily target \normalfont is a Boolean relation, or a probability vector in $[0,1]^l$
if it is categorical with $l$ values.
\end{description}}

\end{quote}

\begin{quote}

\new{\textbf{\ttfamily evaluateGradients($i$, \boldmath$\boldr$)}}
\vspace{2pt}

\new{Computes the gradient for the model parameters of the value that \texttt{evaluate} would return for the same arguments, obtained through the
wrapped model's own backpropagation mechanism.}

\new{\begin{description}
\item[Parameters:] as for \texttt{evaluate}.
\item[Returns:] the gradient of the model parameters $\boldsymbol{W},\boldsymbol{U}$, where, as in Proposition \ref{theo:consistency}, $\boldsymbol{W}$ and $\boldsymbol{U}$ denote the weights of the GNN, and the parameters of other RBN model components, respectively.
\end{description}}
\end{quote}

\new{The gradients returned by \texttt{evaluateGradients} are what allow a derivative computed inside the GNN to be
propagated further through the rest of the RBN's components.} More details of the implementation are described in \cite{primulatool}.

\section{Likelihood Graph and MAP Inference}
\label{sec:likgraph}

A key tool for learning and MAP inference in RBNs is the  \emph{likelihood graph}  \cite{JaegerICML07}: a computational graph for the likelihood of observed data given three types of unknown inputs: unknown model parameters (the main objects of interest in learning), query (MAP) atoms that are unobserved and whose most likely joint configuration one wants to infer, and other unobserved atoms that are not part of the query, but on whose values the observed data also depends.

In this section, we are focusing on the likelihood graph for MAP inference, and assume that all model parameters have already been learned. The MAP inference problem  is formalized as follows:  given observed data  ${\boldsymbol \Datoms}={\boldsymbol d}$, and a set of query MAP atoms ${\boldsymbol \MAPatoms}$, find the most probable joint configuration ${\boldsymbol m}$ for  ${\boldsymbol \MAPatoms}$:

\begin{equation}
  \label{eq:mapeq}
  \argmax_{\boldsymbol m}P({\boldsymbol \MAPatoms}={\boldsymbol m}|{\boldsymbol \Datoms}={\boldsymbol d})
  = \argmax_{\boldsymbol m}P({\boldsymbol \MAPatoms}={\boldsymbol m},{\boldsymbol \Datoms}={\boldsymbol d})
  =\argmax_{\boldsymbol m}\sum_{\boldsymbol o}P({\boldsymbol \MAPatoms}={\boldsymbol m},{\boldsymbol \Oatoms}={\boldsymbol o},{\boldsymbol \Datoms}={\boldsymbol d})
\end{equation}
where ${\boldsymbol \Oatoms}$ contains all atoms not in ${\boldsymbol \MAPatoms}$ or ${\boldsymbol \Datoms}$. In the special case
${\boldsymbol \Oatoms}=\emptyset$ this becomes MPE (most probable explanation) inference.

A schematic picture of a likelihood graph for MAP inference is shown in Figure~\ref{fig:lgraph} on the left:  the root node $\Pi$ represents the joint probability of observed data ${\boldsymbol \Datoms}={\boldsymbol d}$  and a candidate MAP configuration ${\boldsymbol \MAPatoms}={\boldsymbol m}$ of query atoms. If any of the atoms in ${\boldsymbol \Datoms}$ or
${\boldsymbol \MAPatoms}$ depend on the values of other unobserved atoms ${\boldsymbol \Oatoms}$, then
the root node represents the marginal likelihood
$\sum_{{\boldsymbol \Oatoms}}P({\boldsymbol \Datoms},{\boldsymbol \MAPatoms},{\boldsymbol \Oatoms})$. The root has one child for each ground atom in ${\boldsymbol \Datoms}\cup{\boldsymbol \MAPatoms}\cup{\boldsymbol \Oatoms}$.
Each of these children computes the conditional probability of its atom given the relational data it depends on. The final likelihood computed at $\Pi$ is simply the product of these probabilities.\footnote{It is important to note that this does \emph{not} imply any kind of independence or naive Bayes assumption; we just use the completely general factorization according to the chain rule.}
The likelihood graph has one input node for each MAP atom in  ${\boldsymbol \MAPatoms}$, and each unobserved atom in
${\boldsymbol \Oatoms}$.  An input for a ${\boldsymbol \MAPatoms}$ atom consists of a single candidate value (indicated in Figure \ref{fig:lgraph} by the node coloring of MAP atom input nodes), whereas an input for an ${\boldsymbol \Oatoms}$ atom consists of a sample of values (indicated by arrays of colorings in Figure \ref{fig:lgraph}).
Between the input nodes and the children of $\Pi$ lie intermediate computation nodes that
correspond to the functions defined by sub-formulas of the RBN.
There are no input nodes for the data atoms ${\boldsymbol \Datoms}$: their fixed values are hard-coded into the functions computed at the children of $\Pi$, and the intermediate nodes.
\new{Since all intermediate nodes of a likelihood graph correspond to groundings of the probability formulas and their constituent sub-formulas
  in the RBN model, one obtains $O(K |V|^f)$ as an upper bound on the number of likelihood graph nodes, where $K$ is the number of (sub-)formulas
in the model, $V$ is the given vertex set, and $f$ is the maximal number of free variables in the probability (sub-)formulas.\footnote{\new{The free variables of a probability formula are defined as in first-order logic: all variables except those that are bound by a \ccforall operator in a
  combination function \cite{jaeger2001complex}.}} A likelihood graph node representing a ground sub-formula $F(\boldi)$ is connected to the nodes
representing ground sub-formulas needed to evaluate  $F(\boldi)$. The worst case here is that $F$ is a combination function (\ref{eq:combfunc})
where the condition $f_0(\boldv)$ is satisfied by all $\boldi\in V^{|\boldv|}$. Then the number of required evaluated sub-formulas is
$m|V|^{|\boldv|}$. Letting $g$ be the maximal number $|\boldv|$ of variables quantified over in any \ccforall clause, and using
$m\leq K$ we obtain $O(K^2 |V|^{f+g})$ as an upper bound on the number of edges in the likelihood graph.}

The likelihood graph supports the following basic computations required for MAP inference:
\begin{itemize}
\item Computation of the likelihood value given a current configuration of the ${\boldsymbol \MAPatoms}$ atoms, and a sample of values used to approximate the intractable sum $\sum_{\boldsymbol o}$.
\item \new{Re-sampling of values ${\boldsymbol o}$ for  ${\boldsymbol \Oatoms}$ by Gibbs sampling conditional on the
  current values of the MAP atoms ${\boldsymbol \MAPatoms}$. Specifically, the likelihood graph maintains a given number \emph{nc} of parallel Gibbs sampling chains. A resampling operation generates \emph{ws} new samples in each chain. Both the number of chains and \emph{ws} are configurable parameters. Figure \ref{fig:lgraph} illustrates the case of \emph{nc}=2, \emph{ws}=4.}
\end{itemize}
\new{Both of these operations are linear in the number of edges of the likelihood graph. However, the likelihood graph also supports
\emph{sub-linear} local updates:} when the ${\boldsymbol \MAPatoms}$ configuration for which the likelihood is to be computed differs from a configuration for which the likelihood graph has already been evaluated only at a single atom $\MAPatoms$, then only a re-evaluation of the nodes that are ancestors of the $\MAPatoms$ input node in the graph is necessary. Similarly, re-sampling values for an atom $\Oatoms\in {\boldsymbol \Oatoms}$ only requires re-evaluations of ancestors of $\Oatoms$. 

\begin{figure}
  \centering
  \includegraphics[width=\textwidth]{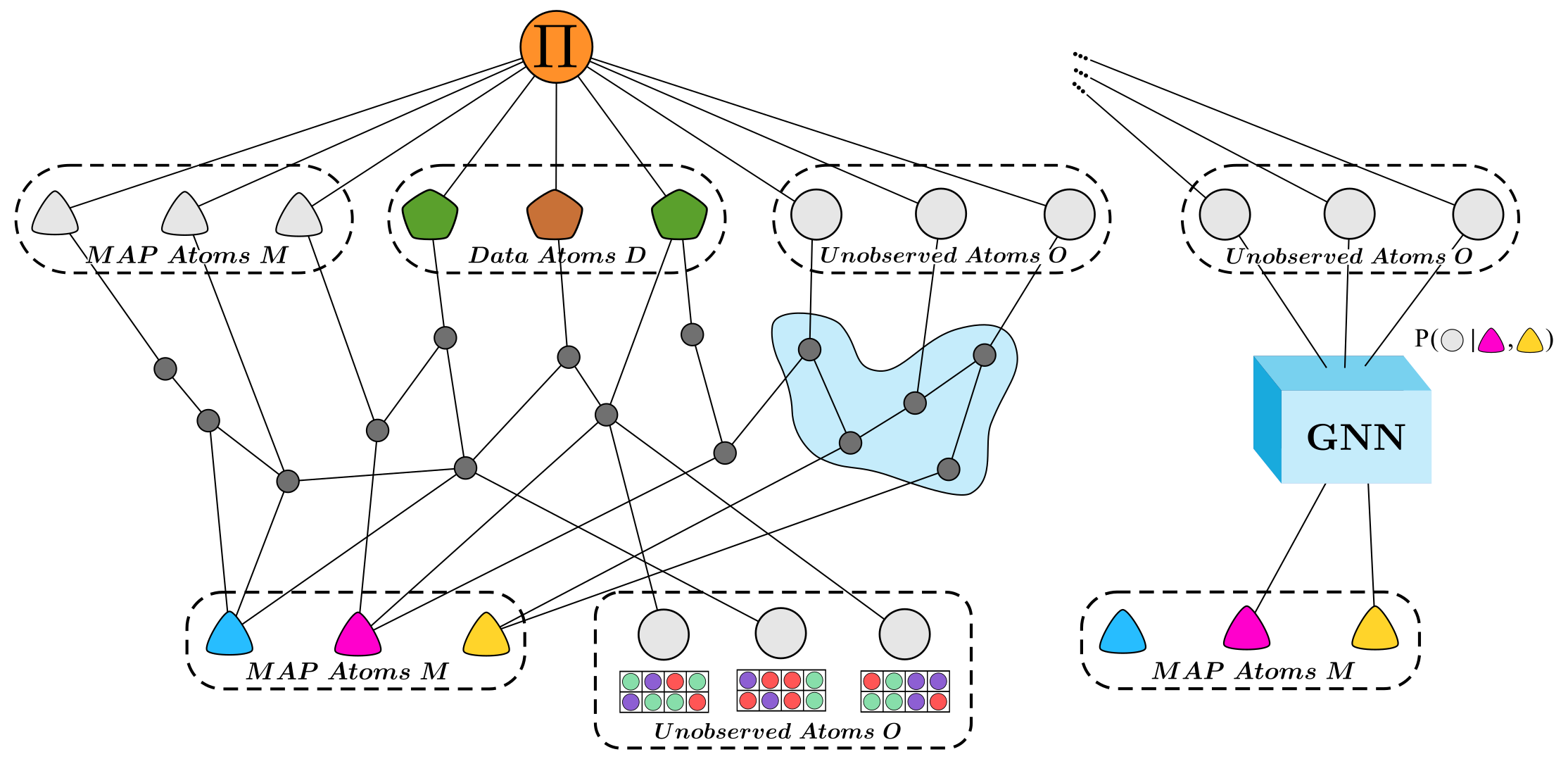}
  \Description{Schematic of the likelihood graph showing the standard RBN version and the version with a GNN module.}
  \caption{The likelihood graph. Left: standard RBN version; Right: replacement of (blue shaded) part of the computation graph by a GNN module. The coloring of atom nodes symbolizes their values: for MAP atoms, the values according to a current configuration; for unobserved atoms, the values given by a random sample. \label{fig:lgraph}}
\end{figure}

We solve the optimization problem (\ref{eq:mapeq}) using a simple greedy algorithm shown in Algorithm \ref{alg:map}. Starting with an initial random configuration ${\boldsymbol \MAPatoms}={\boldsymbol m}$, it determines the MAP atoms $\MAPatoms$ that will lead to the highest gain in likelihood when their current value $m$ is changed to a different value $m'$ (as determined by the {\sc Score} subroutine).
Simultaneously, a batch of at most $b$ atoms are flipped to their most probable value (l.8).
Given the new settings of the MAP atoms, unobserved atoms are resampled (l.9), and MAP atoms whose distribution may have been affected by the changes of l.8 are re-scored (l.10). 
When no single flips of MAP atoms lead to a likelihood improvement (l.12), then in a limited recursive lookahead step the atoms that incur the least loss of likelihood are tentatively flipped (l.14), and the greedy search is called recursively but only operating on those atoms whose likelihood can have been negatively affected by the tentative flips (l.15). For this purpose the algorithm takes a subset $\boldsymbol{\MAPatoms}'$ of MAP atoms as input (initially set to the set $\boldsymbol{\MAPatoms}$ of all MAP atoms).

\begin{algorithm}[tb]
\begin{algorithmic}[1]
\Procedure{MAP}{Likelihood graph $\mathcal{G}$, subset of MAP atoms $\boldsymbol{\MAPatoms}' \subseteq \boldsymbol{\MAPatoms}$, lookahead depth $d$, batchsize $b$, stopping criterion $t$}
\State Set initial random configuration ${\boldsymbol \MAPatoms}'={\boldsymbol m}$
\State Compute $\MAPatoms.\emph{score}= ${\sc Score}$(\MAPatoms)$ for all $\MAPatoms\in {\boldsymbol \MAPatoms}'$
\LComment{\color{darkgray}From now on maintain $\boldsymbol{\MAPatoms}'$ as a set ordered according score values; $\boldsymbol{\MAPatoms}'[:k]$ denotes the subset containing the $k$ highest scoring elements. }
\While{($\max\{\MAPatoms.\emph{score}|\MAPatoms\in{\boldsymbol \MAPatoms}'\}>0$ or $d>0$) and $t$ not satisfied}
\State $b'\gets \min\{b,|\{\MAPatoms|\MAPatoms.\emph{score}>0\}  |\}$
\If{$b'>0$}
\State Flip the $\MAPatoms\in\boldsymbol{\MAPatoms'}[:b']$ to their most probable value $\MAPatoms.\emph{maxval}$
\State Resample the unobserved ${\boldsymbol \Oatoms}$
\State Recompute {\sc Score}$(\MAPatoms')$ for all $\MAPatoms'\in\sibling{\boldsymbol{\MAPatoms'}[:b']}$
\LComment{\color{darkgray} $\sibling{\boldsymbol{\MAPatoms'}[:b']}$ denotes the $\MAPatoms'\in\boldsymbol{\MAPatoms}'$ for which there exists an output ground atom node that depends on both $\MAPatoms$ and $\MAPatoms'$.}
\Else 
\State Let $l,\boldsymbol{m}$ be the current likelihood and configuration of $\boldsymbol{\MAPatoms}$ 
\State Flip the $\MAPatoms\in\boldsymbol{\MAPatoms'}[:b']$ to their most probable value $\MAPatoms.\emph{maxval}$
\State $\boldsymbol{m}',l' = ${\sc MAP}($\mathcal{G},\sibling{\boldsymbol{\MAPatoms'}[:b']}\setminus\boldsymbol{\MAPatoms'}[:b'] ,d-1,b, t)$
\If{ $l'<l$} revert to $\boldsymbol{m}$
\EndIf
\EndIf
\EndWhile
    \State \Return current configuration ${\boldsymbol m}$ and its likelihood value $l$.
    \EndProcedure
    \medskip
\Procedure{Score}{MAP atom $\MAPatoms$  with current setting $\MAPatoms=m$}
\For{all possible values $m'$ of $\MAPatoms$}
\State compute the change in likelihood value when changing  $\MAPatoms=m$ to $\MAPatoms=m'$
\State Set $\MAPatoms.\emph{score}$ and $\MAPatoms.\emph{maxval}$ to the maximal increase in likelihood value and the corresponding value $m'$, respectively
\EndFor
\EndProcedure
\end{algorithmic}
\caption{MAP Inference \label{alg:map}}
\end{algorithm}

The description given so far assumed a pure RBN model, possibly containing compiled GNN components of the form (\ref{eq:targetviacompile}).
When a GNN is
integrated by an interface (\ref{eq:gnn-interface}), then a single node linked to the underlying GNN is added to the
likelihood graph. This node can be queried for the value the GNN returns for any specific node $i$.
Figure~\ref{fig:lgraph} illustrates this structural change for a scenario in which all unobserved atoms are
in the \emph{target} attribute modeled by the GNN. Instead of a sub-graph containing all ground probability formulas and their
sub-formulas for the atoms $\emph{target}(i)\in \boldsymbol{\Oatoms} $, the likelihood graph now contains a single GNN node.
Evaluations of the new GNN nodes require a forward propagation in the GNN, which is no longer a constant time operation
(independent of graph size). Thus, the linear worst-case time bound for the basic likelihood graph computations no longer applies.
More significantly, however, the GNN nodes also limit the possibilities to obtain sub-linear computation times in local updates.
This is because the GNN node acts as a bottleneck connecting many of the nodes below to many of the nodes above it, thereby
leading to  greatly increased ancestor-of relationships between the nodes.

The two different GNN-RBN integration methods described in Sections~\ref{sec:GNN2RBN} and \ref{sec:interface}, in conjunction with the likelihood graph for RBN parameter learning and MAP inference, establish an integrated framework in which GNN operations can be performed either entirely by native RBN operations or by calls to an external GNN implementation. Using Kautz's taxonomy of neural-symbolic systems~\cite{Kautz_2022}, we can position the compilation of a GNN into an RBN within the taxonomy as a middle ground between Neuro[Symbolic], where a neural network performs logical reasoning during execution, and Symbolic[Neuro], where a symbolic solver is the primary system and the neural component acts as a subroutine. With this new direct integration of external GNN models as RBN components, this work clearly falls into the Symbolic[Neuro] category.

We  defer a brief investigation into the computational tradeoffs of the two approaches to Section \ref{sec:runtime}. 
In \cite{pojer2024generalized} the GNN to RBN compilation approach was applied  to two different tasks:
one illustrating the ability to ``invert'' the inference direction of the GNN model by computing conditional probabilities for node
(input) attributes given observed class labels, and one illustrating the use of MAP inference to obtain model-level explanations
of GNN graph classifiers. 
In the following, we introduce two new applications of MAP inference, each of which also introduces new benchmark
tasks and datasets.

\section{Application: Collective Node Classification}
\label{sec:nodeclass}
A basic limitation of GNNs for node classification lies in the independence of the predictions for different nodes, which does not directly support modeling homophilic or heterophilic structures in the label distribution \cite{qu2019gmnn}. It is commonly assumed that GNNs perform better in homophilic than heterophilic scenarios \cite{yan2022two,loveland2024performance,luan2024heterophilic}, although several authors have also cautioned against an over-interpretation of the existing theoretical or empirical evidence for this observation, for example because differences between global and local homophily levels are not sufficiently taken into account \cite{loveland2024performance}, or empirical studies use heterophilic datasets where node features alone provide a strong basis for classification, thus obscuring the role played by heterophilic label distributions \cite{luan2024heterophilic}. Therefore, the use of synthetic datasets in which homo/heterophilic properties can be cleanly isolated from other confounding factors that impact classification performance has been advocated \cite{loveland2024performance,luan2024heterophilic}.

In this section we show how MAP inference in an integrated GNN-RBN model can be used to modulate the purely feature-based and independent GNN predictions in order to take homo/heterophilic distribution patterns into account.

\subsection{MAP Inference for Collective Node Classification}
\label{sec:mapfornode}

Consider a GNN $\mathcal{N}$ trained for predicting a node label $Y$ given node attributes $\boldsymbol{A}$.
The predicted label for node $v$ will then depend on attribute values of $v$ and other nodes $v'$, but not
on the predictions at other nodes, and thus cannot directly incorporate  homo-/heterophily objectives.
We therefore embed a standard node classification GNN into a GNN-RBN model that contains additional predicates
expressing homo-/heterophilic properties, such
that we can combine the independent predictions provided by the GNN with constraints on the overall homo/heterophilic structure
of the solution. These constraints will be expressed by conditioning the node label distribution on the additional predicates. 

We aim to capture \emph{local} homophily structures that can vary across the graph. Any labeling $\boldsymbol{y}$ of
the nodes defines the local homophily  $\emph{LH}_{\boldsymbol{y}}(v)\in [0,1]$ at node $v$
as the proportion of edges connecting $v$ to nodes with the same label as $v$. Let, now, $\boldsymbol{y}^*$ denote the true
(but for test nodes unknown) node labeling, and $\hat{\boldsymbol{y}}$ a predicted labeling. If $\hat{\boldsymbol{y}}$ is
an accurate prediction for  $\boldsymbol{y}^*$, then, in particular, the local homophily values must match, i.e.
$|\emph{LH}_{\boldsymbol{y}^*}(v)-\emph{LH}_{\hat{\boldsymbol{y}}}(v)|$ should be small for all $v$. To enforce this homophily matching
objective, we introduce a Boolean node attribute $\overline{LH}(v)$ and define the probability of
$\overline{LH}(v)$ to be true as the function of $\emph{LH}_{\boldsymbol{y}^*}(v)-\emph{LH}_{\hat{\boldsymbol{y}}}(v)$
shown in Figure~\ref{fig:diff_hom}. This function is the product of two logistic regression functions and can be
represented by an RBN probability formula. It is designed to yield a
smooth increase of probability from the extreme values $\emph{LH}_{\boldsymbol{y}^*}(v)-\emph{LH}_{\hat{\boldsymbol{y}}}(v) \in\pm 1$
to the optimal value $\emph{LH}_{\boldsymbol{y}^*}(v)-\emph{LH}_{\hat{\boldsymbol{y}}}(v) = 0$.

Let $\boldsymbol{LH_{\boldsymbol{y}^*}},\boldsymbol{A},\boldsymbol{\overline{LH}} $ and $\hat{\boldsymbol{Y}}$
denote the true local homophily values, input attributes, $\overline{LH}(v)$ values, and predicted labels, respectively,
for all nodes. Then $\mathcal{N}$ and our model for $\overline{LH}$
define a conditional distribution
\begin{equation}
  \label{eq:isingmap}
  P(\hat{\boldsymbol{Y}},\boldsymbol{\overline{LH}}|\boldsymbol{LH}_{\boldsymbol{y}^*} ,\boldsymbol{A})=
  P_{\mathcal N}(\hat{\boldsymbol{Y}}|\boldsymbol{A})P(\boldsymbol{\overline{LH}}|\hat{\boldsymbol{Y}},\boldsymbol{LH}_{\boldsymbol{y}^*} )
\end{equation}
where $ P_{\mathcal N}$ is the label prediction distribution 
defined by the GNN. Having trained ${\mathcal N}$ in a conventional way (i.e., using cross-entropy loss for the label prediction on the training nodes), we condition (\ref{eq:isingmap}) on $\boldsymbol{\overline{LH}}=\boldsymbol{true}$,
and perform MAP inference for $\hat{\boldsymbol{Y}}$.

\begin{figure}[h]
    \centering
    \includegraphics[width=0.3\linewidth]{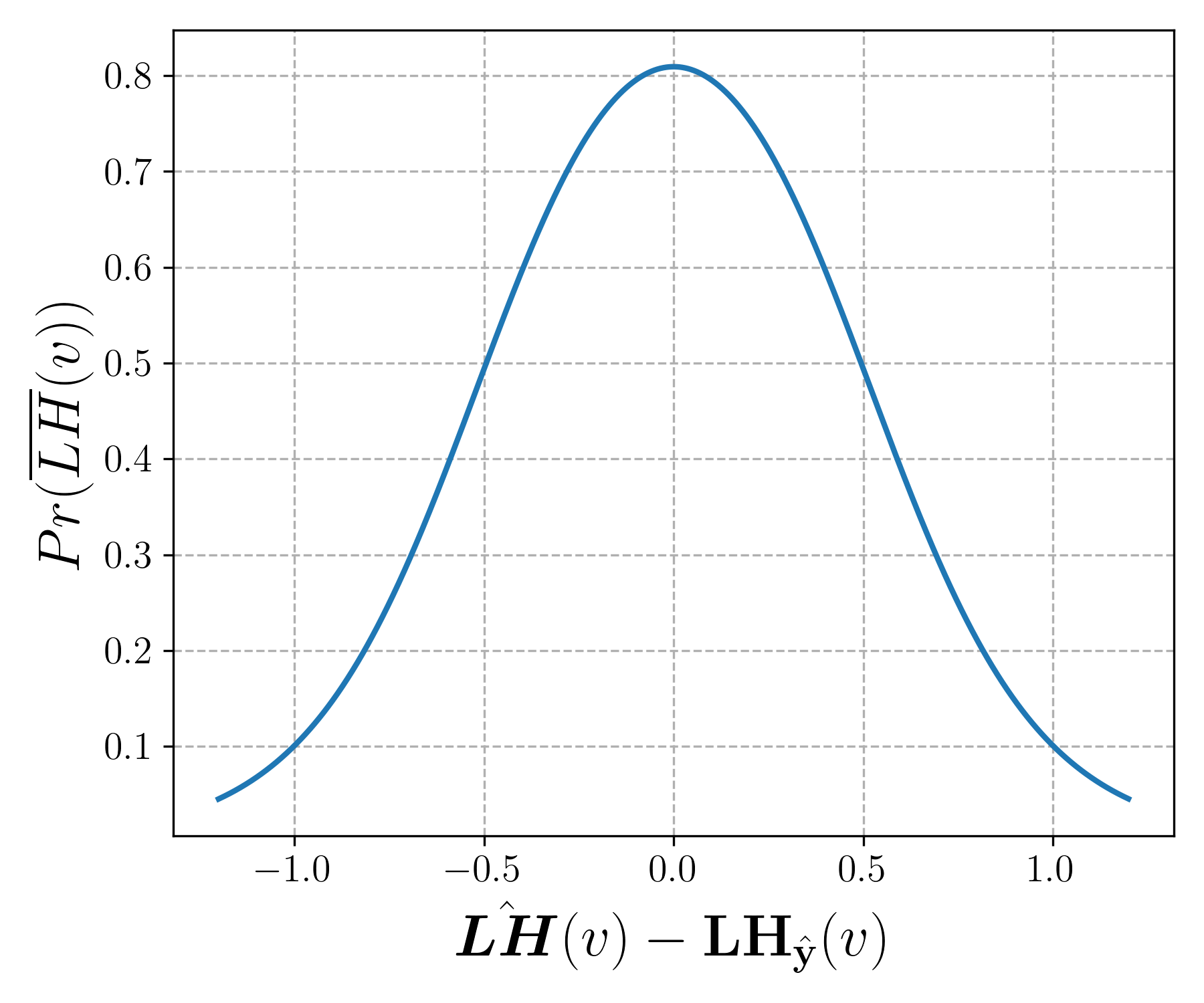}
    \caption{Probability of $\overline{LH}=\emph{true}$ as a function of the difference between target and current local homophily level. }
    \label{fig:diff_hom}
\end{figure}

Note that even though in $ P_{\mathcal N}(\hat{\boldsymbol{Y}})$ the variables $\hat{Y}(v)$ for the different nodes are independent,
conditioning on  $\boldsymbol{\overline{LH}}=\boldsymbol{true}$ induces dependencies, thereby turning this into a collective classification approach.
Thus, our objective in MAP inference for $\hat{\boldsymbol{Y}}$ is to both maximize the probability that the GNN assigns
to $\hat{\boldsymbol{Y}}$ based on  observed attribute values $\boldsymbol{A}=\boldsymbol{a}$, and the probability of the constraints
$\overline{LH}(v)=\emph{true}$. The form of our model for $\overline{LH}$ shown in Figure \ref{fig:diff_hom} leads to a
specific tradeoff between the potentially conflicting goals of maximizing
$ P_{\mathcal N}(\hat{\boldsymbol{Y}}|\boldsymbol{A})$ and $P(\boldsymbol{\overline{LH}}|\hat{\boldsymbol{Y}},\boldsymbol{LH}_{\boldsymbol{y}^*} )$.
The parameters that determine the shape of this function can
be seen as hyperparameters of our approach, and could be optimized by standard hyperparameter optimization routines. We do not pursue this optimization in the current paper, and report results only for the default model of Figure \ref{fig:diff_hom}.

Evaluating and maximizing (\ref{eq:isingmap}) exactly would require access to the true values $\boldsymbol{LH}_{\boldsymbol{y}^*} $, which,
in reality, are unknown.  We therefore approximate the true values by estimates $\hat{\boldsymbol{LH}}$. We obtain our estimates by a \emph{homophily propagation} method that is inspired by standard \emph{label propagation}: seeded by the (partly) known local homophily values of labeled nodes, we iteratively update estimated \emph{LH} values by averaging
estimates of neighboring nodes. The process not only updates the local homophily values for unlabeled nodes, but also those for labeled nodes with some unlabeled neighbors. A detailed description of the algorithm is contained in the appendix.

\subsection{Ising: Data} 
\label{sec:isingdata}

We propose a special class of synthetic benchmark datasets that allows us to finely tune and analyze the three relevant factors: the  homophily/heterophily level in the node label distribution, the  informativeness of node attributes, and 
the informativeness of node neighborhood data.

We generate random node labelings according to the Ising model from statistical physics. In this model, the (undirected) graph structure consists of a (fixed) regular square $n\times n$ grid of nodes, where a node is connected to its (at most 4) upper, lower, left, and right neighbors. Nodes have a binary class label $+1$ or $-1$ (representing positive or negative electromagnetic spin). The probability of  
a node labeling $\boldsymbol{y}\in \{-1,+1\}^{n \times n}$ is given by
$
    P(\boldsymbol{y}) = \frac{1}{Z} e^{\phi(\boldsymbol{y})}
$
with
\begin{equation}
\label{eq:phi}
    \phi(\boldsymbol{y})=\sum_{v} \boldsymbol{y}(v)\left(F\cdot f(v) + H\cdot \sum_{u\in {\mathcal{N}}_v} \boldsymbol{y}(u)\right)
\end{equation}
where $f(v)\in \Rset$ is a scalar node feature (the "external magnetic field" value at $v$), $F,H$ are parameters, and $Z$ is a normalizing constant. The $H$ parameter controls the bias towards more homophilic ($H >0$) or heterophilic ($H < 0$) node labelings. The $F$ parameter controls the importance of the node feature $f(v)$. With $F=0$ label probabilities only depend on the number of neighbors with the same label.
For generating labelings $\boldsymbol{y}$, we sample from the model (\ref{eq:phi}) with
$f(v)$ a linear function  that increases from the top left to the bottom right corner of the grid.

For a given sampled $\boldsymbol{y}$, we create two versions of train/test data: in the first version, nodes are
equipped as an input attribute $A$ with the actual value $A(v)=F\cdot f(v)$ used in sampling the node labels.
In this case, any linear aggregation of the $A(u)$ values of $v$'s neighbors $u$ can also be obtained as a linear
function of $A(v)$ alone. This means that GNNs here can not exploit their ability to aggregate node
neighborhood data. We therefore create a second version in which nodes receive as an attribute
randomly perturbed values $\tilde{A}(v):=A(v)+N(0,F\cdot\sigma)$ of the first version. The variance of the
noise here is scaled with $F$ in order to calibrate the amount of noise to the value range of $A(v)$.
A GNN can now benefit from averaging $\tilde{A}(u)$ over the neighbors $u$ of $v$,
which will be a better approximation of the underlying $F\cdot f(v)$ than $\tilde{A}(v)$ alone
(and hence a better predictor for $\boldsymbol{y}$). 

In summary, the data generation model allows us to calibrate:
\begin{itemize}
\item Homo/Heterophily of the label distribution: by the parameter $H$;
\item Influence of the node feature $f(v)$ on the label distribution: by the parameter $F$;
\item Informativeness of node neighbor data: by the construction of node input attributes with/without noise\footnote{We here only consider the two distinct with/without noise options; finer variations obtained by varying the noise model and/or the underlying magnetic field function $f(v)$ can obviously be
    considered.}.  
\end{itemize}

Table \ref{tab:ising_data_setup} shows  data generated for a $32\times 32$ grid under different settings of the $H,F$ parameters, \new{each column $i_{1\dots5}$ represents a different setup}. 
Row (A) shows the sampled graph with node labels +1 (yellow) and -1 (purple). 
The values for the positive vs. negative class ratio for both train and test nodes, as well as the global homophily values 
(defined as the proportion of edges connecting nodes of the same class) for the graphs, are listed above the plots. 
Row (B) visualizes the value of the unperturbed feature value $F\cdot f(v)$ of the nodes. 
When $F>0$, then the feature $F\cdot f(v)$ is just a scaled version of the same function, and, in principle, equally informative. 
For $F=0$, the feature becomes the uninformative constant 0 (green). Row (C) illustrates the noisy version $F\cdot\tilde{f}$ of the node feature. 

\begin{table}
    \centering
    \caption{Ising data: parameters and visualizations.
    \label{tab:ising_data_setup}}
    \resizebox{\textwidth}{!}{%
    \begin{tabular}{llccccc}
      \toprule
      & Setup & $i_1$ & $i_2$ & $i_3$ & $i_4$ & $i_5$ \\
      \midrule
      & $H$  
                 & 0.5
                 & -0.5
                 & -0.4
                 & -0.7
                 & 0.9 \\
      & $F$
                 & 0.0
                 & 0.0   
                 & 0.1
                 & 0.3
                 & 0.05 \\
      & \texttt{Class Ratio (train)}
                 & 0.55    
                 & 0.49
                 & 0.49     
                 & 0.48
                 & 0.45 \\
      & \texttt{Class Ratio (test)}
                 & 0.56    
                 & 0.53   
                 & 0.49     
                 & 0.47
                 & 0.48 \\
      & \texttt{Global Homophily (h)}
                 & 0.92  
                 & 0.07 
                 & 0.39
                 & 0.56
                 & 0.97 \\
    
    \raisebox{3.3\height}{(A)} & \raisebox{3.7\height}{\texttt{Node Labels}}
                & \includegraphics[width=0.14\textwidth]{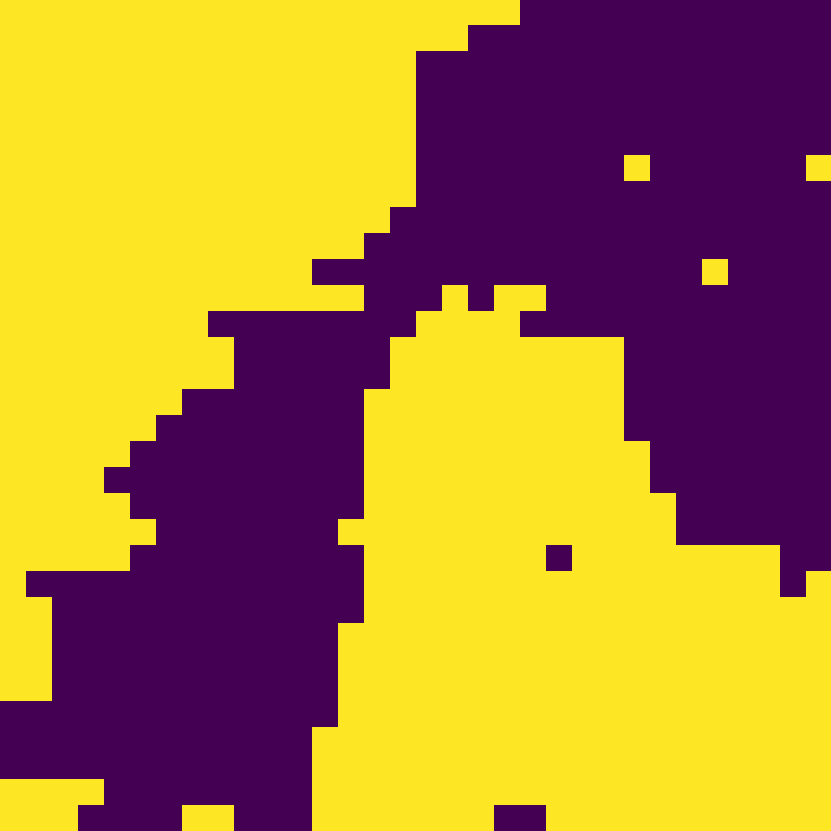}
                & \includegraphics[width=0.14\textwidth]{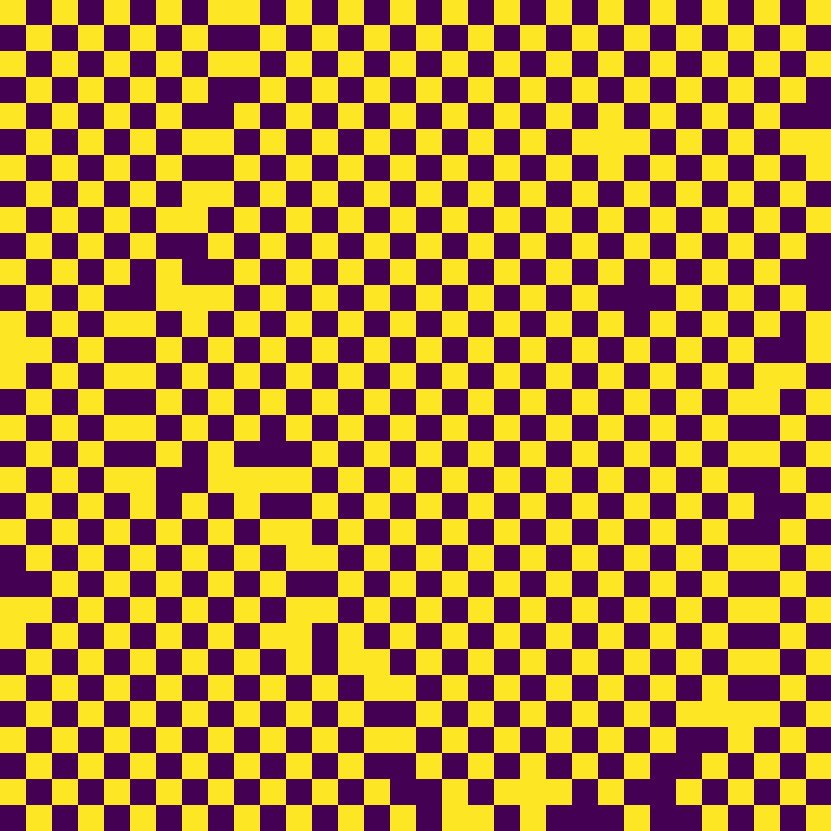}
                & \includegraphics[width=0.14\textwidth]{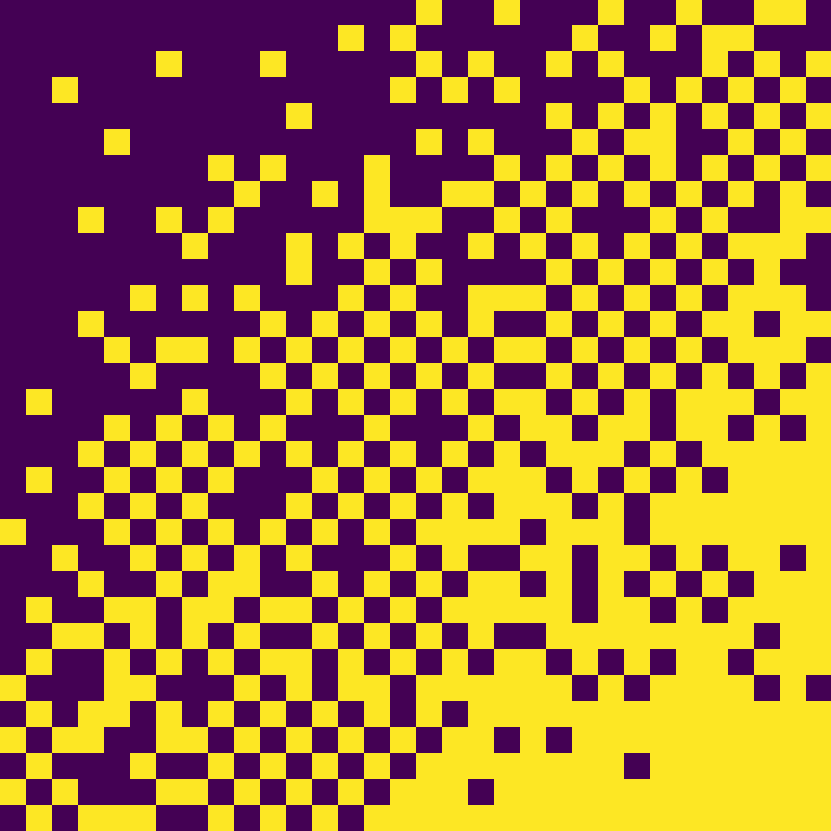}
                & \includegraphics[width=0.14\textwidth]{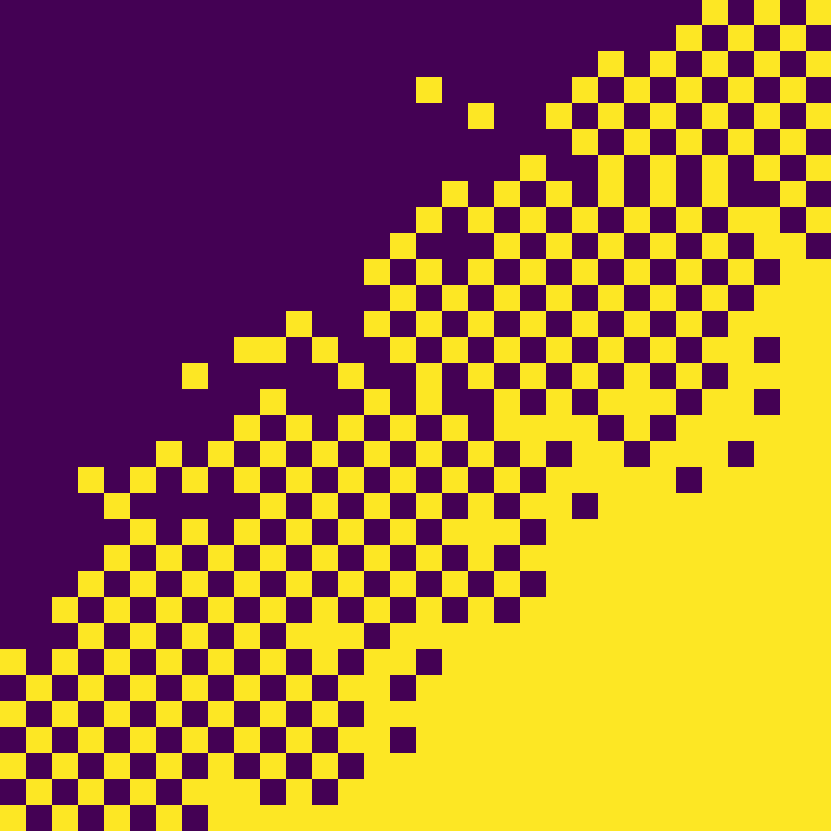}
                & \includegraphics[width=0.14\textwidth]{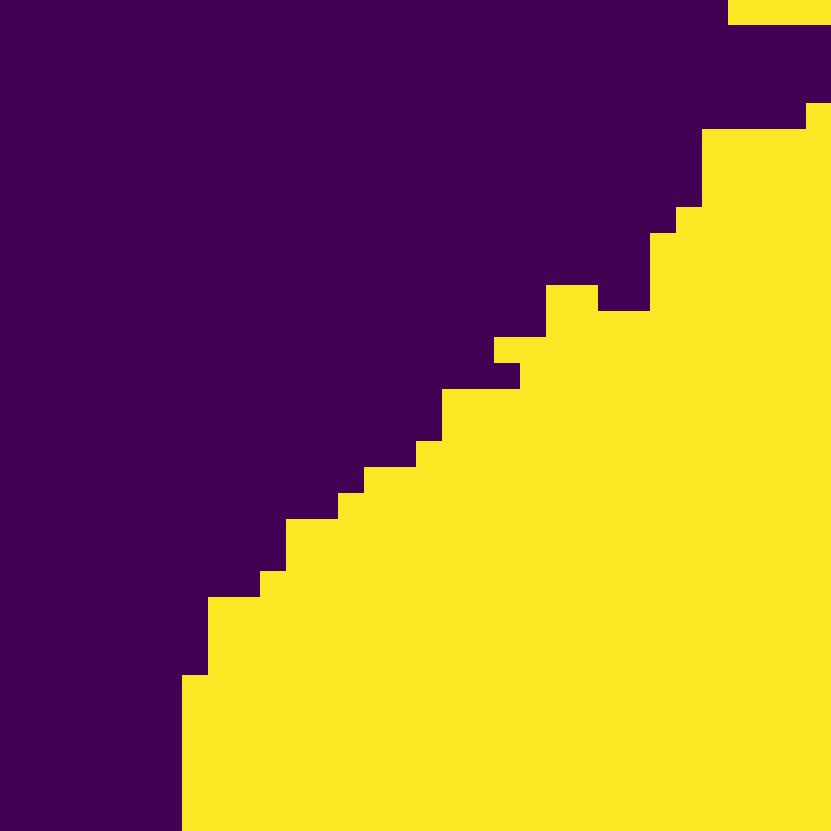}\\
    \raisebox{3.3\height}{(B)} & \raisebox{3.7\height}{\texttt{Attributes}}
                & \includegraphics[width=0.14\textwidth]{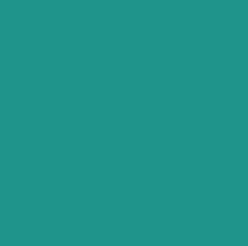}
                & \includegraphics[width=0.14\textwidth]{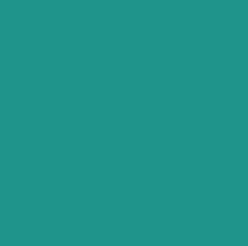}
                & \includegraphics[width=0.14\textwidth]{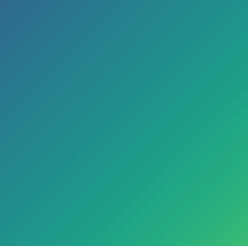}
                & \includegraphics[width=0.14\textwidth]{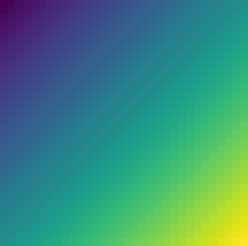}
                & \includegraphics[width=0.14\textwidth]{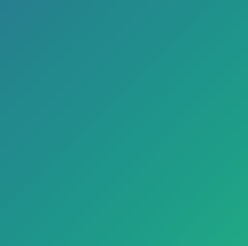}\\
    \raisebox{3.3\height}{(C)} & \raisebox{3.7\height}{\texttt{Attributes with noise}}
                & \includegraphics[width=0.14\textwidth]{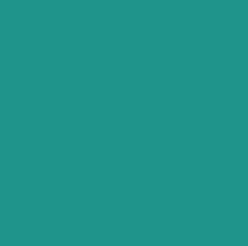}
                & \includegraphics[width=0.14\textwidth]{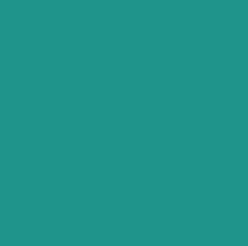}
                & \includegraphics[width=0.14\textwidth]{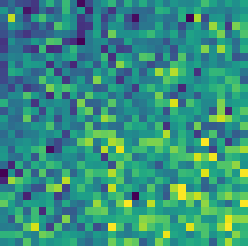}
                & \includegraphics[width=0.14\textwidth]{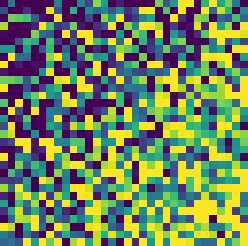}
                & \includegraphics[width=0.14\textwidth]{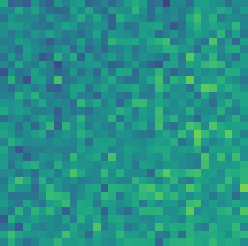}\\
    
    \raisebox{3.3\height}{(D)} & \raisebox{3.7\height}{\texttt{Estimated Homophily}}
                & \includegraphics[width=0.14\textwidth]{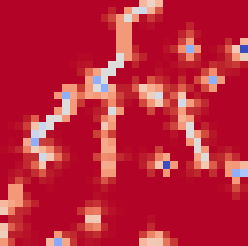}
                & \includegraphics[width=0.14\textwidth]{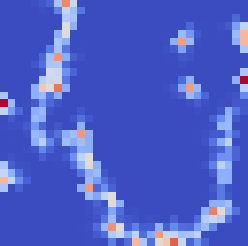}
                & \includegraphics[width=0.14\textwidth]{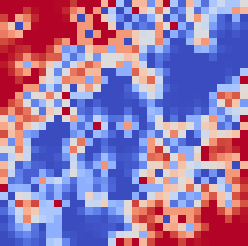}
                & \includegraphics[width=0.14\textwidth]{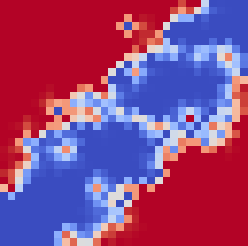}
                & \includegraphics[width=0.14\textwidth]{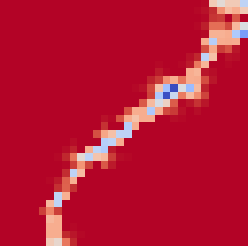}\\
    
    \midrule
    \multicolumn{7}{c}{
      \begin{minipage}{\textwidth}
        \centering
        \includegraphics[width=0.12\textwidth]{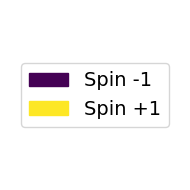}
        \hspace{1.cm}
        \includegraphics[width=0.37\textwidth]{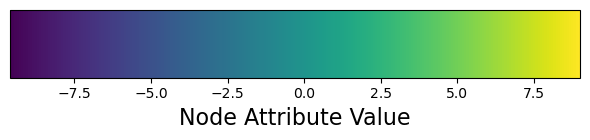}
        \hspace{0.8cm}
        \includegraphics[width=0.37\textwidth]{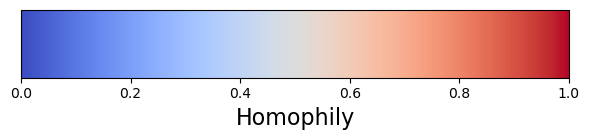}
      \end{minipage}
    } \\
    \bottomrule
  \end{tabular}}
\end{table}

\begin{table}
    \centering
    \caption{\new{Ising classification accuracies (\%), mean $\pm$ std over 5 runs, for setups $i_1$--$i_5$ (see Table~\ref{tab:ising_data_setup}). Bold fonts indicate the best-performing method for each column within its respective noise condition, while gray cells highlight the better result between MAP and non-MAP variants.}
    \label{tab:ising_data_results}}
    \resizebox{\textwidth}{!}{%
    \begin{tabular}{llccccc}
      \toprule
      & Setup & $i_1$ & $i_2$ & $i_3$ & $i_4$ & $i_5$ \\
      \midrule
    \multirow{11}{*}{\rotatebox{90}{\textbf{No noise}}}
    & \eqmakebox[model][l]{\textbf{GGCN}} & 52.20{\scriptsize$\pm$0.00} & 50.63{\scriptsize$\pm$0.95} &  61.17{\scriptsize$\pm$0.66} & 76.68{\scriptsize$\pm$1.17} & 94.63{\scriptsize$\pm$0.00}  \\
    & \eqmakebox[model][l]{\textbf{GGCN}} + \textbf{MAP} & {\cellcolor{gray}91.01{\scriptsize$\pm$0.49}}& {\cellcolor{gray}95.56{\scriptsize$\pm$0.47}}& {\cellcolor{gray}79.52{\scriptsize$\pm$0.39}}& {\cellcolor{gray}92.27{\scriptsize$\pm$0.53}}& {\cellcolor{gray}96.23{\scriptsize$\pm$0.71}}   \\

    & \eqmakebox[model][l]{\textbf{GCN}}  & 52.20{\scriptsize$\pm$0.00} & 48.78{\scriptsize$\pm$0.00} &  61.56{\scriptsize$\pm$0.37} & 71.90{\scriptsize$\pm$0.39} & 94.63{\scriptsize$\pm$0.00}  \\
    & \eqmakebox[model][l]{\textbf{GCN}} + \textbf{MAP} & {\cellcolor{gray}90.82{\scriptsize$\pm$0.61}}& {\cellcolor{gray}96.04{\scriptsize$\pm$0.47}}& {\cellcolor{gray}80.39{\scriptsize$\pm$0.39}}& \textbf{{\cellcolor{gray}93.33{\scriptsize$\pm$0.64}}}& {\cellcolor{gray}95.65{\scriptsize$\pm$0.00}} \\

    & \eqmakebox[model][l]{\textbf{AC}}  & 52.20{\scriptsize$\pm$0.00} & 49.07{\scriptsize$\pm$0.96} &  61.46{\scriptsize$\pm$0.53} & 72.00{\scriptsize$\pm$0.59} & 94.63{\scriptsize$\pm$0.00}  \\
    & \eqmakebox[model][l]{\textbf{AC}} + \textbf{MAP} & {\cellcolor{gray}91.30{\scriptsize$\pm$0.75}} & {\cellcolor{gray}96.23{\scriptsize$\pm$1.12}}& {\cellcolor{gray}80.39{\scriptsize$\pm$0.72}}& {\cellcolor{gray}92.56{\scriptsize$\pm$0.39}}& \textbf{{\cellcolor{gray}97.10{\scriptsize$\pm$0.00}}} \\

    & \eqmakebox[model][l]{\textbf{MLP}}  & 52.20{\scriptsize$\pm$0.00} & 51.22{\scriptsize$\pm$0.00} &  62.05{\scriptsize$\pm$0.57} & 77.07{\scriptsize$\pm$0.62} & 94.63{\scriptsize$\pm$0.00}  \\
    & \eqmakebox[model][l]{\textbf{MLP}} + \textbf{MAP} & {\cellcolor{gray}90.43{\scriptsize$\pm$0.99}}& \textbf{{\cellcolor{gray}96.33{\scriptsize$\pm$0.72}}}& \textbf{{\cellcolor{gray}80.48{\scriptsize$\pm$0.95}}}& {\cellcolor{gray}92.75{\scriptsize$\pm$0.43}}& {\cellcolor{gray}96.91{\scriptsize$\pm$0.39}} \\
    \cmidrule{2-7}
    & \eqmakebox[model][l]{\textbf{GMNN}}&  84.00{\scriptsize$\pm$2.37} & 57.22{\scriptsize$\pm$2.67} & 60.59{\scriptsize$\pm$0.89} & 76.93{\scriptsize$\pm$1.46} & 92.78{\scriptsize$\pm$2.61} \\
    & \eqmakebox[model][l]{\textbf{NeuPSL}$_{\textbf{LP}}$}&  \textbf{93.17{\scriptsize$\pm$0.00}} & 95.12{\scriptsize$\pm$0.00} & 70.83{\scriptsize$\pm$1.44} & 61.27{\scriptsize$\pm$1.79} & 96.68{\scriptsize$\pm$0.20} \\
    & \eqmakebox[model][l]{\textbf{NeuPSL}$_{\textbf{LP+FP}}$}&  - & - & 71.02{\scriptsize$\pm$1.69} & 61.12{\scriptsize$\pm$0.98} & 96.68{\scriptsize$\pm$0.20} \\
    \midrule
    \multirow{11}{*}{\rotatebox{90}{\textbf{With noise}}}
    & \eqmakebox[model][l]{\textbf{GGCN}}  & - & - &  60.39{\scriptsize$\pm$4.10} & 62.15{\scriptsize$\pm$7.15} & 76.10{\scriptsize$\pm$0.44}  \\
    & \eqmakebox[model][l]{\textbf{GGCN}} + \textbf{MAP} & - & - & {\cellcolor{gray}78.26{\scriptsize$\pm$1.10}}& \textbf{{\cellcolor{gray}90.24{\scriptsize$\pm$2.76}}}& {\cellcolor{gray}94.01{\scriptsize$\pm$0.24}} \\
    & \eqmakebox[model][l]{\textbf{GCN}}  & - & - &  62.34{\scriptsize$\pm$0.57} & 71.41{\scriptsize$\pm$0.79} & 76.59{\scriptsize$\pm$0.69}  \\
    & \eqmakebox[model][l]{\textbf{GCN}} + \textbf{MAP}  & - & - & {\cellcolor{gray}79.13{\scriptsize$\pm$0.19}}& {\cellcolor{gray}89.08{\scriptsize$\pm$0.58}}& {\cellcolor{gray}93.53{\scriptsize$\pm$0.58}} \\
    & \eqmakebox[model][l]{\textbf{AC}}  & - & - &  61.46{\scriptsize$\pm$1.88} & 71.61{\scriptsize$\pm$1.17} & 76.68{\scriptsize$\pm$0.78}  \\
    & \eqmakebox[model][l]{\textbf{AC}} + \textbf{MAP}  & - & - & {\cellcolor{gray}77.78{\scriptsize$\pm$0.92}}& {\cellcolor{gray}89.95{\scriptsize$\pm$0.83}}& {\cellcolor{gray}94.01{\scriptsize$\pm$0.24}} \\
    & \eqmakebox[model][l]{\textbf{MLP}}  & - & - &  54.34{\scriptsize$\pm$0.90} & 52.68{\scriptsize$\pm$0.82} & 59.32{\scriptsize$\pm$0.90}  \\
    & \eqmakebox[model][l]{\textbf{MLP}} + \textbf{MAP}  & - & - & \textbf{{\cellcolor{gray}80.00{\scriptsize$\pm$0.66}}}& {\cellcolor{gray}87.15{\scriptsize$\pm$0.39}}& \textbf{{\cellcolor{gray}97.87{\scriptsize$\pm$0.39}}} \\
    \cmidrule{2-7}
    & \eqmakebox[model][l]{\textbf{GMNN}}  &  - & - & 60.93{\scriptsize$\pm$1.56} & 75.66{\scriptsize$\pm$2.04} & 90.44{\scriptsize$\pm$1.24} \\
    & \eqmakebox[model][l]{\textbf{NeuPSL}$_{\textbf{LP}}$}  &  - & - & 70.05{\scriptsize$\pm$1.38} & 57.56{\scriptsize$\pm$0.00} & 97.07{\scriptsize$\pm$0.00} \\
    & \eqmakebox[model][l]{\textbf{NeuPSL}$_{\textbf{LP+FP}}$}  &  - & - & 71.41{\scriptsize$\pm$0.32} & 59.32{\scriptsize$\pm$0.73} & 97.07{\scriptsize$\pm$0.00} \\
    \bottomrule
  \end{tabular}}
\end{table}

\subsection{Ising: Results}
\label{sec:ising-res}
Row (D) in Table \ref{tab:ising_data_setup} illustrates the estimated local homophily values $\hat{\emph{LH}}$ as a heatmap (red: $\hat{\emph{LH}}(v)\approx 1$, blue: $\hat{\emph{LH}}(v)\approx 0$).
The estimates here are quite close to the ground truth values. We note that our method here benefits from the fact that the local homophily values in the Ising graphs vary rather smoothly over the graph. A  limitation of our approach for estimating $\hat{\emph{LH}}$ is that it will not work well on graphs where neighboring nodes have very different local homophily values.

Table \ref{tab:ising_data_results} shows classification accuracies obtained by embedding different GNN architectures in our framework. 
Bold font identifies the best performances for each data setting (columns). We use the state-of-the-art GGCN \cite{yan2022two}, 
which is specifically designed to handle both homophilic and heterophilic data, as well as standard graph convolutional networks (GCN), 
\new{the aggregate-combine GNN (AC) of \cite{barcelo2020logical}}, and multi-layer perceptrons (MLP). 
The latter serve as a baseline to evaluate what can be learned from node attributes alone, without exploiting the graph structure. 
All models were configured with a common architecture of 2 message passing layers and hidden dimension of 16 (32 for MLP). 
GGCN was run with the hyperparameter setting in the source code provided by the authors. In all cases, we train the GNN ${\mathcal N}$ on the node classification task, 
and then compare the accuracies obtained by ${\mathcal N}$ alone, and the accuracy obtained by performing MAP inference with ${\mathcal N}$ as described in Section \ref{sec:mapfornode}.
MAP inference is executed with random restarts, where the inference procedure is run multiple times from different random initializations. 
Specifically, we perform 3 restarts of Algorithm \ref{alg:map}, and use the solution that yields the highest likelihood value. 
The results in the table are averages and standard deviations over 5  executions of the experiments. 
The division into train and test nodes is fixed throughout all the experiments, and the train/validation/test splits are (48/32/20).
For comparison, we also include the GMNN model \cite{qu2019gmnn} \new{and NeuPSL~\cite{pryor-nepsl}} on the same set of training and test nodes as the other models. 
\new{Node classification for citation networks with a homophilic label distribution
is already included as an 
application of NeuPSL~\cite{pryor-nepsl}. 
There, a neural network node classifier based on only the target node's own features is combined with a label-propagation rule designed for homophilic graphs. Apart from using the original bag-of-word node feature vectors, \cite{pryor-nepsl} also apply their approach to modified data where in a pre-processing step the original node features are replaced by an average of the features of a node and its neighbors (thus implementing a basic message passing update). The two versions are denoted  $LP$ (label propagation) and
$LP+FP$ (label and feature propagation)
respectively. 
For graphs with homophily $h > 0.5$, we use the  logical homophily rule as given in \cite{pryor-nepsl}. For graphs with $h<0.5$  we employ a similar rule expressing a heterophily assumption instead. 
Both GMNN and NeuPSL already perform a form of collective node classification aimed at exploiting mutual label dependencies. They are therefore included in the experiments as stand-alone alternatives to our MAP approach, not as possible base models that can expect significant improvement by embedding them into our framework. 
}

Not surprisingly, when the nodes have no informative attributes ($F=0$), then ${\mathcal N}$ alone cannot do much better than predicting the majority class among the training nodes. 
Adding MAP inference, here enables a very significant jump in performance. In Table \ref{tab:ising_data_results},
the gray cells indicate the better performance between the MAP inference variant and the non-MAP instance of each model. 
In the cases $F>0$ and noise-free data settings, as expected, the MLP is performing almost as well as the proper GNN models, and AC performs comparably to GCN and GGCN.
Adding the MAP inference still gives a marked improvement in all cases, with AC+MAP achieving the best noise-free result on two of the five settings ($i_1$ and $i_5$).
The graph with the most fragmented distribution of local homophily values ($i_3$) poses the biggest challenge.
With noisy node attributes, the performance drops for all models.
The comparison between MLP and the proper GNN models shows that here the GNNs have learned to exploit information from neighboring nodes.
Notably, here the basic GCN \new{and AC models slightly outperform GGCN as base predictors for the setup $i_5$,
and MAP inference improves over the corresponding base model in \emph{every} setting, with no exceptions.
AC tracks GCN and GGCN closely in this noisy regime, and the best overall result at each setting is achieved by different models:
GGCN+MAP attains the highest accuracy at $i_4$ ($90.24\%$), while MLP+MAP attains the highest accuracy at both $i_3$ ($80.00\%$) and $i_5$ ($97.87\%$),
the latter despite the MLP's much weaker base accuracy at $i_5$ ($59.32\%$, compared to $76.10$--$76.68\%$ for the GNN models).
This suggests that the homophily-based MAP inference step can substantially compensate for the MLP's inability to directly exploit noisy neighborhood information,
by instead correcting its independent predictions using the inferred local homophily structure alone.}

\new{The stand-alone competitors achieve their strongest results for strongly homophilic graphs $i_1,i_5$, and in the case of NeuPSL with the heterophily rule, also for the extremely heterophilic $i_2$. Neither approach copes well with graphs with local variations in the homophily structure ($i_3,i_4$)}

{\em Overall, GNN-RBN models consistently boost the performance of their embedded GNN models, and outperform alternative approaches to directly exploit homophilic/heterophilic structures in most cases, especially when the homophily of the label distribution shows local variations.}

\section{Real-World Data}
\new{
Having demonstrated the benefits of MAP inference on the fully controlled Ising benchmarks, we now turn to
real-world graphs. We evaluate on the nine benchmark datasets used in~\cite{yan2022two}, which together span a wide range of homophily levels, from strongly heterophilic
(Texas, Wisconsin) to strongly homophilic (Cora, Citeseer, Pubmed). This lets us examine whether the gains
from MAP inference observed in Section~\ref{sec:isingdata} generalize beyond synthetic data. All datasets are loaded via PyTorch
Geometric~\cite{fey2019pyg}, using the dataset classes that reproduce the original splits introduced by
Geom-GCN~\cite{pei2020geomgcn}. For the \texttt{WikipediaNetwork} datasets (Chameleon and Squirrel), however,
we move from the original Geom-GCN splits to the filtered versions described in~\cite{platonov2024critical}, which eliminates some data duplications that compromise a clean train/test split.  We use the same experimental protocols and competitors as for the Ising data experiments.
}

\new{
Table~\ref{tab:realworld} reports the resulting accuracies, with datasets ordered by increasing homophily $h$.
The results show that for both GGCN and GCN,
adding the MAP inference step improves accuracy on eight of the nine datasets; PubMed, both with GCN and GGCN, is the only
case where MAP inference degrades performance, suggesting that when the base model is already close to optimal, then imposing the homophily-matching constraint can override
correct predictions rather than correct them. The improvements are generally largest on
the more heterophilic datasets: Squirrel, Chameleon, Actor, Cornell, Texas, and Wisconsin all show clear gains
for both base models, with the effect being especially pronounced for the weaker GCN baseline.
On the strongly homophilic citation networks (CiteSeer, PubMed, and Cora), the effect of MAP inference is
comparatively modest, due to the already high accuracy obtained by the base models. 
}

\new{Among the competitors, NeuPSL in the LP version performs best
overall, but still ranks below GGCN+MAP on all datasets except Pubmed. These results generally confirm the observations made based on the Ising experiments, except that here we cannot easily trace performance differences to (controlled) data properties like local homophily variations.}

\begin{table}[t]
\centering
\caption{\new{Node classification accuracy (\%) on real-world benchmarks, ordered by
increasing homophily ($h$), reported as mean\,$\pm$\,std. Gray cells highlight the
better result within each base/+MAP pair; \textbf{bold} indicates the best result
overall across all three models.}}
\label{tab:realworld}
\resizebox{\columnwidth}{!}{%
\begin{tabular}{llccccccccc}
\toprule
& & \rotatebox{60}{\textbf{Texas}} & \rotatebox{60}{\textbf{Wisconsin}}
& \rotatebox{60}{\textbf{Actor}} & \rotatebox{60}{\textbf{Squirrel}}
& \rotatebox{60}{\textbf{Chameleon}} & \rotatebox{60}{\textbf{Cornell}}
& \rotatebox{60}{\textbf{Citeseer}} & \rotatebox{60}{\textbf{Pubmed}}
& \rotatebox{60}{\textbf{Cora}} \\
\midrule
\multicolumn{2}{l}{\textit{Homophily} ($h$)}
& .13 & .20 & .22 & .21 & .24 & .31 & .74 & .80 & .81 \\
\multicolumn{2}{l}{\# Nodes}
& 183 & 251 & 7,600 & 2,223 & 890 & 183 & 3,327 & 19,717 & 2,708 \\
\multicolumn{2}{l}{\# Edges}
& 325 & 515 & 26,752 & 93,996 & 17,708 & 298 & 4,676 & 44,327 & 5,278 \\
\multicolumn{2}{l}{\# Classes}
& 5 & 5 & 5 & 5 & 5 & 5 & 6 & 3 & 7 \\
\midrule
\multirow{2}{*}{\textbf{GGCN}}
& Base
& $82.97${\scriptsize$\pm$4.99} & $86.67${\scriptsize$\pm$3.70} & $34.97${\scriptsize$\pm$0.94}
& $31.71${\scriptsize$\pm$2.19} & $33.96${\scriptsize$\pm$2.74} & $75.95${\scriptsize$\pm$4.09}
& $76.87${\scriptsize$\pm$1.46} & \cellcolor{lightgray}$\mathbf{89.13}${\scriptsize$\pm$0.33} & $86.02${\scriptsize$\pm$1.53} \\
& +MAP
& \cellcolor{lightgray}$\mathbf{85.95}${\scriptsize$\pm$5.64} & \cellcolor{lightgray}$\mathbf{87.65}${\scriptsize$\pm$4.39}
& \cellcolor{lightgray}$\mathbf{43.32}${\scriptsize$\pm$3.78} & \cellcolor{lightgray}$36.99${\scriptsize$\pm$2.57}
& \cellcolor{lightgray}$42.17${\scriptsize$\pm$7.44} & \cellcolor{lightgray}$\mathbf{80.27}${\scriptsize$\pm$3.21}
& \cellcolor{lightgray}$\mathbf{81.80}${\scriptsize$\pm$2.62} & $80.83${\scriptsize$\pm$3.71}
& \cellcolor{lightgray}$\mathbf{89.40}${\scriptsize$\pm$1.69} \\
\midrule
\multirow{2}{*}{\textbf{GCN}}
& Base
& $57.03${\scriptsize$\pm$4.90} & $48.24${\scriptsize$\pm$6.86} & $28.14${\scriptsize$\pm$0.65}
& $35.35${\scriptsize$\pm$1.39} & $39.55${\scriptsize$\pm$2.94} & $40.27${\scriptsize$\pm$4.90}
& $74.85${\scriptsize$\pm$1.79} & \cellcolor{lightgray}$86.75${\scriptsize$\pm$0.50} & $86.52${\scriptsize$\pm$1.05} \\
& +MAP
& \cellcolor{lightgray}$62.43${\scriptsize$\pm$8.50} & \cellcolor{lightgray}$68.24${\scriptsize$\pm$11.05}
& \cellcolor{lightgray}$42.23${\scriptsize$\pm$4.60} & \cellcolor{lightgray}$\mathbf{59.28}${\scriptsize$\pm$8.06}
& \cellcolor{lightgray}$\mathbf{45.52}${\scriptsize$\pm$3.67} & \cellcolor{lightgray}$55.41${\scriptsize$\pm$9.15}
& \cellcolor{lightgray}$79.60${\scriptsize$\pm$3.22} & $82.19${\scriptsize$\pm$3.07}
& \cellcolor{lightgray}$89.22${\scriptsize$\pm$1.81} \\
\midrule
\textbf{GMNN}
& 
& $58.11${\scriptsize$\pm$4.23} & $58.43${\scriptsize$\pm$9.44} & $28.74${\scriptsize$\pm$1.51}
& $34.74${\scriptsize$\pm$1.30} & $36.80${\scriptsize$\pm$2.91} & $44.32${\scriptsize$\pm$5.82}
& $74.65${\scriptsize$\pm$1.38} & $84.91${\scriptsize$\pm$0.51} & $81.37${\scriptsize$\pm$1.60} \\
\midrule
\multirow{2}{*}{\textbf{NeuPSL}}
& LP
& $77.57${\scriptsize$\pm$6.29} & $73.53${\scriptsize$\pm$5.49} & $28.59${\scriptsize$\pm$1.03}
& $23.79${\scriptsize$\pm$2.68} & $20.50${\scriptsize$\pm$1.91} & $65.41${\scriptsize$\pm$7.72}
& $74.08${\scriptsize$\pm$2.04} & $84.54${\scriptsize$\pm$0.33} & $85.67${\scriptsize$\pm$1.48} \\
& LP+FP
& $64.32${\scriptsize$\pm$5.77} & $58.04${\scriptsize$\pm$7.50} & $24.59${\scriptsize$\pm$1.56}
& $31.19${\scriptsize$\pm$2.07} & $26.36${\scriptsize$\pm$2.67} & $47.57${\scriptsize$\pm$5.30}
& $73.54${\scriptsize$\pm$1.78} & $83.01${\scriptsize$\pm$0.29} & $84.97${\scriptsize$\pm$1.51} \\
\bottomrule
\end{tabular}}
\end{table}

\section{Application: MAP for Multi-Objective Decision Making}
\label{sec:optimization}

In this section, we explore a fundamentally different type of task than is usually considered in the context of graph learning. We consider multi-objective optimization problems under uncertainty in network domains and show how they can be cast as MAP inference problems.

\subsection{Maximizing Expectations by MAP}
\label{sec:maxmap}

In this sub-section, we cast the problem of maximizing expected values of one or several objective functions as a
MAP optimization problem. This problem transformation is general, and the material in this sub-section is not
specific to relational domains. We assume a problem domain described by variables
partitioned into \emph{control} variables whose values can be freely set, and \emph{random} variables \new{that depend on
  the control variables, and will usually also be subject to additional random noise. Thus, when $X$ is a random and $C$ a control variable, then conditional probabilities $P(X|C)$ are defined, but probabilities
$P(C)$ or $P(C|X)$ are not, because no prior distribution is given for $C$.} 
We are interested in maximizing the expected values of one or several random variables by optimal settings of the control
variables.
In the application we will consider in the following section, the random variables whose expectation is to be
maximized represent conflicting
environmental and economic interests (water quality vs. profit from agriculture), and the control variables represent
decisions on land use. The following shows how such problems can be reduced to MAP inference by essentially
the same trick already employed in our collective node classification application: define the probability of  auxiliary
Boolean random variables as functions of continuous quantities of interest, and then condition the auxiliary variables to \emph{true}
(cf. Section \ref{sec:mapfornode}). 

\begin{definition}
  \label{def:etaX}
  Let $X\in[\emph{min}_X,\emph{max}_X]$ be a bounded, real-valued random variable. We define the min-max normalization of $X$:
  \begin{equation}
    \label{eq:LX}
    L(X):=\frac{X-\emph{min}_X}{\emph{max}_X-\emph{min}_X},
  \end{equation}
  and then a Boolean random variable $\eta_X$ conditional on $X$ by 
  \begin{equation}
    \label{eq:etaX}
     P(\eta_X=\emph{true}|X) = L(X).
  \end{equation}
\end{definition}

The following proposition makes the connection between maximizing $E[X]$ and MAP inference for $\eta_X$.

\begin{proposition}
  Let $X$ and $\eta_X$ as in Definition \ref{def:etaX}.
  Assume that $X$ has a conditional distribution
  $P(X|{\boldsymbol C})$ depending on
  control variables ${\boldsymbol C}$. Then for every configuration ${\boldsymbol c}$ of the control variables
  \begin{equation}
    \label{eq:conddist}
     E[L(X)|{\boldsymbol C}={\boldsymbol c}]= P(\eta_X=\emph{true}|{\boldsymbol C}={\boldsymbol c}),
   \end{equation}
   and 
  \begin{equation}
    \label{eq:argmaxc}
    \emph{argmax}_{\boldsymbol c} E[X|{\boldsymbol C}={\boldsymbol c}]=
    \emph{argmax}_{\boldsymbol c} P(\eta_X=\emph{true}|{\boldsymbol C}={\boldsymbol c}).
  \end{equation}
\end{proposition}
\begin{proof}
  We obtain (\ref{eq:conddist}) by marginalizing over $X$ and using that $\eta_X$ is independent of
  ${\boldsymbol C}$ given $X$:
  \begin{multline*}
    P(\eta_X=\emph{true}|{\boldsymbol C}={\boldsymbol c})=
    \int P(\eta_X=\emph{true}|X=x)P(X=x|{\boldsymbol C}={\boldsymbol c})dx =\\
     \int L(x) P(X=x|{\boldsymbol C}={\boldsymbol c})dx =
     E[L(X)|{\boldsymbol C}={\boldsymbol c}].
  \end{multline*}
  (\ref{eq:argmaxc}) then follows from the linearity and monotonicity of $L$:
  \begin{displaymath}
    \begin{array}{c}
      E[L(X)|{\boldsymbol C}={\boldsymbol c}]=L( E[X|{\boldsymbol C}={\boldsymbol c}])\\
      E[X|{\boldsymbol C}={\boldsymbol c}]>E[X|{\boldsymbol C}={\boldsymbol c'}]
      \ \Leftrightarrow\ 
      L(E[X|{\boldsymbol C}={\boldsymbol c}])>L(E[X|{\boldsymbol C}={\boldsymbol c'}]).
    \end{array}
  \end{displaymath}
\end{proof}
\new{
  MAP inference is usually defined over a joint distribution of random variables, without a distinction of control and
  random variables. In order to cast the right-hand side of (\ref{eq:argmaxc}) as an instance of MAP inference in
  this strict sense, we can equip ${\boldsymbol C}$ with any prior distribution $P({\boldsymbol C})$. Then we can
  apply Bayes rule:
  \begin{displaymath}
    P(\eta_X=\emph{true}|{\boldsymbol C}={\boldsymbol c})= P({\boldsymbol C}={\boldsymbol c}|\eta_X=\emph{true})
    \frac{P(\eta_X=\emph{true})}{{P(\boldsymbol C}={\boldsymbol c})}.
  \end{displaymath}
  If, in particular, we let $P({\boldsymbol C})$ be the uniform distribution, then
  \begin{displaymath}
    \emph{argmax}_{\boldsymbol c} P(\eta_X=\emph{true}|{\boldsymbol C}={\boldsymbol c})=
    \emph{argmax}_{\boldsymbol c} P({\boldsymbol C}={\boldsymbol c}|\eta_X=\emph{true}),
  \end{displaymath}
  where the right-hand side now matches the generic MAP inference expression (\ref{eq:mapeq}) with
  ${\boldsymbol M}\sim {\boldsymbol C}$ and ${\boldsymbol D}\sim \eta_X$.
  }

Now, suppose the objective is to simultaneously maximize the expected value of two random variables
$X_1,X_2$ depending on control variables ${\boldsymbol C}$.
A first way to handle this scenario is to just apply the above approach to the random variable $X:=\lambda X_1+(1-\lambda)X_2$ with
$\lambda\in[0,1]$ representing a tradeoff between the objectives $X_1,X_2$.
Then MAP inference for ${\boldsymbol C}$ given $\eta_X=\emph{true}$ will maximize $E[\lambda X_1+(1-\lambda)X_2]$.
In this case, if the numerical range of $X_2$ is much
larger than the numerical range of $X_1$ (e.g. $\emph{min}_{X_1}=\emph{min}_{X_2}=0; \emph{max}_{X_2}\gg \emph{max}_{X_1}$), 
then extreme $\lambda$ values may be needed for  $X_1$ to play a major role in the optimization. Alternatively, one
can perform min-max normalizations  $L_1(X_1), L_2(X_2)$ for $X_1,X_2$ individually first, and define
. Then
$X$ already is min-max normalized, and MAP inference for  ${\boldsymbol C}$ given $\eta_X=\emph{true}$ will maximize $\lambda E[L_1(X_1)]+(1-\lambda)E[L_2(X_2)]$.
While essentially equivalent (modulo a re-scaling of the tradeoff parameter $\lambda$), the second approach can be more
intuitive as here $\lambda$ represents directly a tradeoff between the normalized objectives
$L_1(X_1),L_2(X_2)$. It is the approach we will adopt in our following application, where $X_1$ will be just a Boolean variable, and
$X_2$ a numerical variable representing a financial objective.

\subsection{Environmental Planning: Data and Tasks}
\label{sec:hawqs-combinatorial}

We demonstrate our approach using an environmental planning scenario described by real-world watershed data and
simulations performed using  advanced simulation tools.

\paragraph{Simulation data.}
We generate data using the 
\emph{Soil and Water Assessment Tool (SWAT)} \cite{SWAT} and the \emph{Hydrologic and Water Quality System (HAWQS)} \cite{hawqs}.
SWAT is a river basin-scale hydrologic model developed to simulate the effects of land use, land management practices, and climate on water, sediment, and agricultural chemical yields in large, complex watersheds.
HAWQS is a web-based platform that provides user-friendly, cloud-based access to SWAT simulations, integrated with nationally consistent datasets for topography, land use, soil, and weather across the United States. In this study, we utilize SWAT simulation data from HAWQS for the Honey Creek watershed in Iowa, U.S.A. (HUC12: 102802010407), which is predominantly agricultural. Figure \ref{fig:watershed} is a screenshot from HAWQS, showing the watershed and the main water flow directions.

\begin{figure}[t]
    \centering
    \includegraphics[width=0.7\linewidth]{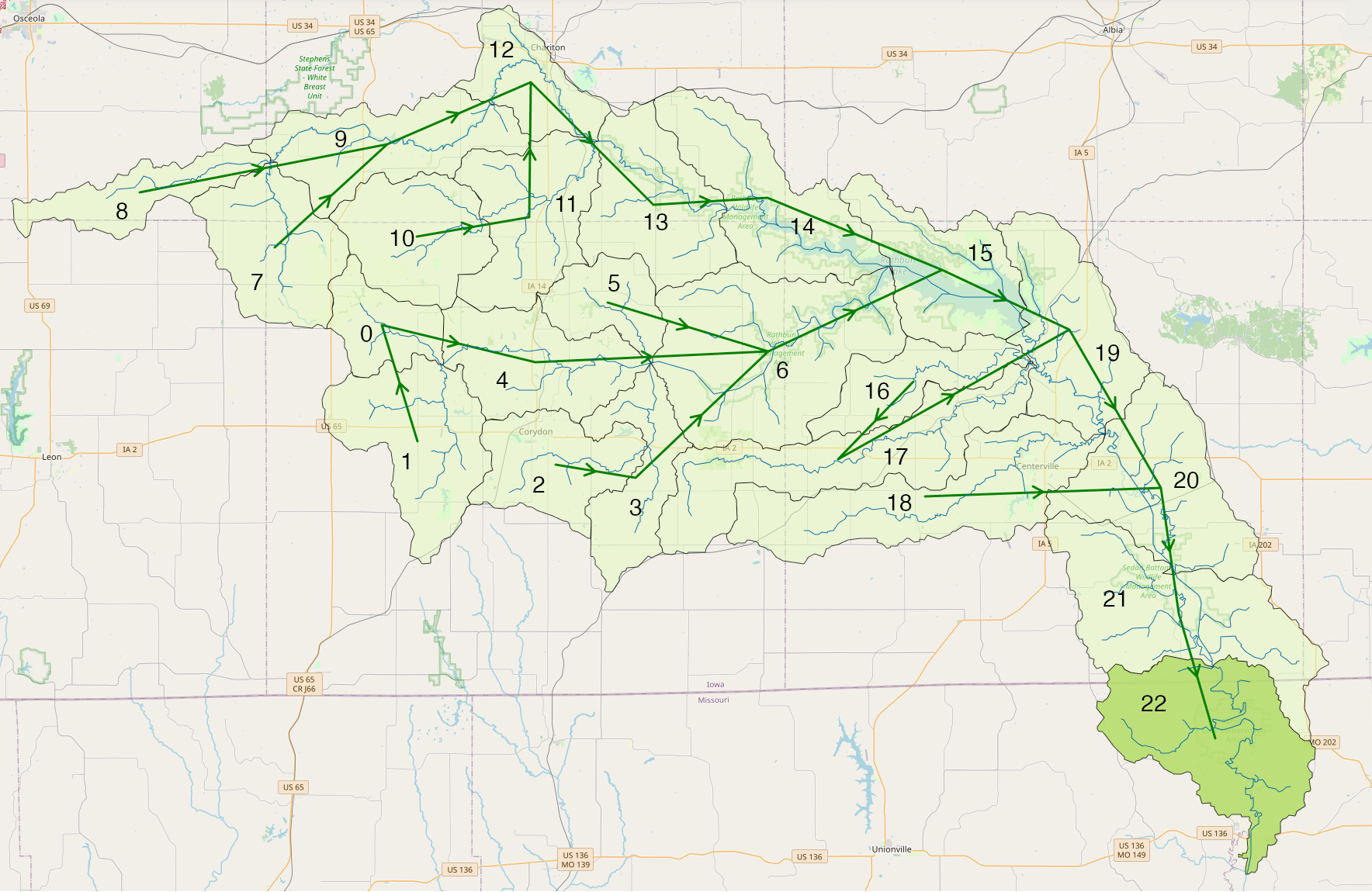}
    \caption{Screenshot from HAWQS of the watershed. The lines with the direction represent the water flow dynamics used by the simulator. The numbers indicate the different subbasins.}
    \label{fig:watershed}
\end{figure}

The basin is subdivided into subbasins. Each subbasin consists of one main water channel and multiple units of land characterized by common geo-physical properties and land use. Subbasins are connected by a downstream relation, and land units are connected to the unique water channel in their subbasin. For our example, land units are divided into types
\emph{agriculture} and \emph{other}, where the latter represent units such as forests or urban areas that are not
subject to annual land use decisions. The basin can be represented  in a graph structure as depicted on the left of
Figure \ref{fig:basingraphs}. 

The SWAT simulations integrate weather data, seasonal variations, and water availability over time.
For our experiments, we simulate an 11-year period (2010–2020) under different crop scenarios. A scenario consists of a specification of crop compositions for each of the subbasins.  We focus on four primary crops (corn, soybean, corn/soy rotation (abbreviated "cosy"), and pasture), and run simulations for 14 manually defined crop scenarios.
We consider nitrogen concentration as the water quality indicator of interest.
The simulation generates time series at a daily resolution of nitrogen concentration values at all subbasins. 

\begin{figure}[t]
    \centering
    \includegraphics[width=\linewidth]{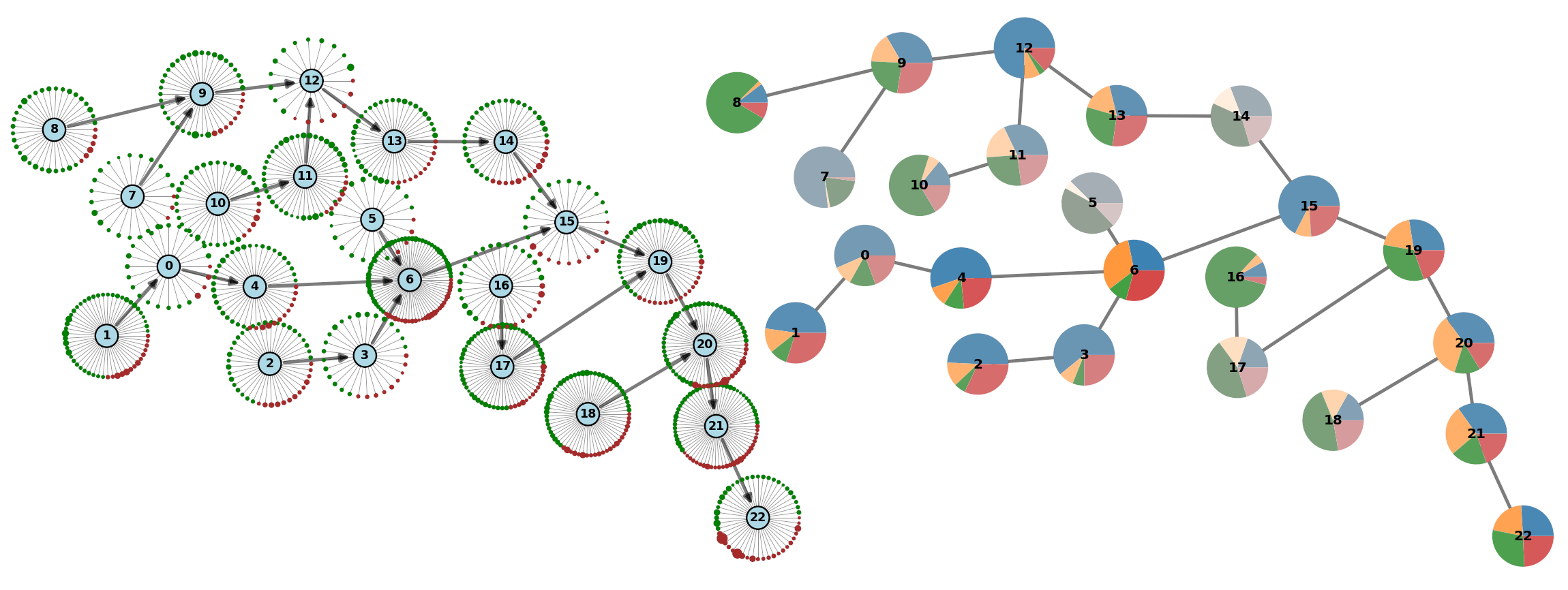}
    \caption{Left: Watershed graph. Blue: water, green: agricultural land, brown: other land. Right: inferred optimal crop composition for $\lambda=0.7$ (blue: corn, green: pasture, red: soy, orange: cosy).}
    \label{fig:basingraphs}
\end{figure}

  We encode the data as heterogeneous graphs with three types of nodes (23 \emph{water},
  676 \emph{land\_agr}, 290 \emph{land\_other} nodes),
  and two types of edges (\emph{land2water}, \emph{water2water}).
   All land nodes are equipped with the attribute  
   $\emph{Area}\in\Rset$ (indicated by node size in Figure \ref{fig:basingraphs}). The agricultural land nodes have the attribute $\emph{Crop}\in \{\emph{pasture},\emph{soy},\emph{cosy},\emph{corn}\}$, which represents our control variable of interest.
 The water nodes are equipped with an attribute
 $\emph{Pollution}\in\{\emph{high},\emph{medium},\emph{low}\}$.  This attribute contains the annual average nitrogen concentration, discretized into three levels defined by equal frequency binning. Every 11-year simulation under a given crop scenario then gives us 11 graphs that only differ for the \emph{Pollution} values at the water nodes.
 Our data, thus, consists of a total of 154 watershed graphs (14 crop scenarios x 11 years). Figure \ref{fig:all-data} in Appendix \ref{sec:envplanning} visualizes the data.

\paragraph{GNN design and training.}

We divide the data into 84 graphs for training, 23 for validation, and 47 for testing. Using the modules for heterogeneous convolution layers of Pytorch Geometric, we build and train a model with two message passing layers and hidden dimension of 20 on the task of predicting the \emph{Pollution} values at the water nodes. The resulting model achieves a 61.15\% accuracy on the test data. We note that maximizing this accuracy is not the purpose of our work. Here, we only need to ascertain that the trained model has sufficient predictive capabilities to be used in our subsequent tasks.
We also note that our learning scenario here is different from the standard inductive or transductive settings: as in transductive settings, test nodes are already seen during training, and, indeed, all nodes are both train and test nodes.
What changes between the train and the test phase are the \emph{Crop} attribute values of the neighbors of the train/test nodes. 

\paragraph{Integrated Model}

The Graph Neural Network defines the conditional probabilities $ P(\boldsymbol{p}|\boldsymbol{c})$ for pollution values $\boldsymbol{p}\in \{\emph{high},\emph{medium},\emph{low}\}^{23} $ given crop assignments
$\boldsymbol{c}\in $ $\{\emph{pasture},$ $\emph{soy},\emph{cosy},\emph{corn}\}^{676} $.
We integrate this GNN prediction model with manually (``expert'') defined RBN components related to the optimization objectives.
For each water node $v$, we define two random variables of interest:
\begin{equation}
  \begin{array}{l}
  X_1(v) = \indicator{\emph{Pollution}(v)=\emph{low}}\hspace{3mm} \\
  {\displaystyle X_2(v) = \sum_{w:land2water(w,v) } \hspace{-3mm}\emph{Area}(w)\sum_{c\in \emph{crops}}  \hspace{-3mm}\beta_c \indicator{\emph{Crop}(w)=c}}
  \end{array}
  \label{eq:map_random_vars}
\end{equation}
where $\indicator$ stands for the indicator function.
$X_1$ represents the objective of low pollution. $\expected{X_1(v)}$ is equal to the probability of low pollution at $v$. $X_2$ is the objective of maximizing the profit from the crops grown in the subbasin of $v$.
The parameters $\beta_c$ represent user defined values for the expected profit per area from growing crop $c$. We assume values
$\beta_c = 1,4,2,5$ for $c=\emph{pasture},\emph{soy},\emph{cosy},\emph{corn}$, respectively. $X_1(v)$ is lower/upper bounded by 0 and 1 for all $v$, and $X_2(v)$ by $\emph{ta}(v)$ and $5\emph{ta}(v)$, where $\emph{ta}(v)$ denotes
the total area of agricultural land in the subbasin of $v$. Thus, the assumptions of Definition \ref{def:etaX} are satisfied, and for $X_1,X_2$ and a
tradeoff parameter $\lambda>0$ we can define the Boolean variable  $\eta_X(v)$ for each water node $v$ as described in Section \ref{sec:maxmap}.
The full RBN encoding of the integrated model is given in Appendix \ref{sec:app_hawqs}.

\paragraph{Results.}

We perform MAP inference for MAP query atoms $\boldsymbol{C}$ consisting of the 676 $\emph{Crop}(v)$ variables,
conditioned on the 23 $\eta_X(v)$ variables set to \emph{true}.
We let $\lambda$ used in the definition of $\eta_X(v)$ vary between 0 and 1. Given a solution $\boldsymbol{C}=\boldsymbol{c}$ obtained for a 
setting of $\lambda$, we compute 
the expected number of water nodes with low pollution, and
the expected total profit in the river basin. Figure \ref{fig:pareto} on the left shows the different tradeoffs for these objectives depending on the $\lambda$ value. For each $\lambda$ value we performed 5 random restarts of the MAP optimization, which then gives 5 distinct points in the scatterplot, here indicating fairly stable results of the MAP inference at each $\lambda$. Figure \ref{fig:pareto} on the right shows the inferred optimal crop compositions at all $\lambda$, illustrating how the crop that has been learned to be the least polluting one (\emph{pasture}, green) gets replaced by the most profitable one (\emph{corn}, blue) as the objective moves towards profitability.

\begin{figure}
    \centering
    \includegraphics[width=0.49\linewidth]{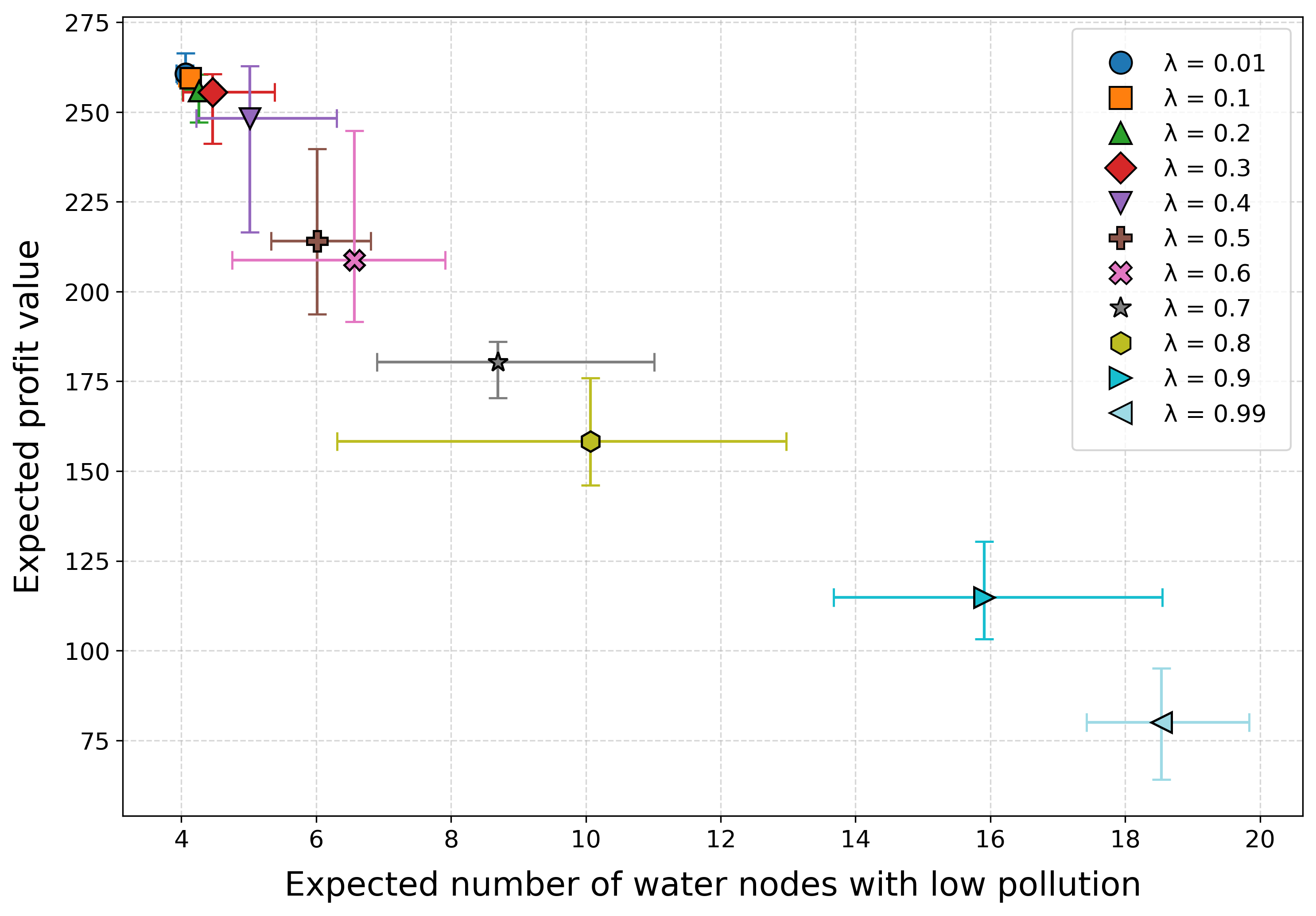}
    \includegraphics[width=0.49\linewidth]{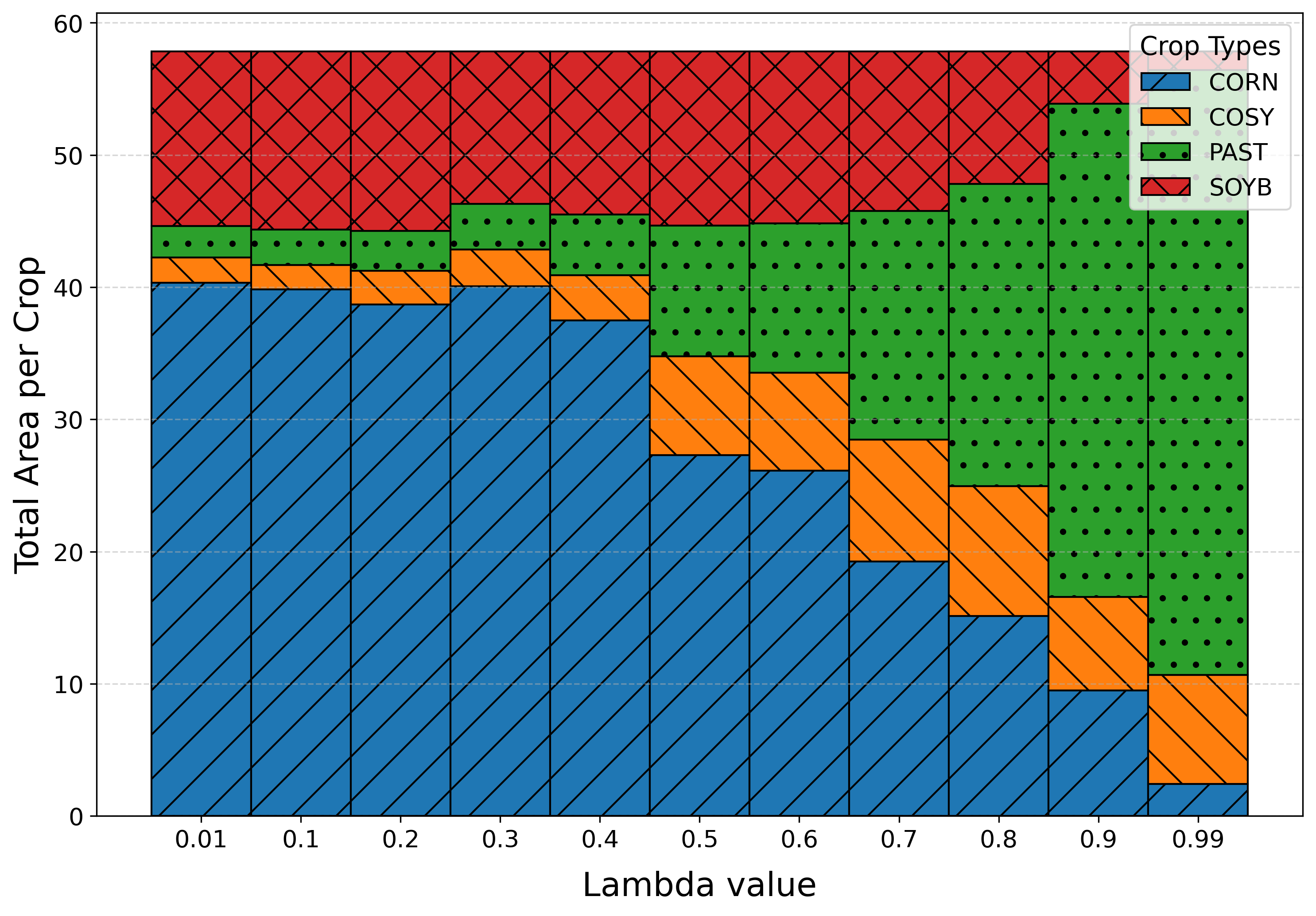}
    \caption{Left: achievable tradeoffs between clean water and profit, the bars indicate the min/max value range that multiple restarts have produced. Right: mean across runs of the optimal crop compositions for different tradeoff coefficients}
    \label{fig:pareto}
\end{figure}

A more fine-grained result is shown on the right of Figure \ref{fig:basingraphs}. It shows the optimized crop composition
at the subbasin level for $\lambda=0.7$. The coloring indicates the optimal composition, averaged over the 5 MAP restarts. The saturation of the coloring represents the variance of the compositions in the restarts. Apart from few subbasins, the results were very stable, and clearly identify those subbasins that are less susceptible to pollution, and hence allow for a high percentage of the profitable \emph{corn} crop, and those where the low pollution objective leads to a high proportion of \emph{pasture}. 

\new{We note at this point that the problem we tackle here is very different from more standard machine learning tasks: there is no ground truth labeling that we are trying to predict, and therefore there also is no simple success metric in terms of accuracy or squared error. Nevertheless, our task here is designed to resemble very relevant practical problems: to provide decision makers with model based analyses of possible (near-)optimal solutions under different tradeoff scenarios. The quality of the advice our approach can provide depends on two main factors: the predictive accuracy of the GNN component, and the optimality of the solution found in the heuristic MAP search. 
}
\new{
The predictive accuracy of the GNN is not our concern at this point. We may well assume that the $\sim 60\%$ accuracy we have obtained is the best one can obtain in light of inherently noisy data, and a target that is affected by many unobserved factors.  For assessing the quality of our MAP solutions no simple comparisons with baseline solutions are available: we are trying to solve a combinatorial optimization problem where the value function is partly defined by a non-linear neural network model. This is outside the scope of classical integer linear programming formulations, for example. One can, of course, consider many alternative heuristic optimization techniques, such as beam search or genetic algorithms. Our claim here is not that our current MAP search routine already is the best solution, but that our application study demonstrates  {\em the utility of a general MAP inference engine in our expressive graph NeSy integration for planning and decision support problems
in complex, network-structured domains.}}   

\subsection{Continuous Relaxation}
\label{sec:gradopt}

\new{
  So far we have modeled our planning problem as a combinatorial optimization problem. This matches reality, where land use
  decisions will have to be made for given units of land. However, the problem allows for a continuous relaxation in which we
  summarize the land usage by percentages of 
  crop allocations, as already used for illustration in Figure \ref{fig:basingraphs} on the right.
  In this section we show how our framework also supports this
  alternative problem formulation, that this formulation allows for faster optimization, and that the results we obtain
  for the relaxed version confirm the validity of our results for the combinatorial version.
  }

\new{
  For this experiment, the heterogeneous watershed graph is collapsed into a homogeneous graph of 23 nodes,
  each representing one subbasin.
Each node $v$ is characterized an attribute $\emph{Area}(v)$ 
that contains the total agricultural area of $v$, and attributes
$\emph{PercCrop}_c(v)$ containing the perecentage of the area allocated to each agricultural crop type
$c\in\{\emph{corn}, \emph{cosy}, \emph{pasture}, \emph{soy}\}$.
Each node also carries the \emph{Pollution}$(v)$ attribute that originally was associated with the unique water node of the subbasin.
}

\new{ We train, in PyTorch Geometric, a two-layer Graph Convolutional Network (GCN) with 32 hidden units on the
  \emph{Pollution} prediction task given the node areas and crop percentages as input features. 
  The resulting model reaches 60.92\% test accuracy, closely 
matching the 61.15\% obtained with the heterogeneous model of Section~\ref{sec:hawqs-combinatorial}.}

\new{
  In this new data model, the profit objective $X_2$  now is expressed as
\begin{equation}
  \label{eq:newX2}
  X_2(v) = \emph{Area}(v)\sum_{c\in \emph{crops}}  \beta_c \emph{PercCrop}_c(v),
\end{equation}
whereas the low pollution objective $X_1$ remains unchanged from (\ref{eq:map_random_vars}). As before, we define
the joint weighted objective $X:=\lambda L_1(X_1)+(1-\lambda)L_2(X_2)$ and its associated Boolean variable $\eta_X$.
Since RBNs support the numeric relations $\emph{PercCrop}_c$ only as fixed inputs without a prior probability distribution,
we cannot solve the problem of which  $\emph{PercCrop}_c$ settings maximize the
probability of $\eta_X=\emph{true}$  as a MAP inference problem. However, as shown in \cite{jiang2015numericinputrelationsrelational},
maximizing the likelihood of observed data by optimizing value settings of numerical input relations reduces to standard
RBN parameter optimization, as already supported by the likelihood graph data structure (the likelihood graph will now look a little
different from the one depicted in Figure \ref{fig:lgraph}: it has no MAP atom nodes, but instead contains at the input level
atom nodes for the numeric input relations on whose value the single observed data atom $\eta_X=\emph{true}$ depends).
In this case, now, in order to compute the gradient of the likelihood function with respect to the numeric input relations, we
actually have to use the {\tt evaluateGradients} function of the GNN interface in order to access GNN gradients inside the
RBN level gradient-based optimization procedure.\footnote{\new{The likelihood graph computes gradients that then can be fed into
  any optimization algorithm. The \emph{Primula} software implements several algorithms, including ADAM which is used in our experiments.}}

As in the combinatorial MAP version, the likelihood function requires marginalization over the unobserved \emph{Pollution} atoms.
Thus, here too, we approximate the full sum by a samples of  \emph{Pollution} values that are periodically re-sampled during
the gradient descent optimization process. 
}

\new{\paragraph{Results.}
  Figure~\ref{fig:mapvsadam} shows the pollution-profit tradeoffs obtained in the continuous relaxation version for different
  $\lambda$ values, and compares these with what was obtained in the combinatorial MAP setting (Figure \ref{fig:pareto}). We observe
  that the continuous relaxation provides slightly better tradeoff options, but that the results obtained from both versions are
  quite consistent. The slight improvement can be due to two reasons: a) the relaxation of the problem formulation allows for
  solutions with higher values for the objectives; b) gradient-based optimization is more effective than heuristic combinatorial
  optimization. An analysis of the difference in runtime between the two versions is included in the following section. 
}

\begin{figure}[h]
    \centering
    \includegraphics[width=0.6\linewidth]{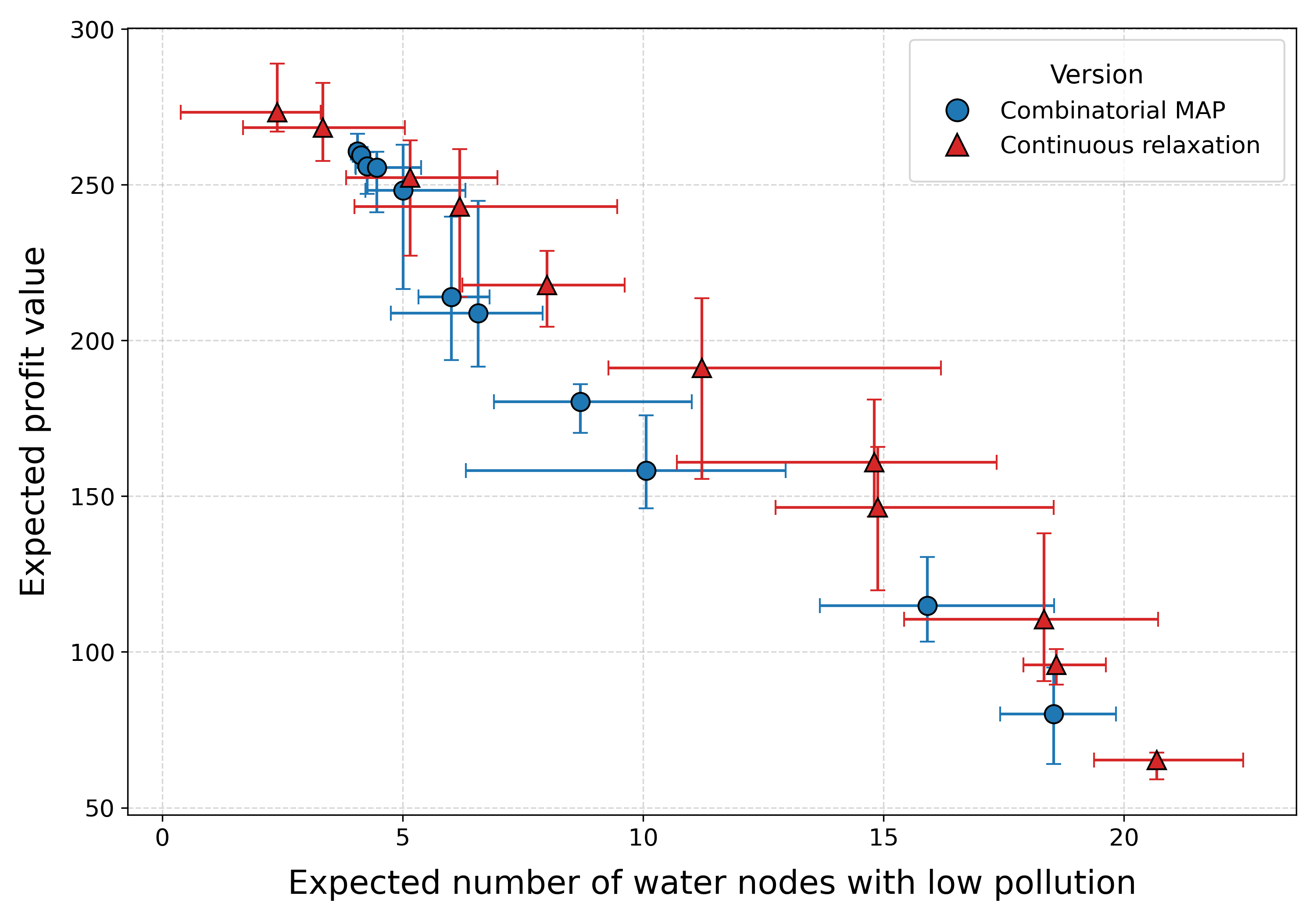}
    \caption{\new{Pollution/Profit tradeoff using the two different approaches: MAP inference and with numeric optimization. Both methods have to perform Gibbs sampling on the unobserved water pollution nodes.}}
    \label{fig:mapvsadam}
\end{figure}

\section{Runtime Behavior}
\label{sec:runtime}

\new{In this section, we investigate some of the key factors that determine the time complexity and scalability of our framework. 
First, we consider the tradeoffs involved between the two integration methods introduced in Sections \ref{sec:GNN2RBN} and \ref{sec:interface}.
Although the compiled and interface versions are semantically equivalent, they differ in
computational aspects. 
To investigate some of these aspects,
we performed a first experiment with the Ising dataset ($32\times32$ grid). 
We trained the same AC GNN architecture, with a hidden dimension of 32
under both settings: first, as an external PyTorch component interfaced to the RBN
(Section~\ref{sec:interface}), trained with ADAM in PyTorch; and second, as a compiled
RBN representation (Section~\ref{sec:GNN2RBN}), trained with the ADAM implementation in 
\textit{Primula}.
The results obtained from the trainings are equivalent, but as the Ising Model rows of Table~\ref{tab:merged_results} show at column \textbf{Time/epoch}, 
training in PyTorch is much faster than training the compiled RBN version (to make the comparison more meaningful
the PyTorch model is trained only on CPU). Part of the explanation for the huge performance gap we observe can be the
difference between all operations performed at the scalar level in the compiled version, vs. matrix/vector operations
in PyTorch.}
\new{The columns \textbf{LG constr.} and  \textbf{Time/restart} contain the times for the construction of the likelihood
  graph, and the time for one restart of MAP inference, respectively, in the collective node classification task
  based on the trained models.
The largest difference between the two versions is evident in the likelihood graph
construction time, which is to be expected as the likelihood graph in the compiled version must re-construct the
full computation graph of the GNN model. 
The MAP search itself is very fast for this task, with both versions completing in a fraction
of a second. Still, the interface version is noticeably slower than the compiled version (0.16
seconds vs.\ 0.02 seconds). This difference can be explained by the overhead of the Java-PyTorch interface,
and some losses in local update possibilities, as described in Section \ref{sec:likgraph}.}

\new{Next, we consider the  runtime performance of
inference in the applications of Sections \ref{sec:hawqs-combinatorial} and \ref{sec:gradopt}.
Both  were handled with the interface method, and in both applications, inference posed greater computational challenges than model training.
Table~\ref{tab:merged_results} reports in the rows ``Multi-objective MAP'' and ``Multi-objective Numeric'' execution times, along with the
the numbers of the different types of atoms that define the problem scenario. We observe that the continuous relaxation version of
the problem is significantly faster to solve. This can be partly attributed to the more compact solution space (92 numerical variables, vs. 676 categorical ones), but is probably mostly due to faster convergence of gradient-based optimization than combinatorial search.
}

\new{We next try to analyze the root causes of the very long computation time of $\sim$80 min.  for inference in the multi-objective
  MAP problem. In particular, we aim to assess the contribution to this time that stems from the repeated re-sampling of the
  23 unobserved atoms. For this we design a "Simplified MAP" version of the
  problem in which the observations $\eta_X(v)=\emph{true}$, where $\eta_X(v)$ then depends on the unobserved
  \emph{Pollution}$(v)$ atoms, are replaced with the direct conditions \emph{Pollution}$(v)$=\emph{low}. The solution space defined
  by the 676 MAP atoms is the same in both versions, as is the functional role that the integrated GNN model plays
  in the optimization task. 
We see that in the simplified MAP problem the time per restart is reduced to $\sim$21 sec., giving strong indication that sampling is the main computational bottleneck in the original problem.  }


\begin{table}[h]
    \centering
    \small
    \renewcommand{\arraystretch}{1.3}
    \resizebox{\linewidth}{!}{%
    \begin{tabular}{@{}llcccccc@{}}
    \toprule
    \textbf{Experiment} & \textbf{Integration} & \textbf{MAP atoms} &
    \makecell{\textbf{Numerical}\\\textbf{input atoms}} & \textbf{Unobs. atoms} & \textbf{Time/epoch} &
    \textbf{LG constr.} & \textbf{Time/restart} \\
    \midrule
    \multirow{2}{*}{Ising Model}
        & Interface & 1,024 & 0 & 0 & $\sim$0.0023 s & $\sim$3.4\,s   & $\sim$0.16\,s \\
        & Compiled  & 1,024 & 0 & 0 & $\sim$84.287 s & $\sim$78.4\,s  & $\sim$0.02\,s \\
    \midrule
    Multi-objective MAP
        & Interface & 676 & 0 & 23 & - & $\sim$0.23\,s & $\sim$80\,min \\
    \midrule
    Simplified MAP
        & Interface & 676 & 0 & 0  & - & $\sim$0.20\,s & $\sim$21.31\,s \\
    \midrule
    Multi-objective Numeric
        & Interface & 0 & 92 & 23 & - & $\sim$0.07\,s & $\sim$96.57\,s \\
    \bottomrule
    \end{tabular}%
    }
    \caption{\new{Execution time and construction cost across GNN--RBN integration strategies and experiments. Time/restart is averaged over 5 restarts. 
    The GNNs used in the multi-objective experiments are not supported by our current GNN-to-RBN compiler.} 
    }
    \label{tab:merged_results}
\end{table}
\new{
To complement the synthetic benchmarks above, we also evaluated our framework on a real-world dataset. 
Using the CiteSeer citation network in the node classification setting of Section~\ref{sec:nodeclass}, 
we measured the average wall-clock time required to complete MAP inference for a single train/test split. 
Our framework required approximately 70 minutes per split, compared to approximately 45 seconds per split 
for NeuPSL on the same data. 
Our approach performs MAP inference over a Likelihood Graph, 
whereas NeuPSL relies on a continuous relaxation solved via convex optimization, 
which scales considerably better  with the size of the underlying graphical model. 
This difference highlights an important tradeoff: 
our framework offers greater flexibility and a more expressive relational representation (with accuracy improvements ranging from $4\%$ to almost $30\%$ on real-world datasets, see Table~\ref{tab:realworld}), 
at the cost of substantially higher inference time on larger, real-world graphs.}

\new{Finally, we study how inference time and the Likelihood Graph scale with the size of the input graph, using the Ising model on square grids of increasing dimension $n$, from $16\times16$ (256 nodes) up to $128\times128$ (16384 nodes). 
It is worth noting here that, for this grid topology, the number of edges grows linearly with the number of nodes 
(each node being connected to at most its four grid neighbors), so that increasing $n$ does not introduce any change in graph density, only in overall size.
Table~\ref{tab:ising_runtime} reports, for each grid size, the likelihood graph (L.G.) internal nodes, the construction time, and the average MAP search time per restart; All quantities grow with $n$, but the growth remains close to linear in the number of nodes, 
demonstrating much better actual scaling behavior than the worst case described in Section~\ref{sec:likgraph} was described.}

\begin{table}[h]
    \centering
    \begin{tabular}{lcccc}
        \toprule
        \textbf{Grid} ($n\times n$) & \textbf{\# nodes} & \textbf{\# L.G. nodes} & \textbf{L.G. construction time} & \textbf{MAP time/restart} \\
        \midrule
        16  & 256    & 1,729    & 0.401 s  & 0.081 s \\
        32  & 1,024   & 7,368    & 1.243 s  & 0.242 s \\
        64  & 4,096   & 29,162   & 4.583 s  & 0.975 s \\
        128 & 16,384  & 117,328  & 20.670 s & 5.795 s \\
        \bottomrule
    \end{tabular}
    \caption{\new{Likelihood graph (L.G) construction time and average MAP search time per restart, as a function of the Ising grid dimension $n$ (grid size $n \times n$), including the number of nodes in the gradient graph.}}
    \label{tab:ising_runtime}
\end{table}


\section{Conclusion}

We have introduced a neuro-symbolic system that combines the high predictive power of graph neural networks with the logical
expressivity and general reasoning capabilities of relational Bayesian networks. The obtained integration is seamless and
faithful in that
the original semantics of a GNN model exactly correspond to the semantics of RBN building blocks, and in
that one or several GNN components can be combined with symbolic representations in an arbitrary order, and with arbitrary dependencies.
Thus, there is no requirement for a model structure consisting of a low-level neural, and a high-level symbolic layer.
The resulting generative probabilistic models provide general probabilistic inference methods for conditional probability
queries. In this paper we have focused on the use of maximum a-posteriori probability queries, and have demonstrated how this
type of query can be used to solve in a uniform algorithmic
framework two very different types of problems: collective node classification under
homo- and heterophilic conditions, and multi-objective planning and decision making in relational domains.
The results demonstrate the usefulness and versatility of our approach. Future work will be directed at refining the current
MAP inference algorithm further towards increased robustness and scalability, and by applying our approach to real-world
diagnostic and decision support problems.

\bibliographystyle{plain}
\bibliography{bib}

@inproceedings{ijcai2020p679,
  title     = {Graph Neural Networks Meet Neural-Symbolic Computing: A Survey and Perspective},
  author    = {Lamb, Luís C. and Garcez, Artur d’Avila and Gori, Marco and Prates, Marcelo O.R. and Avelar, Pedro H.C. and Vardi, Moshe Y.},
  booktitle = {IJCAI},
  year      = {2020},
  note      = {Survey track},
  publisher = {International Joint Conferences on Artificial Intelligence Organization}
}

@inproceedings{pryor-etal-2023-using,
    title = "Using Domain Knowledge to Guide Dialog Structure Induction via Neural Probabilistic Soft Logic",
    author = "Pryor, Connor  and
      Yuan, Quan  and
      Liu, Jeremiah  and
      Kazemi, Mehran  and
      Ramachandran, Deepak  and
      Bedrax-Weiss, Tania  and
      Getoor, Lise",
    booktitle = "Proceedings of the 61st Annual Meeting of the Association for Computational Linguistics (Volume 1: Long Papers)",
    year = "2023",
    publisher = "Association for Computational Linguistics"
}

@inproceedings{can2020,
author = {Liello, Luca Di and Ardino, Pierfrancesco and Gobbi, Jacopo and Morettin, Paolo and Teso, Stefano and Passerini, Andrea},
title = {Efficient generation of structured objects with constrained adversarial networks},
year = {2020},
booktitle = {NeurIPS},
publisher = {Curran Associates, Inc.},
volume = {33},
}

@inproceedings{misino2022vael,
  title={{VAEL: Bridging Variational Autoencoders and Probabilistic Logic Programming}},
  author={Misino, Eleonora and Marra, Giuseppe and Sansone, Emanuele},
  booktitle={Advances in Neural Information Processing Systems},
  publisher={Curran Associates, Inc.},
  volume={35},
  year={2022}
}

@InProceedings{pmlr-v202-bendinelli23a,
  title = 	 {Controllable Neural Symbolic Regression},
  author =       {Bendinelli, Tommaso and Biggio, Luca and Kamienny, Pierre-Alexandre},
  booktitle = 	 {Proceedings of ICML},
  year = 	 {2023},
  publisher =    {PMLR}
 }

@inproceedings{Liang22,
  author={Yichao Liang and Josh Tenenbaum and Tuan Anh Le and N. Siddharth},
  title={Drawing out of Distribution with Neuro-Symbolic Generative Models},
  year={2022},
  booktitle={NeurIPS},
  publisher={Curran Associates, Inc.}
}

@inproceedings{vanKrieken2023anesi,
  author    = {Emile van Krieken and Thiviyan Thanapalasingam and Jakub Tomczak and Frank van Harmelen and Annette ten Teije},
  title     = {A-NeSI: A Scalable Approximate Method for Probabilistic Neurosymbolic Inference},
  booktitle = {NeurIPS},
  year      = {2023},
  publisher = {Curran Associates, Inc.}
}

@inproceedings{ahmed2022semantic,
  title={{Semantic Probabilistic Layers for Neuro-Symbolic Learning}},
  author={Ahmed, Kareem and Teso, Stefano and Chang, Kai-Wei and Van den Broeck, Guy and Vergari, Antonio},
  booktitle={NeurIPS},
  year={2022},
  publisher={Curran Associates, Inc.}
}

@article{manhaeve2021neural,
  title={Neural probabilistic logic programming in DeepProbLog},
  author={Manhaeve, Robin and Duman{\v{c}}i{\'c}, Sebastijan and Kimmig, Angelika and Demeester, Thomas and De Raedt, Luc},
  journal={Artificial Intelligence},
  volume={298},
  year={2021}
}

@inproceedings{marra2021neural,
  title={Neural markov logic networks},
  author={Marra, Giuseppe and Ku{\v{z}}elka, Ond{\v{r}}ej},
  booktitle={UAI},
  year={2021},
  publisher={PMLR}
}

@inproceedings{pryor-nepsl,
  title     = {NeuPSL: Neural Probabilistic Soft Logic},
  author    = {Pryor, Connor and Dickens, Charles and Augustine, Eriq and Albalak, Alon and Wang, William and Getoor, Lise},
  booktitle = {International Joint Conference on Artificial Intelligence (IJCAI)},
  year      = {2023},
  publisher = {International Joint Conferences on Artificial Intelligence Organization}
}

@article{badreddine2022ltn,

	year = 2022,

	publisher = {Elsevier {BV}},

	volume = {303},


	author = {Samy Badreddine and Artur d{\textquotesingle}Avila Garcez and Luciano Serafini and Michael Spranger},

	title = {Logic Tensor Networks},

	journal = {Artificial Intelligence}
}

@article{yu2023survey,
  title={A survey on neural-symbolic learning systems},
  author={Yu, Dongran and Yang, Bo and Liu, Dayou and Wang, Hui and Pan, Shirui},
  journal={Neural Networks},
  volume={166},
  pages={105--126},
  year={2023},
  publisher={Elsevier}
}

@article{Kautz_2022, title={The Third {AI} Summer: {AAAI} {R}obert
                  {S}. {E}ngelmore Memorial Lecture}, volume={43},
                    journal={AI  Magazine}, author={Kautz, Henry}, year={2022}}

@inproceedings{ijcai2020p688,
  title     = {From Statistical Relational to Neuro-Symbolic Artificial Intelligence},
  author    = {Raedt, Luc de and Dumančić, Sebastijan and Manhaeve, Robin and Marra, Giuseppe},
  booktitle = {IJCAI},
  year      = {2020},
  note      = {Survey track},
  publisher = {International Joint Conferences on Artificial Intelligence Organization}
}

@incollection{garcez2022neural,
  title={Neural-symbolic learning and reasoning: A survey and interpretation},
  author={Besold, Tarek R and others},
  booktitle={Neuro-Symbolic Artificial Intelligence: The State of the Art},
  year={2021},
  publisher={IOS Press}
}

@article{hitz2022,
    author = {Hitzler, Pascal and Eberhart, Aaron and Ebrahimi, Monireh and Sarker, Md Kamruzzaman and Zhou, Lu},
    title = {Neuro-symbolic approaches in artificial intelligence},
    journal = {National Science Review},
    volume = {9},
    number = {6},
    year = {2022}
}

@inproceedings{rbn1997,
author = {Jaeger, Manfred},
title = {Relational {B}ayesian networks},
year = {1997},
booktitle = {Proceedings of UAI}
}

@article{ZHOU202057,
title = {Graph neural networks: A review of methods and applications},
journal = {AI Open},
year = {2020},
author = {Jie Zhou and Ganqu Cui and Shengding Hu and Zhengyan Zhang and Cheng Yang and Zhiyuan Liu and Lifeng Wang and Changcheng Li and Maosong Sun},
}

@book{getoor2007introduction,
  author = {Getoor, L. and Taskar, B.},
  publisher = {The MIT Press},
  title = {Introduction to statistical relational learning},
  year = 2007
}

@inproceedings{pojer2024generalized,
  title={Generalized Reasoning with Graph Neural Networks by Relational Bayesian Network Encodings},
  author={Pojer, Raffaele and Passerini, Andrea and Jaeger, Manfred},
  booktitle={Learning on Graphs Conference},
  year={2024},
  publisher={PMLR}
}

@misc{hawqs,
    author = {{HAWQS 2.0}},
    publisher = {Texas Data Repository},
    title = {{HAWQS System 2.0 and Data to model the lower 48 conterminous U.S using the SWAT model}},
    year = {2023},
    version = {V2},
    doi = {10.18738/T8/GDOPBA},
}

@article{luan2024heterophilic,
  title={The heterophilic graph learning handbook: Benchmarks, models, theoretical analysis, applications and challenges},
  author={Luan, Sitao and Hua, Chenqing and Lu, Qincheng and Ma, Liheng and Wu, Lirong and Wang, Xinyu and Xu, Minkai and Chang, Xiao-Wen and Precup, Doina and Ying, Rex and others},
  eprint={2407.09618},
  archivePrefix={arXiv},
  primaryClass={cs.LG},
  year={2024}
}

@InProceedings{HamYinLes17,
  author       = {William L. Hamilton and Zhitao Ying and Jure
                  Leskovec},
  title        = {Inductive Representation Learning on Large Graphs},
  year         = 2017,
  booktitle    = {Proceedings of NeurIPS},
  publisher    = {Curran Associates, Inc.}
}

@misc{kipf2016variational,
  title={Variational graph auto-encoders},
  author={Kipf, Thomas N and Welling, Max},
  eprint={1611.07308},
  archivePrefix={arXiv},
  primaryClass={cs.LG},
  url={https://arxiv.org/abs/1611.07308},
  year={2016}
}

@inproceedings{simonovsky2018graphvae,
  title={Graphvae: Towards generation of small graphs using variational autoencoders},
  author={Simonovsky, Martin and Komodakis, Nikos},
  booktitle={ICANN},
  year={2018},
  publisher={Springer}
}

@inproceedings{liu2019graph,
  title={Graph normalizing flows},
  author={Liu, Jenny and Kumar, Aviral and Ba, Jimmy and Kiros, Jamie and Swersky, Kevin},
  booktitle={NeurIPS},
  year={2019},
  publisher={Curran Associates, Inc.}
}

@article{ManEtAl18,
  title={DeepProbLog: Neural Probabilistic Logic Programming},
  author={Manhaeve, Robin and Dumancic, Sebastijan and Kimmig, Angelika and Demeester, Thomas and De Raedt, Luc},
  journal={Advances in Neural Information Processing Systems},
  volume={31},
  pages={3749--3759},
  year={2018}
}

@article{vsourek2021beyond,
  title={Beyond graph neural networks with lifted relational neural networks},
  author={{\v{S}}ourek, Gustav and {\v{Z}}elezn{\`y}, Filip and Ku{\v{z}}elka, Ond{\v{r}}ej},
  journal={Machine Learning},
  volume={110},
  number={7},
  pages={1695--1738},
  year={2021},
  publisher={Springer}
}

@inproceedings{bodnar2022neural,
  title={Neural sheaf diffusion: A topological perspective on heterophily and oversmoothing in gnns},
  author={Bodnar, Cristian and Di Giovanni, Francesco and Chamberlain, Benjamin and Lio, Pietro and Bronstein, Michael},
  booktitle={Proceedings of NeurIPS},
  year={2022},
  publisher={Curran Associates, Inc.}
}

@inproceedings{qu2019gmnn,
  title={Gmnn: Graph {M}arkov neural networks},
  author={Qu, Meng and Bengio, Yoshua and Tang, Jian},
  booktitle={ICML},
  year={2019},
  publisher={PMLR}
}

@InProceedings{jaegerAIB22,
  author =	{Jaeger, Manfred},
  title =	{{Learning and Reasoning with Graph Data: Neural and Statistical-Relational Approaches}},
  booktitle =	{International Research School in Artificial Intelligence in Bergen},
  series =	{Open Access Series in Informatics (OASIcs)},
  year =	{2022},
  publisher =	{Schloss Dagstuhl -- Leibniz-Zentrum fuer Informatik}
}

@InProceedings{JaegerICML07,
  author = 	 {Manfred Jaeger},
  title = 	 {Parameter Learning for Relational {B}ayesian Networks},
  booktitle = 	 {Proceedings ICML},
  year =	 2007,
  publisher =	 {ACM}
}

@article{jaeger2001complex,
  title={Complex probabilistic modeling with recursive relational Bayesian networks},
  author={Jaeger, Manfred},
  journal={Annals of Mathematics and Artificial Intelligence},
  volume={32},
  number={1},
  pages={179--220},
  year={2001},
  publisher={Springer}
}

@inproceedings{loveland2024performance,
  title={On performance discrepancies across local homophily levels in graph neural networks},
  author={Loveland, Donald and Zhu, Jiong and Heimann, Mark and Fish, Benjamin and Schaub, Michael T and Koutra, Danai},
  booktitle={Learning on Graphs Conference},
  year={2024},
  publisher={PMLR}
}

@inproceedings{yan2022two,
  title={Two sides of the same coin: Heterophily and oversmoothing in graph convolutional neural networks},
  author={Yan, Yujun and Hashemi, Milad and Swersky, Kevin and Yang, Yaoqing and Koutra, Danai},
  booktitle={ICDM},
  year={2022},
  publisher={IEEE},
  address={Los Alamitos, CA, USA}
}

@misc{SWAT,
  author = {{USDA Agricultural Research Service and Texas A\&M AgriLife Research}},
  title = {Soil and Water Assessment Tool ({SWAT})},
  url = {https://swat.tamu.edu}
}

@inproceedings{barcelo2020logical,
  title={The logical expressiveness of graph neural networks},
  author={Barcel{\'o}, Pablo and Kostylev, Egor V and Monet, Mikael and P{\'e}rez, Jorge and Reutter, Juan and Silva, Juan-Pablo},
  booktitle={8th International Conference on Learning Representations (ICLR 2020)},
  year={2020}
}

@article{skryagin2023scalable,
  title={Scalable neural-probabilistic answer set programming},
  author={Skryagin, Arseny and Ochs, Daniel and Dhami, Devendra Singh and Kersting, Kristian},
  journal={Journal of Artificial Intelligence Research},
  volume={78},
  pages={579--617},
  year={2023}
}

@inproceedings{yang2021neurasp,
  title={NeurASP: embracing neural networks into answer set programming},
  author={Yang, Zhun and Ishay, Adam and Lee, Joohyung},
  booktitle={IJCAI},
  year={2020},
  publisher={International Joint Conferences on Artificial Intelligence Organization}
}

@inproceedings{zhangefficient,
  title={Efficient Probabilistic Logic Reasoning with Graph Neural Networks},
  author={Zhang, Yuyu and Chen, Xinshi and Yang, Yuan and Ramamurthy, Arun and Li, Bo and Qi, Yuan and Song, Le},
  booktitle={International Conference on Learning Representations},
  year={2020}
}

@inproceedings{primulatool,
  title={Primula-3 for Probabilistic Modeling and Reasoning on Graph Data},
  author={Pojer, Raffaele and Jaeger, Manfred},
  booktitle={ECML PKDD Conference, Demo track},
  year={2025},
  publisher={Springer}
}

@article{sourek2018lifted,
  title={Lifted relational neural networks: Efficient learning of latent relational structures},
  author={Sourek, Gustav and Aschenbrenner, Vojtech and Zelezny, Filip and Schockaert, Steven and Kuzelka, Ondrej},
  journal={Journal of Artificial Intelligence Research},
  volume={62},
  pages={69--100},
  year={2018}
}

@misc{pei2020geomgcn,
      title={Geom-GCN: Geometric Graph Convolutional Networks},
      author={Hongbin Pei and Bingzhe Wei and Kevin Chen-Chuan Chang and Yu Lei and Bo Yang},
      year={2020},
      eprint={2002.05287},
      archivePrefix={arXiv},
      primaryClass={cs.LG},
      url={https://arxiv.org/abs/2002.05287},
}

@misc{fey2019pyg,
      title={Fast Graph Representation Learning with PyTorch Geometric},
      author={Matthias Fey and Jan Eric Lenssen},
      year={2019},
      eprint={1903.02428},
      archivePrefix={arXiv},
      primaryClass={cs.LG},
      url={https://arxiv.org/abs/1903.02428},
}

@misc{platonov2024critical,
      title={A critical look at the evaluation of GNNs under heterophily: Are we really making progress?},
      author={Oleg Platonov and Denis Kuznedelev and Michael Diskin and Artem Babenko and Liudmila Prokhorenkova},
      year={2024},
      eprint={2302.11640},
      archivePrefix={arXiv},
      primaryClass={cs.LG},
      url={https://arxiv.org/abs/2302.11640},
}

@misc{jiang2015numericinputrelationsrelational,
      title={Numeric Input Relations for Relational Learning with Applications to Community Structure Analysis},
      author={Jiuchuan Jiang and Manfred Jaeger},
      year={2015},
      eprint={1506.05055},
      archivePrefix={arXiv},
      primaryClass={cs.LG},
      url={https://arxiv.org/abs/1506.05055},
}

\appendix

\section{Proof of Proposition \ref{prop:reprequiv}}
\label{sec:proof3.1}
\new{
We begin by reviewing the exact definition of the class of ACR networks from \cite{barcelo2020logical}. ACR networks
generalize the rudimentary message passing update described by Equation (\ref{eq:gnn-update}) as follows:
\begin{equation}
  \label{eq:acr-1}
  \mathbf{h}^{k+1}(v) = f\left(W_1^{k+1} \mathbf{h}^{k}(v) + W_2^{k+1} \mbox{AGG}^{k+1}(\leftms  \mathbf{h}^{k}(u)| u \in \mathcal{N}_v \rightms)+
   W_3^{k+1}\mbox{READ}^{k+1}(\leftms  \mathbf{h}^{k}(u)| u \in V \rightms) + \mathbf{b}^{k+1}\right)
\end{equation}
where
\begin{itemize}
\item $W_1^{k+1},W_2^{k+1},W_3^{k+1}$ are $m\times d$ dimensional weight matrices; $\mathbf{b}^{k+1}$ an $m$-dimensional bias vector
  ($m,d$ as in Section~\ref{sec:gnn}),
\item $\mbox{AGG}^{k+1}$ and $\mbox{READ}^{k+1}$ are aggregation functions. In the original definition of \cite{barcelo2020logical} no specific assumptions are made for AGG or READ. We shall require that both are standard aggregation functions such as SUM, MEAN or MAX that are defined component-wise, i.e.,
the aggregation of a multiset of vectors is defined by individual aggregation of each of the vectors' components.
\item $f$ is a non-linear activation function.  
\end{itemize}
}
\new{
Thus,  (\ref{eq:gnn-update}) is the special case where only $W_2$ is nonzero, AGG is SUM, and $f=\sigma$. 
}
\new{
\reprequiv*
}
\new{
The proof of
Proposition \ref{prop:reprequiv} now is just an elaboration of the RBN encoding (\ref{eq:probform-mp}) of the scalar version
(\ref{eq:scalar-mp}) of (\ref{eq:gnn-update}).
}

\begin{proof}
\new{
  Let $\mathcal{N}$ be a $K$-layer GNN as stated in theorem. Denote with $d_k$ the dimension of the
  embedding vectors $\mathbf{h}^k$. 
  We first show inductively for $k=0,\ldots,K$ that the components $ \mathbf{h}^{k}_i(v)\ (i=0,\ldots,d_k-1)$
  of the embedding
  vectors can be computed by probability formulas $F^k_i(v)$.
}

\new{
  $k=0$: here $d_0=l$ and $\mathbf{h}^{0}_i(v)=X_i(v)$. In the RBN setting we can view each $X_i$ as a unary numeric
  input relation belonging to ${\mathcal R}_{\emph{in}}$, and simply put $F^0_i(v)=X_i(v)$.
}

\new{
  $k\mapsto k+1$: we introduce as abbreviations for components of (\ref{eq:acr-1}):
  \begin{displaymath}
    \begin{array}{l}
      \mathbf{g}^{k}_1(v):=\mathbf{h}^{k}(v)\\
      \mathbf{g}^{k}_2(v):= \mbox{AGG}^{k+1}(\leftms  \mathbf{h}^{k}(u)| u \in \mathcal{N}_v \rightms)\\
      \mathbf{g}^{k}_3(v):= \mbox{READ}^{k+1}(\leftms  \mathbf{h}^{k}(u)| u \in V \rightms)
    \end{array}
  \end{displaymath}
  Since the activation function $f$ in  (\ref{eq:acr-1}) is applied component-wise, we have that
  \begin{displaymath}
    \mathbf{h}^{k+1}_i(v) = f\left( (W_1^{k+1}\mathbf{g}^{k}_{1}(v))_i+(W_2^{k+1}\mathbf{g}^{k}_{2}(v))_i+
    (W_3^{k+1}\mathbf{g}^{k}_{3}(v))_i+ \mathbf{b}^{k+1}_i \right)\ \ (i=0,\ldots,d_{k+1}-1)
  \end{displaymath}
  where for $l=1,2,3$ the $W^{k+1}_l$ are $d_{k+1}\times d_k$-dimensional matrices, and the $ \mathbf{g}^{k}_l$ are
  $d_k$-dimensional vectors. Assume, now, that for $l=1,2,3$ and $j=0,\ldots,d_k-1$ we have probability formulas
  $G_{l,j}(v)$ that compute $\mathbf{g}^{k}_{l,j}(v)$. Then
   \begin{alignat*}{3}
    H_{l,i}(v) :=\ & \texttt{COMBINE } && W^{k+1}_{l,i,0} \cdot G_{l,0}(v),\\
    & &&\ldots \\
    & && W^{k+1}_{l,i,d_k-1} \cdot G_{l,d_k-1}(v) \\
    & \texttt{WITH SUM} 
\end{alignat*}
computes $ (W_l^{k+1}\mathbf{g}^{k}_{l}(v))_i$. The construction of the required $G_{l,j}(v)$ differs slightly for 
different $l$. For $l=1$ we simply have $ G_{1,j}(v) = F^k_j(v)$ as per induction hypothesis. For
$l=2$ we set
\begin{equation}
   \label{eq:local1010}
\begin{split}
    G_{2,j}(v) :=& \texttt{COMBINE } F^k_j(u)\\
    & \texttt{WITH AGG} \\ 
    & \texttt{FORALL } u \\
    & \texttt{WHERE } \emph{edge}(v,u),
\end{split}
\end{equation}
whereas for $l=3$ we use
\begin{equation}
   \label{eq:local1020}
\begin{split}
    G_{3,j}(v) :=& \texttt{COMBINE } F^k_j(u)\\
    & \texttt{WITH READ} \\ 
    & \texttt{FORALL } u,
\end{split}
\end{equation}
where the only difference to (\ref{eq:local1010}) is that the aggregation is over all nodes, not only neighbors of $v$.
We finally need to define a probability formula $F^{k+1}_i(v)$ that computes
\begin{equation}
  \label{eq:local1030}
  f(H_{1,i}(v)+H_{2,i}(v)+H_{3,i}(v)+\mathbf{b}_i^{k+1}).
\end{equation}
In the language of probability formulas, both aggregation and activation are handled jointly by a combination function formula.
If $f=\sigma$ in (\ref{eq:local1030}), then (\ref{eq:local1030}) is computed by the logistic regression combination function:
\begin{equation}
   \label{eq:local1040}
\begin{split}
    F^{k+1}_i(v) :=& \texttt{COMBINE } H_{1,i}(v),H_{2,i}(v),H_{3,i}(v),\mathbf{b}_i^{k+1}\\
    & \texttt{WITH LOG-REG}. 
\end{split}
\end{equation}
To support other activation functions, new customized combination functions can be defined (it is worth noting that at the implementation
level all that is needed for a new activation function $f$ is a module that computes values and derivatives of $f$).
}

\new{
As the result of our inductive construction we obtain probability formulas
\begin{displaymath}
  F^K_0(v),\ldots, F^K_{d_K-1}(v)
\end{displaymath}
computing the final node representations. The application of the \emph{softmax} in the GNN is then equivalent to
the probability formula \ccsoftmax$ F^K_0(v),\ldots, F^K_{d_K-1}(v)$.}
\end{proof}

\section{Propagate Homophily Algorithm}
\label{sec:propalg}

\begin{algorithm}[h]
\caption{Propagate Homophily}
\label{alg:propagate_homophily}
\begin{algorithmic}[1]
\Procedure{PropagateHomophily}{graph($V,E$), iterations, tolerance}
    \State Convert \textbf{graph} to an undirected graph
    \State  $W_{train} \gets 2.0, \quad W_{test} \gets 1.0$
    \State Initialize: \texttt{stabilized[$v$]} $\gets$ \texttt{false}, \quad \texttt{train\_hom[$v$]} $\gets 0$ for all $v \in V$
    \State $sumTrainHom \gets 0$
    \State $countValidTrain \gets 0$
    \medskip

    \For{each training node $v \in V$ with non-empty $trainNbrs$}
    \State \texttt{train\_hom[$v$]} $\gets$ \Call{ComputeLocalHomophily}{$trainNbrs$}
    \State \texttt{hom[$v$]} $\gets$ \texttt{train\_hom[$v$]}
    \EndFor
    \State $initValue \gets$ \Call{Average}{\{\texttt{train\_hom[$v$]} $\mid v \in V_{\text{train}}$ and $trainNbrs \neq \emptyset$\}}
    \medskip
    \For{each test node $v$, or training node $v$ with empty $trainNbrs$}
        \State \texttt{hom[$v$]} $\gets initValue$
    \EndFor
            
    \medskip
    \State $converged \gets$ \texttt{false}
    \For{$iter \gets 1$ \textbf{to} iterations \textbf{while not} converged}
        \State $converged \gets$ \texttt{true}
        \For{each node $v \in V$}
            \State $nbrs \gets$ \Call{getAllNeighbors}{$v$, graph}
            \State $trainNbrs \gets$ \Call{getTrainNeighbors}{$v$, graph}
            \State $testNbrs \gets$ \Call{getTestNeighbors}{$v$, graph}
            \medskip
            \If{$v$ is a test node \textbf{or} ($v$ is training and $trainNbrs$ is empty)}
                \If{$nbrs$ is empty}
                    \State \texttt{hom[$v$]} $\gets initValue$
                \Else
                    \State \texttt{hom[$v$]} $\gets$ \Call{WeightedAverage}{$trainNbrs, testNbrs, W_{train}, W_{test}$}
                \EndIf
            \ElsIf{$v$ is a training node and $testNbrs$ is not empty}
                \State $avgTest \gets$ \Call{Average}{\texttt{[hom[$u$] for each $u \in testNbrs$]}}
                \State \texttt{hom[$v$]} $\gets$ \Call{WeightedAverage}{$\text{\texttt{train\_hom[$v$]}}, avgTest, W_{train}, W_{test}$}
            \EndIf
            \medskip
            \State \textbf{if} change in \texttt{hom[$v$]} $>$ tolerance \textbf{then} $converged \gets$ \texttt{false}
        \EndFor
    \EndFor
    \State \Return \texttt{hom}
\EndProcedure
\end{algorithmic}
\end{algorithm}

The Propagate Homophily function (Algorithm \ref{alg:propagate_homophily}) iteratively refines each node's homophily value in a graph by leveraging the connectivity among nodes. Initially, the graph is transformed into an undirected structure to ensure symmetric information flow. Each node is assigned a baseline homophily value derived from training data. For training nodes, a local homophily score is computed using only their training neighbors. During each iteration, the algorithm updates each node's score by averaging the homophily values of its neighbors. For test nodes, or training nodes whose neighbors are exclusively test nodes, the update is a simple average. For training nodes with both training and test neighbors, a weighted average is computed that balances the local training homophily with the average test neighbor score. The process repeats until the change in any node's score is below a predefined tolerance or a maximum number of iterations is reached, yielding a refined homophily vector that encapsulates both the initial training signals and the graph's structure.




\section{Robustness analysis}
\new{In this section, we briefly examine the robustness of our method with respect to the
hyperparameters that most affect the results. As in the previous sections, we present
experiments separately for the Ising and Water planning tasks.}

\new{In the first experiment, we investigate how the number of restarts affects the final
likelihood of the solution found by MAP search. We use the $i_3$ dataset with the AC
GNN and interface, running $N$ runs for each of $K$ different restart counts, keeping
only the best result out of the $K$ restarts at each run. Figure~\ref{fig:rest_k} shows
the results across 50 such runs for each value of $K$. The plot indicates that
increasing the number of restarts makes better solutions more likely, though for this
problem the improvement is modest, suggesting that good solutions can still be found
with relatively few restarts.}

\begin{figure}[h]
    \centering
    \includegraphics[width=0.6\textwidth]{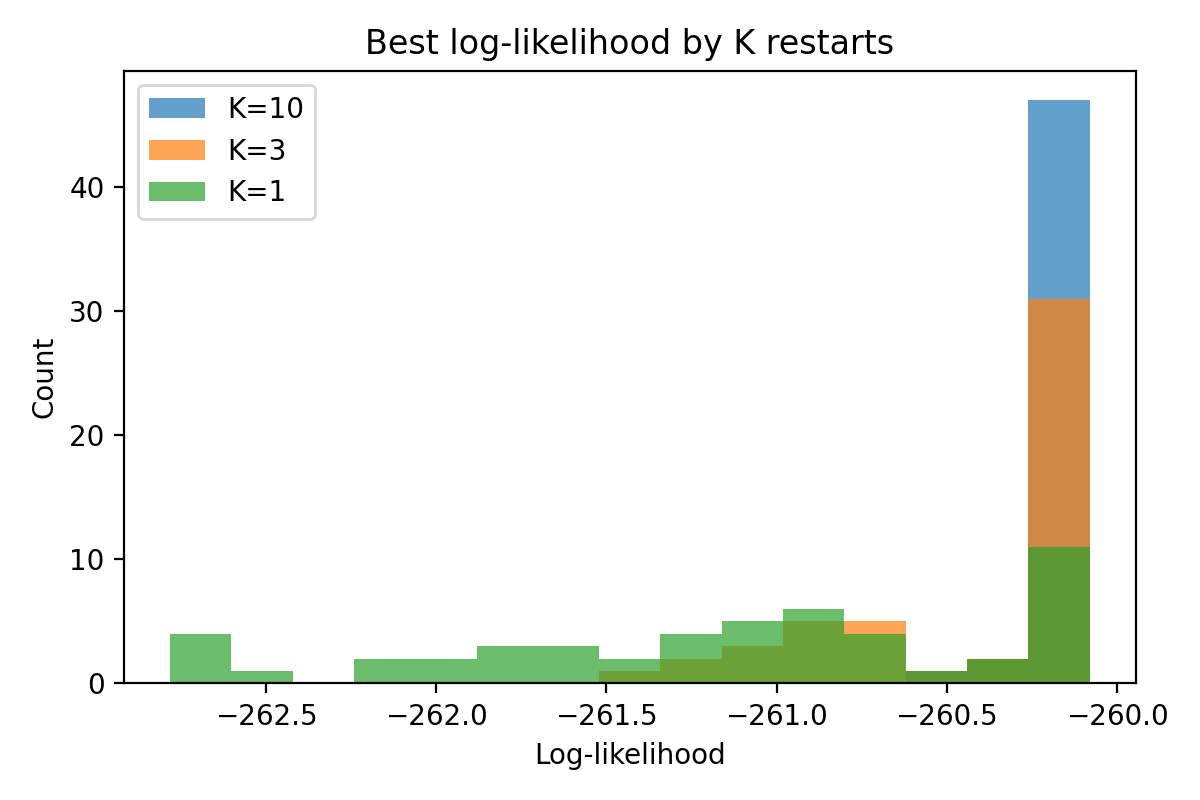}
    \caption{Final log-likelihood of the solutions in the Ising data, using different restarts.}
    \label{fig:rest_k}
\end{figure}

\new{We now investigate how different values of the parameters described in~\ref{sec:likgraph}, \textit{batch size}, \textit{number of chains}, 
and \textit{window size}, affect the results in terms of quality and time. These experiments are conducted on the water planning task described 
in Section~\ref{sec:optimization}, using a tradeoff parameter of $\lambda=0.5$. In this case, we run MAP search with 5 runs and 1 restart each, 
comparing the results in terms of the final log-likelihood of the solution and the total execution time. For each execution, we performed a parameter 
search by keeping two of the three parameters fixed while varying the third, testing batch sizes of 1, 5, 20, and 50, number of chains of 
5, 10, 20, and 30, and window sizes of 2, 5, 10, and 100.}

\new{Figure~\ref{fig:robustness_map} clearly shows that the batch size (the number of atoms
flipped at each iteration) is inversely proportional to the total time: increasing it
reduces execution time while still yielding results very similar to those obtained
with smaller batch sizes. For sampling, on the other hand, increasing the number of
resamples, whether through window size or number of chains, leads to a substantial
increase in total time spent. Regarding the final log-likelihood, it is similarly
clear that spending more time does not substantially improve the results, confirming
that comparable results can still be obtained with fewer samples.}

\begin{figure}
    \centering
    \includegraphics[width=0.9\linewidth]{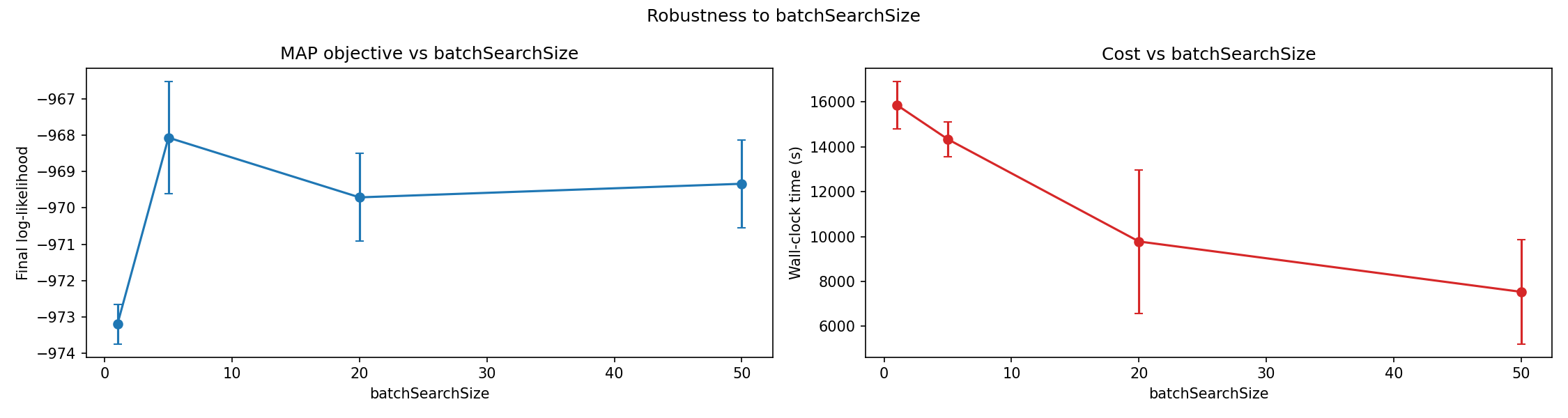}
    \includegraphics[width=0.9\linewidth]{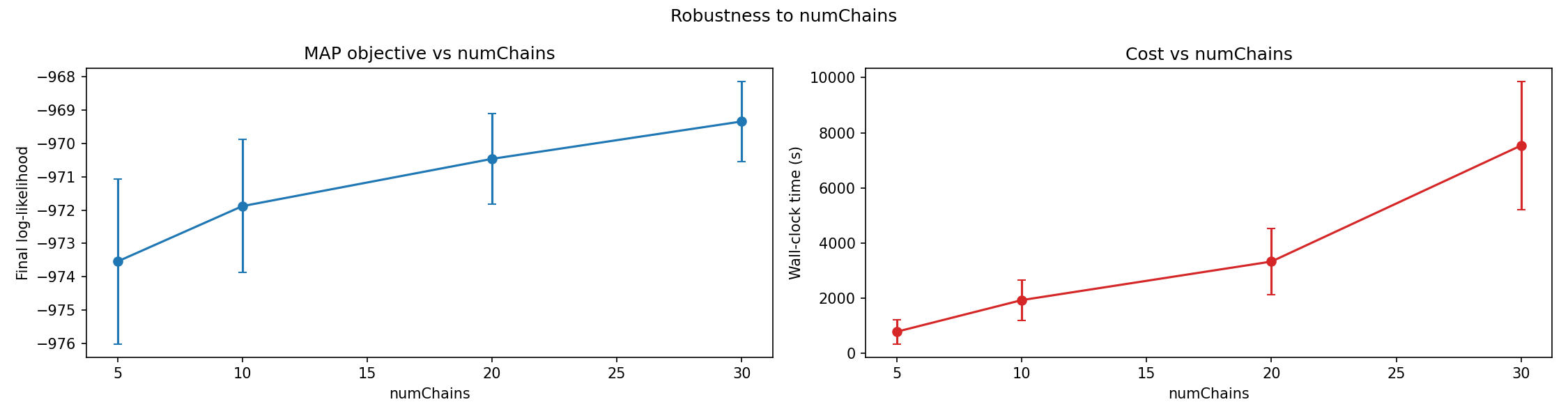}
    \includegraphics[width=0.9\linewidth]{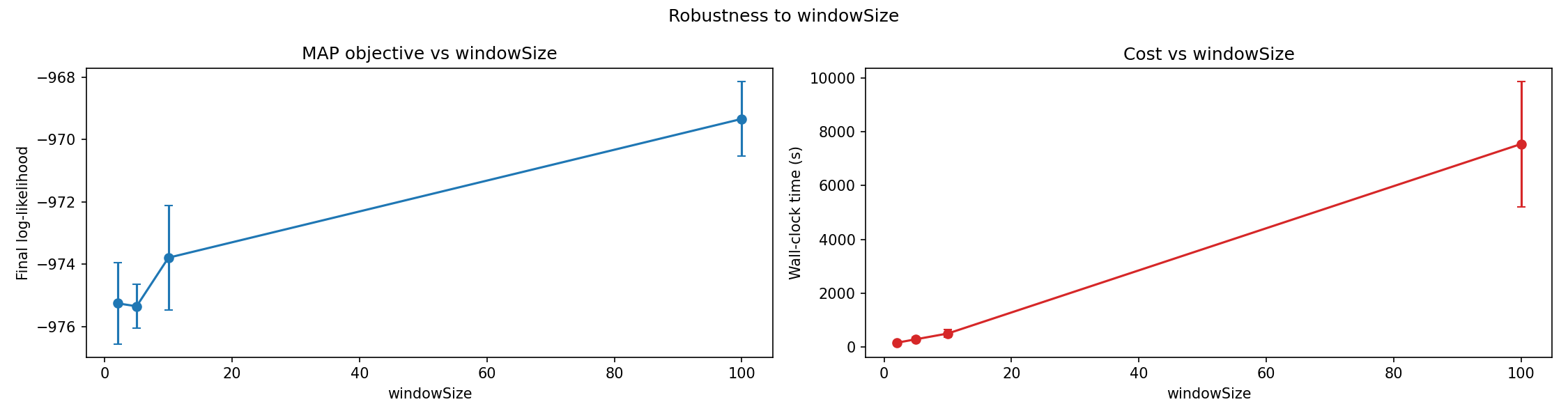}
    \caption{Left: final log-likelihood using different parameters of MAP search, right: total time (in seconds). The parameters used (from top to bottom): batch size, number of chains, and window size.}
    \label{fig:robustness_map}
\end{figure}

\section{Environmental planning}
\label{sec:envplanning}
\subsection{HAWQS Data}
As described in Section 5.1 of the article, the data generated from the SWAT simulation is aggregated and binned annually, resulting in a total of 154 graphs (Fig. \ref{fig:all-data}) from a specific watershed (Fig. \ref{fig:watershed}). The nodes representing subbasins (referred to as water nodes) contain the simulation data, including water flow, temperature, nitrogen concentration, and many other parameters. Additionally, the dataset provides various characteristics of each subbasin. In our case, we encode a boolean feature for each water node to indicate the presence of a reservoir. The non-agricultural land nodes have 12 attributes, representing different land cover types. In the SWAT simulation, they are referred to as BERM (Bermudagrass), FESC (Fescue), FRSD (Deciduous Forest), FRST (Mixed Forest), RIWF (Riparian Forested Wetlands), RIWN (Riparian Non-Forested Wetlands), UPWF (Upland Forested Wetlands), UPWN (Upland Non-Forested Wetlands), and WATR (Open Water Bodies).

\begin{figure}
    \centering
    \includegraphics[width=\linewidth]{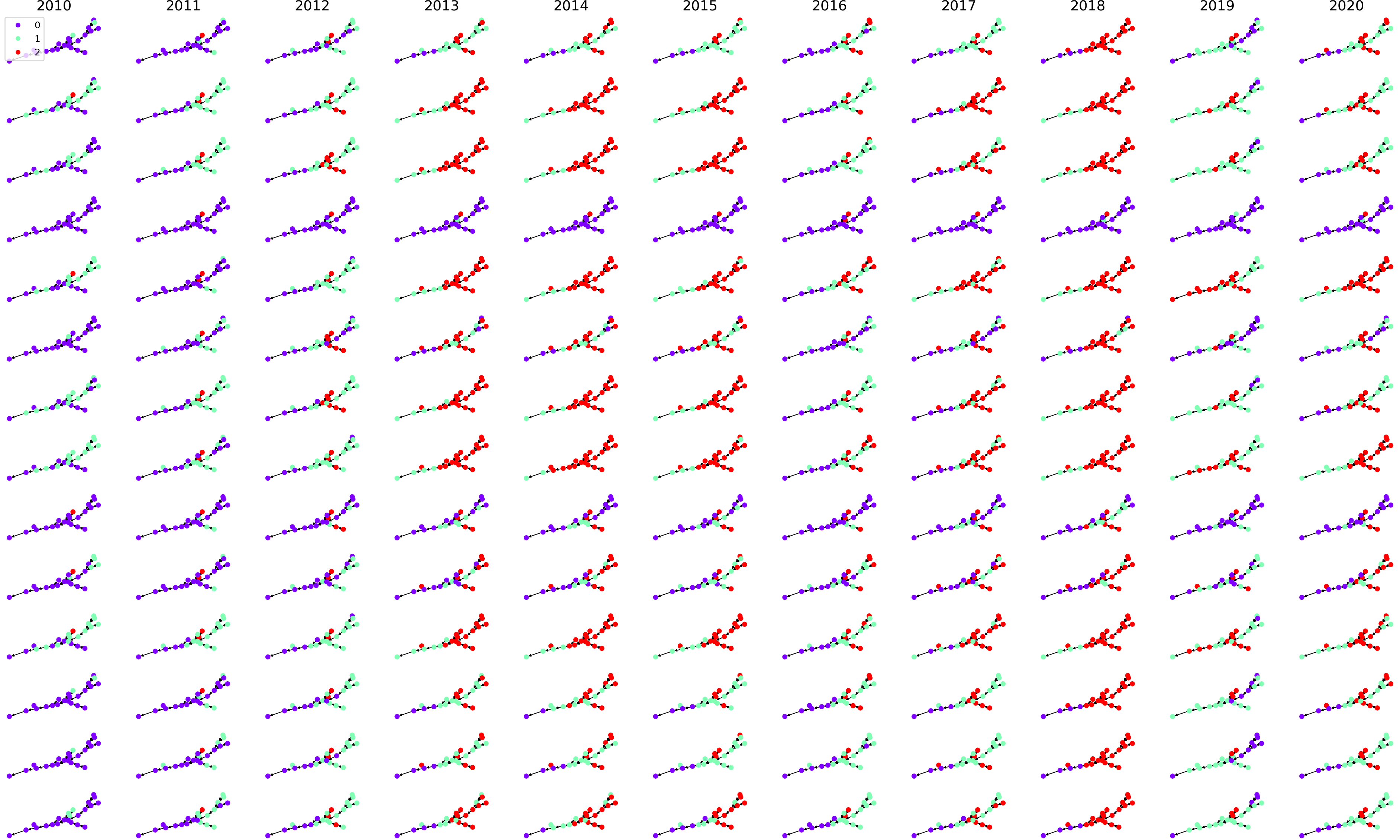}
    \caption{All the graphs generated from the HAWQS simulations. The rows represent one crop scenario (14) while the columns the different years (from 2010 to 2020). The color of the nodes represents the corresponding value of pollution for each subbasin (purple: low, green: medium, red: high). The land nodes are not shown here.}
    \label{fig:all-data}
\end{figure}

\newpage

\subsection{RBN for Node classification under homophilic and heterophilic distribution}

The {\tt Label} attribute is modeled by a GNN. The {\tt @} symbol indicates identifiers for sub-formulas of the model, to be distinguished from actual attributes and relations in the heterogeneous graph (analogous to the $F_{\emph{triangle}}$ element in Figure \ref{fig:rbnillus}).

{\tt hom\_hat} corresponds to $\hat{\emph{LH}}$ in the article, {\tt @predict\_hom} to $\emph{LH}_{\boldsymbol{y}}$, and
{\tt overline\_LH} to $\overline{LH}$.

\begin{lstlisting}[basicstyle=\ttfamily\footnotesize, frame=single]

Label(i) = COMPUTEWITHTORCH config_model [/path/to/model/]            
            COMBINE attr(n) USINGTORCH FORALL n WHERE edge(n, i),
            % second-layer attributes
            COMBINE attr(m) USINGTORCH FORALL m WHERE (edge(m, n) & edge(m, i));

@predict_hom([node]i) = 
        COMBINE 
            Label(i) = Label(j)
        WITH mean
        FORALL j
        WHERE (edge(i,j)|edge(j,i));

@diff(i) = (hom_hat(i) + (-1*@predict_hom(i)));

@lowerbound([node]i) = 
        COMBINE 
	   (4.39 * @diff(i)),
	2.2
	WITH l-reg
	FORALL;
	
@upperbound([node]i) = 
        COMBINE 
	   (-4.39* @diff(i)),
	2.2
	WITH l-reg
	FORALL;

overline_LH([node]i) =  (@upperbound(i) * @lowerbound(i))
        
\end{lstlisting}

\newpage
\subsection{RBN for the environmental optimization experiments}
\label{sec:app_hawqs}

\begin{lstlisting}[basicstyle=\ttfamily\footnotesize, frame=single]
LandUse([hru]l) = SOFTMAX 1,1,1,1;

Pollution(i) = COMPUTEWITHTORCH config_model [/path/to/model/]				
              % first-layer attributes
              COMBINE LandUse(la),AreaAgr(la) USINGTORCH 
                FORALL la WHERE downstream_agr(la, v),
              % second-layer attributes
              COMBINE LandUse(la),AreaAgr(la) USINGTORCH 
                FORALL la, lb WHERE (downstream(lb, v) & downstream_agr(la, lb)),

              COMBINE LandUseUrb(lu),AreaUrb(lu) USINGTORCH 
                FORALL lu WHERE downstream_urb(lu, v),
              COMBINE LandUseUrb(lu),AreaUrb(lu) USINGTORCH 
                FORALL lu WHERE (downstream(lb, v) & downstream_urb(lu, lb)),

              COMBINE SubType(v) USINGTORCH,
              COMBINE SubType(vb) USINGTORCH 
                FORALL vb WHERE downstream(vb, v);

@profit_land([hru]l) =  
    COMBINE 
        WIF LandUse(l)=CORN THEN (5.0*AreaAgr(l)) ELSE 0.0,
        WIF LandUse(l)=COSY THEN (2.0*AreaAgr(l)) ELSE 0.0,
        WIF LandUse(l)=PAST THEN (1.0*AreaAgr(l)) ELSE 0.0,
        WIF LandUse(l)=SOYB THEN (4.0*AreaAgr(l)) ELSE 0.0
    WITH sum;

@profit_sub([sub]s) =   
    COMBINE 
        @profit_land(l)
    WITH sum
    FORALL l
    WHERE hru_agr_to_sub(l,s);

@max_profit_sub([sub]s) =   
    COMBINE (5.0*AreaAgr(l))
    WITH sum
    FORALL l
    WHERE hru_agr_to_sub(l,s);

@min_profit_sub([sub]s) =   
    COMBINE (1.0*AreaAgr(l))
    WITH sum
    FORALL l
    WHERE hru_agr_to_sub(l,s);

@inv_maxmin_sub([sub]s) =
    COMBINE 
        (-1*@min_profit_sub(s)),
        @max_profit_sub(s)
    WITH invsum
    FORALL 
    WHERE true;

@target_s([sub]s) =  
    ((@profit_sub(s) + (-1*@min_profit_sub(s)))
        * @inv_maxmin_sub(s));

all_const([sub]s) = WIF 0.1
                    THEN Pollution(s)=LOW
                    ELSE @target_s(s);
\end{lstlisting}

This section describes the Relational Bayesian Network (RBN) designed to address a multi-objective optimization problem that balances economic returns with environmental quality. The Pollution is modeled by a GNN for predicting pollution probabilities alongside expert-defined constraints for profit optimization.

For each agricultural land unit $l$, the profit is computed based on its land use type, $\mathrm{LandUse}(l)$, and its agricultural area, $\mathrm{AreaAgr}(l)$. The profit contribution is defined as:
\begin{equation}
\text{@profit\_land}(l) =
\begin{cases}
5.0 \times \mathrm{AreaAgr}(l) & \text{if } \mathrm{LandUse}(l)=\mathrm{CORN},\\[1mm]
2.0 \times \mathrm{AreaAgr}(l) & \text{if } \mathrm{LandUse}(l)=\mathrm{COSY},\\[1mm]
1.0 \times \mathrm{AreaAgr}(l) & \text{if } \mathrm{LandUse}(l)=\mathrm{PAST},\\[1mm]
4.0 \times \mathrm{AreaAgr}(l) & \text{if } \mathrm{LandUse}(l)=\mathrm{SOYB}.
\end{cases}
\end{equation}
These contributions are aggregated by summing over the relevant lands. \\
The profit constraint is modeled by the \texttt{@target\_(s)} formula, which returns the min-max normalized profit for each subbasin $\in [0,1]$, based on the minimum and maximum possible profit values for that subbasin.\\
The overall constraint of the model is given by:
\[
\text{all\_const}(s) = \lambda * \mathrm{Pollution}(s) + (1-\lambda)*(@target(s)), \\
\]
Here, the constant \(0.1\) in the \texttt{WIF} represents the \(\lambda\) parameter from the article.
\texttt{Pollution(s)} is the probability formula represented by the GNN ($X_1$) and \texttt{@target\_s(s)} represents the second objective for maximizing the profit ($X_2$). $\texttt{all\_const(s)}$ corresponds to $\eta_X$ in the main text.

\end{document}